\documentclass{article}
\usepackage{xinttools} % for the \xintFor***
\usepackage{tikz}
\usetikzlibrary{calc}
\usepackage{subcaption}
\usepackage{booktabs}
\usepackage{pgfplots}
\usepackage{graphicx}

\usepackage{amsthm,amsmath,amssymb}
\usepackage{bbm}
\usepackage{comment}
\usepackage{hyperref}
\usepackage{cleveref}
\usepackage{fullpage}
\usepackage{newpxmath, newpxtext}
\hypersetup{
    colorlinks=true,
    linkcolor=purple,
    citecolor=blue,
    filecolor=purple,
    urlcolor=purple,
}
\usepackage{cite}[sort&compress]
\renewcommand{\leq}{\leqslant}
\renewcommand{\geq}{\geqslant}

\renewcommand{\le}{\leqslant}
\renewcommand{\ge}{\geqslant}

\newcommand{\np}{{\rm NP}}
\newcommand{\conp}{{\rm coNP}}

\newtheorem{theorem}{Theorem}
\newtheorem{proposition}{Proposition}
\newtheorem{example}{Example}
\newtheorem{lemma}{Lemma}
\newtheorem{definition}{Definition}

\newtheorem{corollary}{Corollary}

\usepackage[lined,boxed,commentsnumbered]{algorithm2e}
\def\biglen{20cm} % playing role of infinity (should be < .25\maxdimen)
\tikzset{
  half plane/.style={
    to path={
       ($(\tikztostart)!.5!(\tikztotarget)!#1!(\tikztotarget)!\biglen!90:(\tikztotarget)$)
    -- ($(\tikztostart)!.5!(\tikztotarget)!#1!(\tikztotarget)!\biglen!-90:(\tikztotarget)$)
    -- ([turn]0,2*\biglen) -- ([turn]0,2*\biglen) -- cycle}},
  half plane/.default={1pt}
}

\def\n{23} % number of random points
\def\maxxy{4} % random points are in [-\maxxy,\maxxy]x[-\maxxy,\maxxy]

\author{
  Pablo Barcel\'o\\
  IMC PUC, IMFD, CENIA\\
  \url{pbarcelo@ing.uc.cl}
  \and 
  Alexander Kozachinskiy\\
  CENIA\\
  \url{alexander.kozachinskyi@cenia.cl}
  \and
  Miguel Romero Orth\\
  DCC PUC, CENIA\\
  \url{mgromero@uc.cl}
  \and 
  Bernardo Subercaseaux\\
  Carnegie Mellon University\\
  \url{bersub@cmu.edu}
  \and
  Jos\'e Verschae\\
  IMC PUC\\
  \url{jverschae@uc.cl}
  }
\title{Explaining $k$-Nearest Neighbors:\\ 
Abductive and Counterfactual Explanations}

\begin{document}

\maketitle

\begin{abstract}
      Despite the wide use of $k$-Nearest Neighbors as classification models, their explainability properties remain poorly understood from a theoretical perspective.
 While nearest neighbors classifiers offer  interpretability from a \emph{``data perspective''}, in which the classification of an input vector $\bar{x}$ is explained by identifying the vectors $\bar{v}_1, \ldots, \bar{v}_k$ in the training set that determine the classification of $\bar{x}$, we argue that such explanations can be impractical in high-dimensional applications, where each vector has hundreds or thousands of features and it is not clear what their relative importance is.
Hence, we focus on understanding nearest neighbor classifications through a \emph{``feature perspective''}, in which the goal is to identify how the values of the features in $\bar{x}$ affect its classification. Concretely, we study \emph{abductive explanations} such as ``minimum sufficient reasons'', which correspond to sets of features in $\bar{x}$ that are enough to guarantee its classification, and 
\emph{counterfactual explanations} based on the minimum distance feature changes one would have to perform in $\bar{x}$ to change its classification. We present a detailed landscape of positive and negative complexity results for counterfactual and abductive explanations, distinguishing between discrete and continuous feature spaces, and considering the impact of the choice of distance function involved. Finally, we show that despite some negative complexity results, Integer Quadratic Programming and SAT solving allow for computing explanations in practice.
\end{abstract}

\section{Introduction}\label{sec:intro}
\paragraph{{\bf Nearest Neighbor classification.}} 
$k$-Nearest Neighbor ($k$-NN) classification is one of the most widely used supervised learning techniques~\cite{Biau2015}.
In $k$-NN classification, we assume a set of points $S$ over a metric space, where each point has already been labeled as either \emph{positive} or \emph{negative}. Then, a new point $\bar{x}$ is classified as either positive or negative by taking the majority label of its $k$ closest neighbors in $S$. 
%Naturally, for this approach to be reasonable, one needs the distance function to approximate our context-dependent idea of semantic similarity. 
%As an example, training a $5$-NN classifier over the MNIST handwritten-digit classification dataset~\cite{deng2012mnist}
%, using the \texttt{scikit-learn} open source library~\cite{pedregosaScikitlearnMachineLearning2011},
% yields over $97\%$ of accuracy. It turns out that images 
 %(encoded as $784$-dimensional vectors, where each component corresponds to a pixel) 
 %of the same digit often have a small $\ell_2$-distance compared to images of other digits, as illustrated in~\Cref{fig:digits}.
The study of
$k$-NN classification has been a recurring focus in the data management community, encompassing extensive research on its behavior in high-dimensional spaces \cite{DBLP:conf/icdt/BeyerGRS99,DBLP:conf/vldb/HinneburgAK00,DBLP:conf/sigmod/TaoYSK09} and its properties when dealing with uncertain data \cite{DBLP:conf/pods/AgarwalESZ12,DBLP:conf/pods/AgarwalAHPYZ13,DBLP:conf/icdt/FanK22}. Considerable effort has also been directed toward the development of efficient algorithms and data structures to enable scalable NN queries \cite{DBLP:conf/sigmod/RoussopoulosKV95,DBLP:conf/pods/AgarwalFMN16}. As of late, $k$-NN has also become 
key to several search and retrieval problems in vector databases \cite{DBLP:journals/corr/abs-2402-01763}. 
%For a recent example, RAG (Retrieval Augmented Generation) systems have become one of the main tools for contextualizing LLMs in complex domains. 
%RAG systems work by storing~\emph{textual embeddings}, that encode potentially large sources of text into a set high-dimensional numerical vectors (illustrated in~\Cref{fig:word-embeddings}, which are then used to retrieve relevant context by a NN search. 
For example, in \emph{Retrieval-Augmented Generation} (RAG) systems, the goal is to identify the most relevant sections of a document for a given query. This is achieved by performing a nearest-neighbor query within a textual-embedding space.

\paragraph{{\bf Formal explainability.}} 
Emerging data-driven applications, particularly those leveraging machine learning systems, are introducing new demands on classification methods. One of the most critical requirements is {\em explainability}: in many high-stakes applications, it is not enough for classifiers to be accurate; they must also provide clear and understandable explanations for their decisions~\cite{BarredoArrieta2020}. 
%Explainability has increasingly been recognized as a fundamental prerequisite for deploying trustworthy AI systems, as it fosters transparency, accountability, and user trust. 
A significant milestone in this field has been the development of {\em formal} frameworks for explainability. The advantages of adopting such a principled approach have been comprehensively outlined in a recent survey~\cite{formal-xai}. 
Two prominent examples of this methodology are: 
\begin{itemize} 
    
    \item {\em Abductive explanations:} These aim to identify a small subset of components in the input $\bar{x}$ that is sufficient to justify the classifier's output for $\bar{x}$ \cite{DBLP:conf/aaai/Ribeiro0G18,DBLP:journals/corr/abs-2002-09284,DBLP:conf/aaai/IgnatievNM19}. 
    %This approach seeks to pinpoint the essential features or inputs that directly lead to the observed decision, ensuring the explanation is as concise as possible. 
    More formally, an abductive explanation for $\bar x$ with respect to a given classifier is a subset $X$ of the components of $\bar x$ such that every input $\bar y$ that coincides with $\bar x$ over the components in $X$ is classified in the same way by the classifier. Abductive explanations are also called {\em sufficient reasons} \cite{shih2018symbolic}. One then aims to find sufficient reasons for $\bar x$ that are {\em minimum} in terms of their cardinality.  

\item {\em Contrastive explanations:} These focus on the robustness of a classification, examining how much a given point $\bar{x}$ must be altered to change the output of the classifier \cite{artelt2019counterfactual,Guidotti2022}. 
%In essence, they identify the minimal changes required to shift the classification of $\bar{x}$, offering insights into the decision boundaries of the classifier.
More formally, a {\em contrastive explanation at distance $p$ from $\bar x$}, with respect to a given classifier, is another input $\bar y$ such that $\lVert\bar x - \bar y\rVert \leq p$ and $\bar y$ is classified differently from $\bar x$. 
\end{itemize}
%Intuitively, abductive explanations for a binary classifier function $f$ over an input $\bar{x}$ provide small certificates for the fact that $f(\bar{x}) = 1$, while
%contrastive explanation explain the fact $f(\bar{x}) = 1$ by answering the counterfactual ``why not $f(\bar{x}) = 0$?''.

\paragraph{{\bf Why feature-based explanations for $k$-NNs?}} 
Traditionally, $k$-NN models are considered "self-interpretable" because they identify a subset of training data that determines a new input's classification~\cite{molnar2022}. However, this view is overly simplistic, as interpretability depends on whether individual instances and their features are understandable~\cite{lipton2018mythos,molnar2022}. In high-dimensional settings, the $k$ nearest neighbors may already be too complex for direct human interpretation. Similar challenges arise in other classifiers like decision trees, often viewed as "self-interpretable". This has spurred research into concise, feature-based explanations—such as abductive and counterfactual ones~\cite{izzaComputingProbabilisticAbductive2023b}. As the next example shows, applying this approach to NN classification yields meaningful insights.

\begin{example} 
Consider a $1$-NN classifier trained on a subset of the MNIST dataset containing digits 4 and 
9. The test image in~\Cref{fig:mnist-counterfactuals}a is correctly classified as a 
4 based on its NN in~\Cref{fig:mnist-counterfactuals}b, while its closest counterfactual, shown in~\Cref{fig:mnist-counterfactuals}c, is classified as a 
9 due to its NN in~\Cref{fig:mnist-counterfactuals}d. The explanation for why the test image is not classified as a 
9, highlighted in~\Cref{fig:mnist-counterfactuals}e, identifies 13  pixels that correspond to key differences in the digit’s structure. This counterfactual explanation reveals the minimal changes, among the dataset's 784 features, needed to alter the classification. \qed
%Let us present a concrete example of the information revealed by ``feature-based'' explanations.  Consider a $1$-NN classifier trained on a 
%binarized 
%subset of the MNIST dataset~\cite{deng2012mnist} containing only images of digits $4$ and $9$.~\Cref{fig:mnist-counterfactuals}a depicts a test image that is correctly classified as a $4$, since its NN in the training set, presented in~\Cref{fig:mnist-counterfactuals}b, was a $4$. Its closest counterfactual is presented in~\Cref{fig:mnist-counterfactuals}c, and is classified as a $9$ based on its NN depicted in~\Cref{fig:mnist-counterfactuals}d. The explanation for why the test image of ~\Cref{fig:mnist-counterfactuals}a is \emph{not} classified as a $9$ is presented in~\Cref{fig:mnist-counterfactuals}e, by highlighting 13 pixels that would be enough to change its classification, and that intuitively seem to correspond to the top of the digit approximating a closed loop, while extending the vertical stroke further down.
%This way, the counterfactual explanation clarifies a small set of features, among the $784$ features of the dataset, that would be enough to change the classification of the test image.  
\end{example}

%An abductive explanation, in turn, would clarify which features are making ...
%\medskip 

\begin{figure}
    \begin{subfigure}{0.24\textwidth}
    \centering
    \includegraphics[scale=0.12]{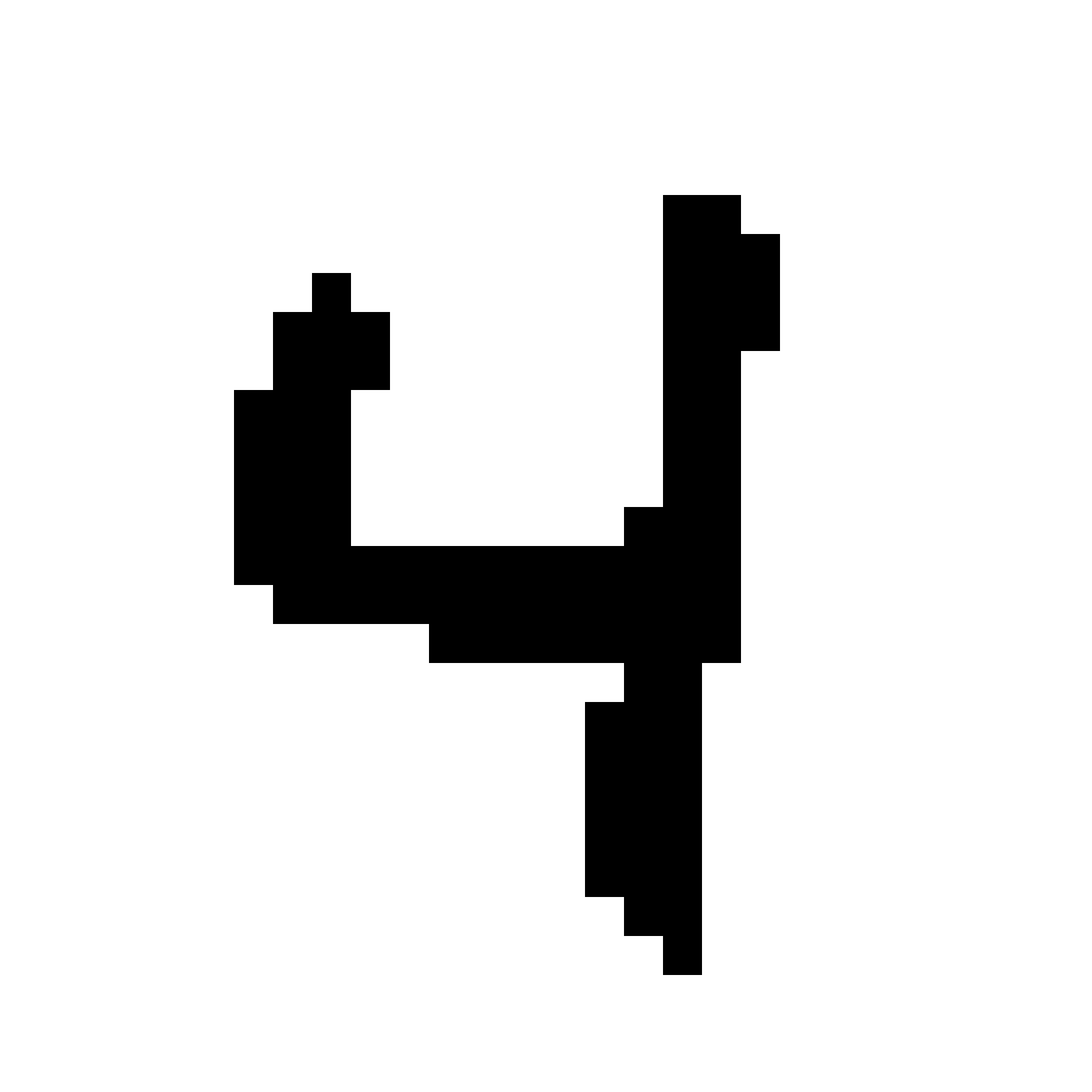}
    \caption{Test image.}
    \end{subfigure}
    \begin{subfigure}{0.24\textwidth}
     \centering
    \includegraphics[scale=0.12] {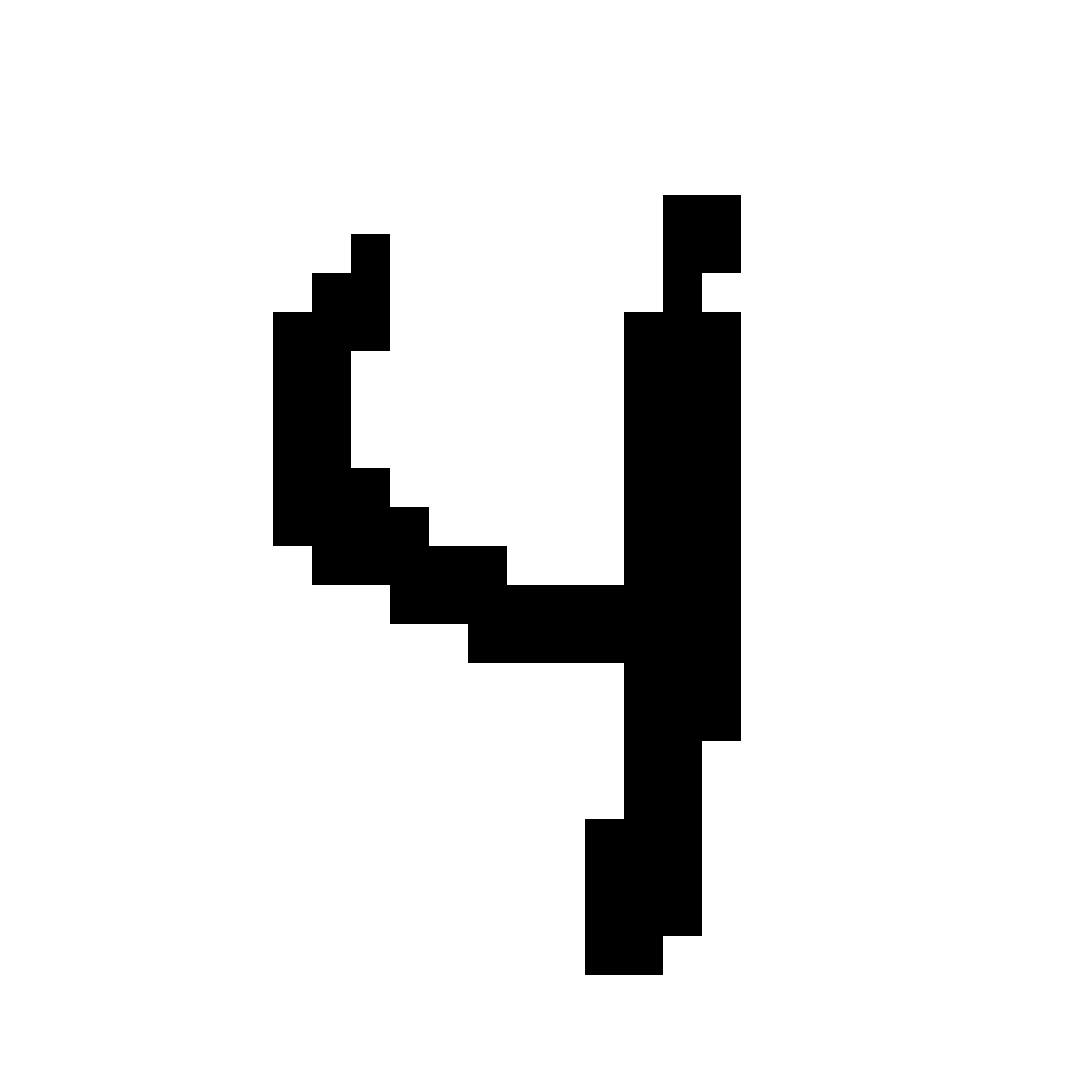}
    \caption{Nearest neighbor of (a).}
    \end{subfigure}
    \begin{subfigure}{0.24\textwidth}
     \centering
        \includegraphics[scale=0.12]{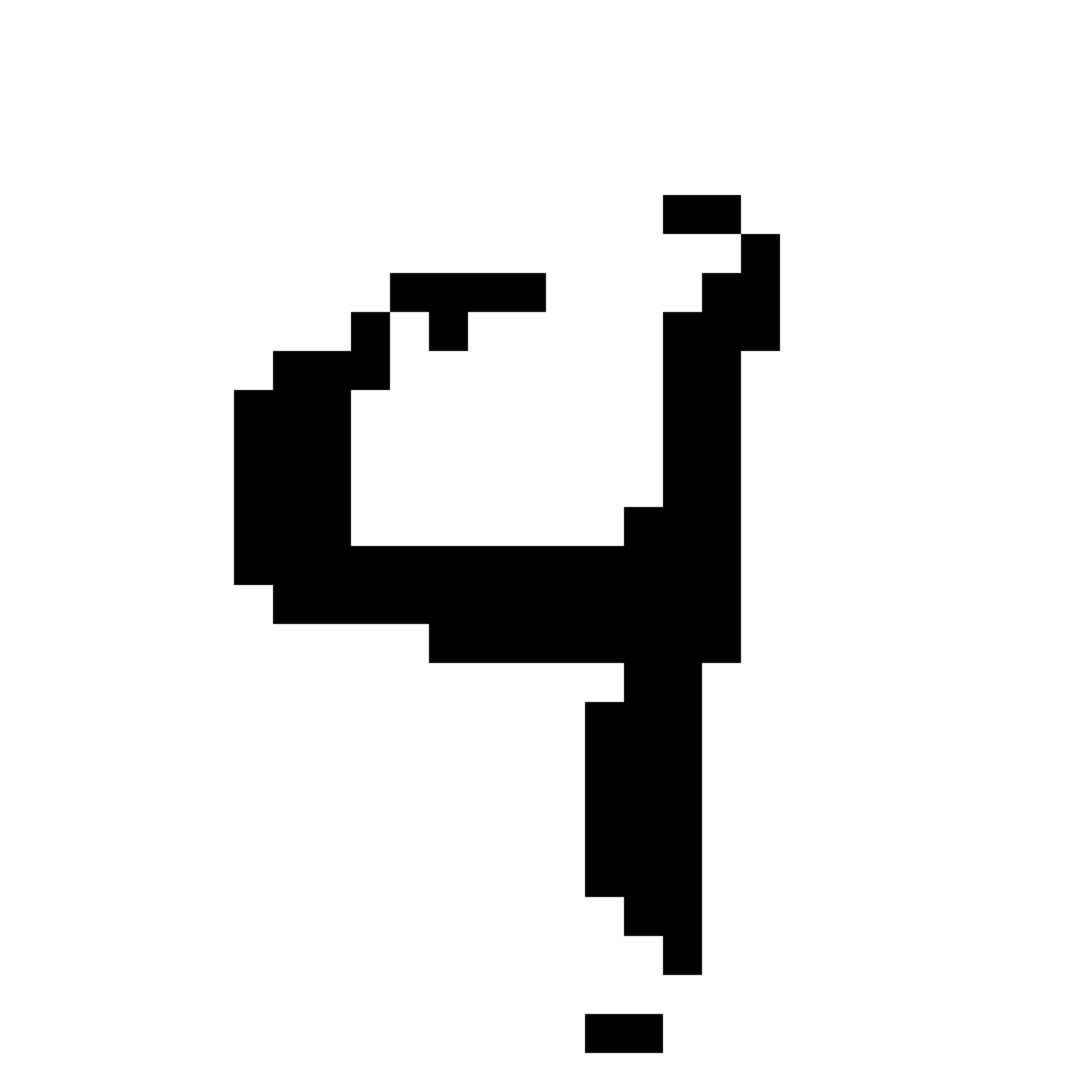}
        \caption{Closest counterfactual.}
    \end{subfigure}
    \begin{subfigure}{0.24\textwidth}
     \centering
        \includegraphics[scale=0.12]{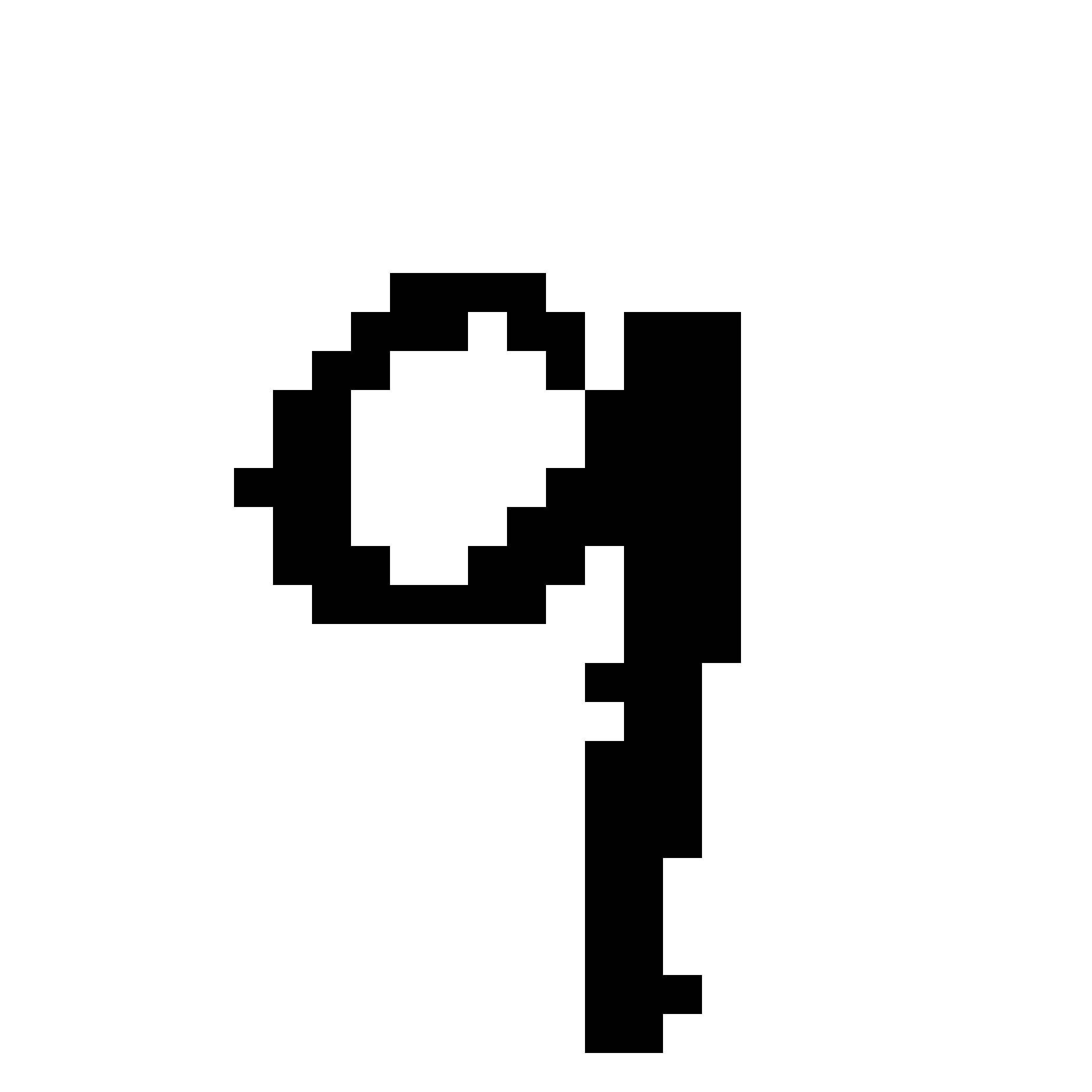}
        \caption{Nearest neighbor of (c).}
    \end{subfigure}

    \vspace{0.5cm}
    
    \begin{subfigure}{0.31\textwidth}
    \centering
        \includegraphics[scale=0.12]{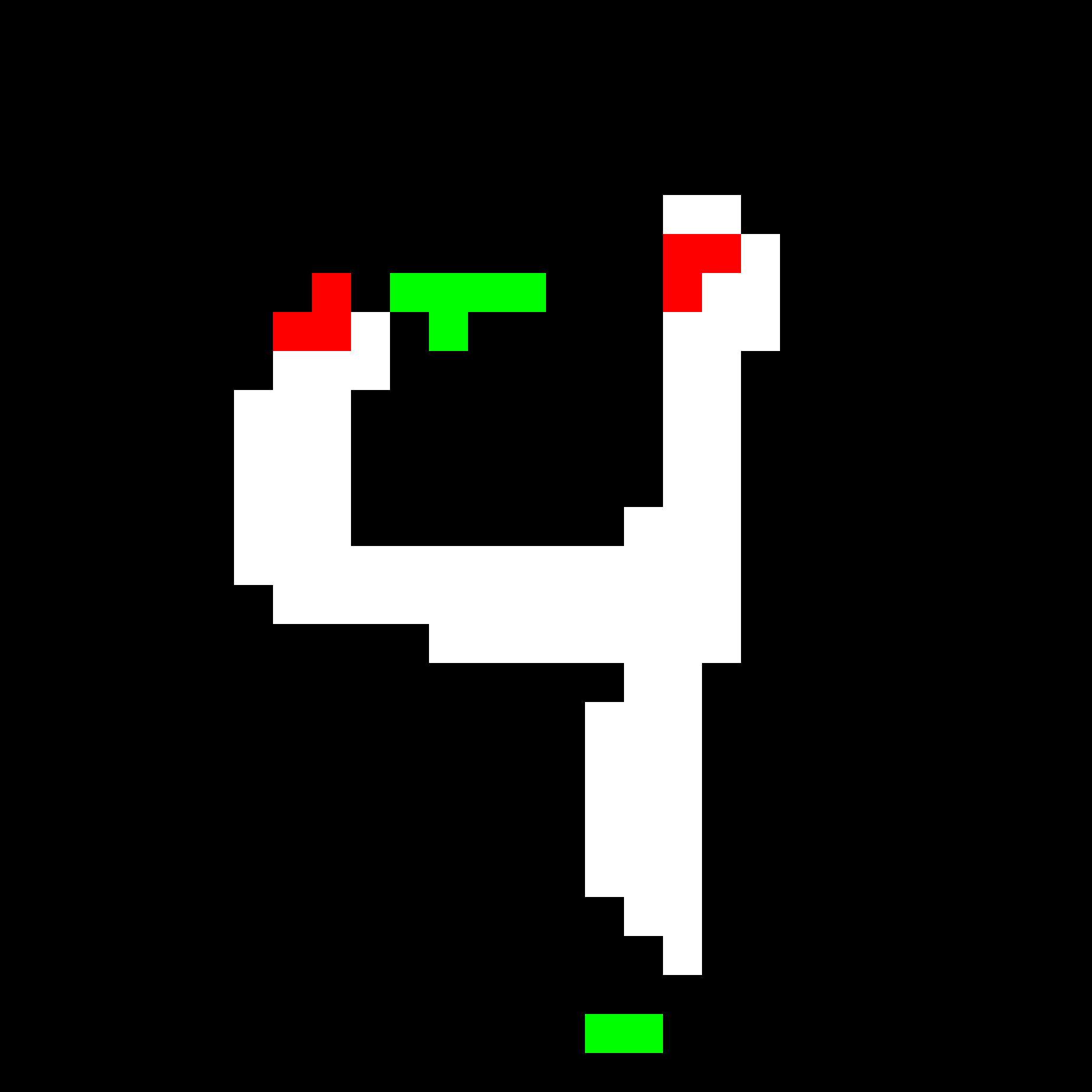}
        \caption{Diff. map between (a) and (c).}
    \end{subfigure}
     \begin{subfigure}{0.31\textwidth}
    \centering
        \includegraphics[scale=0.12]{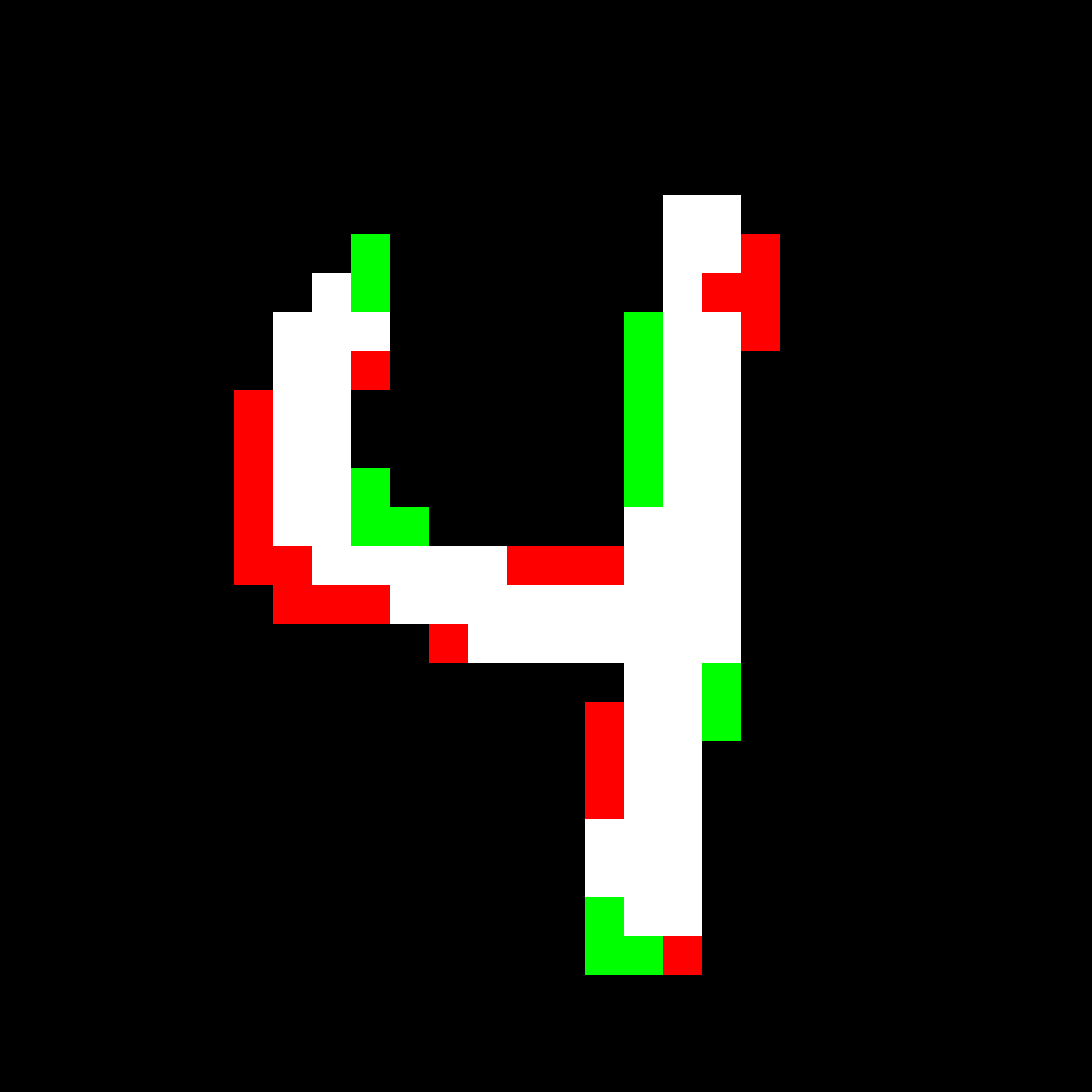}
        \caption{Diff. map between (a) and (b).}
    \end{subfigure}
     \begin{subfigure}{0.31\textwidth}
    \centering
        \includegraphics[scale=0.12]{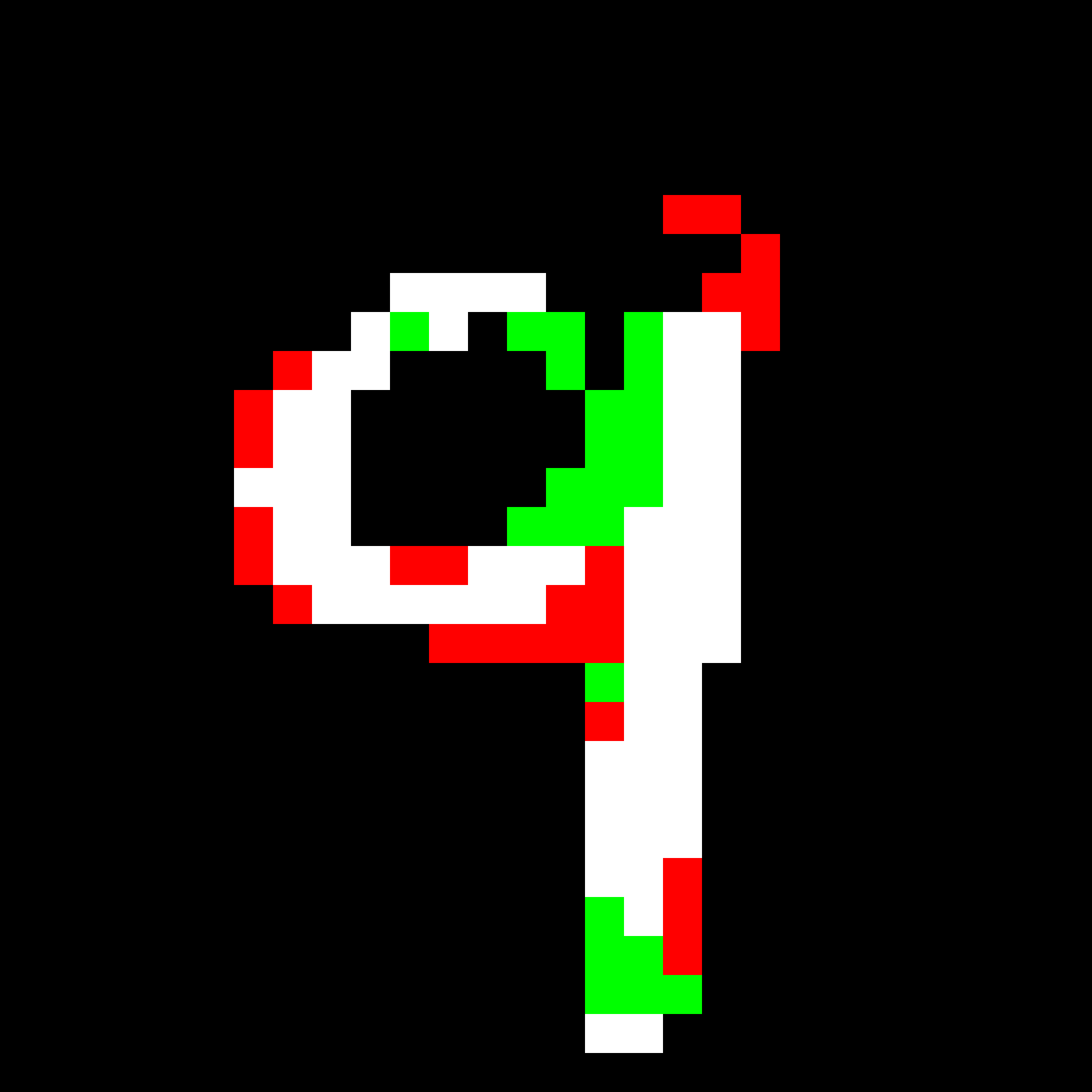}
        \caption{Diff. map between (c) and (d).}
    \end{subfigure}

\caption{Illustration of a counterfactual explanation for an image of digit $4$ in the binarized MNIST dataset, which after changing $13$ pixels is classified as a $9$.}\label{fig:mnist-counterfactuals}
\end{figure}

\paragraph{{\bf Context.}}
The algorithmic aspects of computing abductive and counterfactual explanations for ML models have garnered significant attention in recent years \cite{shih2018symbolic,DBLP:conf/nips/0001GCIN20,NEURIPS2020_b1adda14,Izza2021EfficientEW,DBLP:conf/kr/HuangII021,pmlr-v139-marques-silva21a,blanc2021provably,DBLP:conf/icml/0001GCIN21,DBLP:journals/jair/WaldchenMHK21,DBLP:conf/nips/ArenasBOS22,Ignatiev2021,Ignatiev2021a,Izza2022,Audemard,https://doi.org/10.48550/arxiv.2205.09569,DBLP:conf/ijcai/HanCD23,DBLP:conf/lics/Darwiche23,DBLP:conf/ecai/AudemardLM24}.
These efforts have explored the computational cost of generating explanations across various ML models, including decision trees, binary decision diagrams, Bayesian networks, neural networks, and graph-based classifiers.
Surprisingly, despite the foundational importance of $k$-NN in machine learning and data management, the literature on explainability for $k$-NN classifiers remains sparse. The few existing works primarily adopt operational approaches, leveraging mixed integer programming and constraint programming techniques to solve relevant explainability problems \cite{DBLP:conf/icml/ForelPV23,DBLP:journals/dam/LamotheCG24,magagniniNearestNeighborsCounterfactuals2025}. However, the theoretical complexity of computing explanations for $k$-NN models remains an unexplored question, which constitutes a critical gap in the field. In particular, we do not know which of these problems are computationally intractable, and thus can only be approached by applying modern solvers technology for satisfiability or integer programming problems.       

\paragraph{{\bf Our contributions}} 
We perform an analysis of the complexity of checking and computing explanations for $k$-NN classifiers in two common settings: (a) the {\em continuous} setting, where points are vectors of real numbers, and the distance is based on the $\ell_p$-norm for some integer $p > 0$, and (b) the {\em discrete} setting, where points are Boolean vectors, and the Hamming distance is used.

\begin{itemize} 
\item
We start by studying the complexity of checking for the existence of \emph{sufficient reasons} of a certain size. We show that this problem is \np-hard, in terms of the number of features, in both the discrete and continuous settings. In the latter case, hardness holds for every distance $\ell_p$, 
where $p$ is a positive integer. 

\item We next examine the continuous setting to assess tractability for other standard explainability problems.  
Focusing on the \(\ell_2\) and \(\ell_1\) distances, two of the most commonly used metrics in practice, we show that tractability depends crucially on the metric chosen.

\begin{itemize} 
\item 
For the $\ell_2$-distance, several positive results are established. First, we prove that the problem of checking the existence of a counterfactual explanation at a certain distance is tractable. Furthermore, if such a counterfactual explanation exists, at least one can be computed in polynomial time. Next, we examine the problem of checking if a subset of components of an input $\bar x$ constitutes a sufficient reason and prove that this problem is also tractable. Consequently, a {\em minimal} sufficient reason for $\bar x$—where minimality refers to set containment—can always be computed in polynomial time. The previous problems can be solved in time $n^{O(k)}$, which is polynomial when $k$ is fixed. For the case of counterfactual explanations, we further show that, under widely believed assumptions from parameterized complexity, the exponential dependence in $k$ is unavoidable, as the problem is W$[1]$-hard when $k$ is the parameter.

\item We then examine the $\ell_1$-distance and show that not all positive results derived for the $\ell_2$-distance carry over. Specifically, checking the existence of a counterfactual explanation at a certain distance becomes \np-complete in this case. However, for a $k$-NN classifier with $k = 1$, a {\em minimal} sufficient reason for input $\bar x$ can still be computed in polynomial time.  For fixed $k \geq 3$, we show that checking whether a subset of components of an input $\bar{x}$ is a sufficient reason is $\conp$-complete, and, as a consequence, that computing minimal sufficient reasons in this case is intractable.

\end{itemize} 

\item Afterwards, we consider the discrete setting and show that it resembles the behavior observed for the $\ell_1$-distance in the continuous setting. Specifically: (a) checking the existence of a counterfactual explanation at a certain distance is \np-complete, and (b) for a $k$-NN classifier with $k = 1$, a {\em minimal} sufficient reason for an input $\bar x$ can be computed in polynomial time. However, this property does not generalize beyond $k = 1$. In particular, checking if a subset of components of an input $\bar x$ is a sufficient reason for a $k$-NN classifier, for fixed $k\geq 3$, is $\conp$-complete. Our results also imply hardness for computing minimal sufficient reasons, in the case $k\geq 3$.  We further show that the problem of checking the existence of a sufficient reason of a certain size is complete for the complexity class $\Sigma_2^p$.

\item We present a preliminary analysis of computing explanations for $k$-NN when $k=1$, addressing both polynomial-time and \np-hard problems. For the latter, we employ two standard approaches: Integer Quadratic Programming (IQP) and SAT-solving (SAT). While Mixed Integer Programming (MIP) for counterfactual explanations was previously explored~\cite{DBLP:conf/iscopt/ContardoFRV24}, our SAT encoding is novel and utilizes a recent solver with native support for cardinality constraints~\cite{reeves2024clauses}. Interestingly, even for \np-hard problems, explanations for datasets with hundreds of features and thousands of points can be computed in under two minutes.

%We provide a preliminary analysis of the practicality of computing explanations for $k$-NN when $k=1$, both for problems that we show to admit polynomial-time algorithms and for the ones we prove NP-hard. Concretely, we use two standard avenues for dealing with NP-hard problems: Integer Quadratic Programming (IQP) and SAT-solving (SAT). While a Mixed Integer Program (MIP) for counterfactual explanations has already been developed by Contardo et al.~\cite{DBLP:conf/iscopt/ContardoFRV24}, our SAT encoding is novel and leverages a recent solver with native support for reasoning with cardinality constraints~\cite{reeves2024clauses}. In summary, even for the problems proved NP-hard, it is possible to compute explanations for hundreds of features, and thousands of datapoints in under a minute.

\end{itemize} 

\paragraph{{\bf Organization of the paper}}  
Basic definitions are provided in Section \ref{sec:definitions}, and the explainability problems are discussed in Section \ref{sec:problems}. Negative results concerning minimum sufficient reasons are presented in Section \ref{sec:msr}. Additional results for the continuous setting, based on the \(\ell_2\)-distance, are covered in Section \ref{sec:l2}, while those for the \(\ell_1\)-distance are detailed in Section \ref{sec:l1}. Results for the discrete setting appear in Section \ref{sec:discrete}. Further complexity completeness results are presented in Section \ref{sec:further}. Experimental results are discussed in Section~\ref{sec:implementation}, while final remarks are included in Section \ref{sec:final}.

\section{Definitions}\label{sec:definitions}
\noindent 
\paragraph{{\bf Basics}} 
We consider pairs of the form $(M,D)$, called {\em metric space families}, in which $M$ is a set and $D = \{d_n \mid n > 0\}$ satisfies that $d_n : M^n \times M^n \to \mathbb{R}$ is a metric (often referred as distance) on $M^n$, for every $n > 0$. Elements in $M^n$ are called vectors, and are typically denoted as $\bar x,\bar y,\bar z$. For $i \in \{1,\dots,n\}$, we write $\bar x[i]$ to denote the $i$th component of vector $\bar x$. When it is clear from the context, we also use $\bar{x}_i$ to denote the $i$th component of vector $\bar x$. 

\paragraph{{\bf Metric spaces studied in the paper}} 
In this article, we focus on two particular cases for the metric space families of the form $(M,D)$: 

\begin{itemize} 

\item {\em Continuous case:} Here $M = \mathbb{R}$ and $D = \{d_n \mid n > 0\}$ satisfies that there exists an integer 
$p > 0$ such that the distance $d_n$ is the one based on the $\ell_p$-norm over $\mathbb{R}^n$, for every $n > 0$. In this particular case, we denote $D$ as $D_p$. 

\item {\em Discrete case:} Here $M = \{0,1\}$ and $D = \{d_n \mid n > 0\}$ satisfies that $d_n$ is 
the Hamming distance on $\{0,1\}^n$. That is, if $\bar x,\bar y \in \{0,1\}^n$, then $d_n(\bar x,\bar y)$ is the number of components $i \in \{1,\dots,n\}$ for which $\bar x[i] \neq \bar y[i]$. 
In this case, we denote $D$ as $D_H$. 

\end{itemize}

\noindent 
\paragraph{{\bf Nearest neighbor classification}} 
We fix a metric space family $(M,D)$ as defined above. 
Let $k$ be a fixed odd integer. Consider two subsets $S^+$ and $S^-$ of $M^n$, for $n > 0$, 
where vectors in $S^+$ represent 
{\em positive examples}, and vectors in $S^-$ represent {\em negative examples}. 
We assume that $|S^+ \cup S^-| \geq k$. 
For the pair $(M,D)$, we aim to construct a {\em $k$-Nearest Neighbor ($k$-NN) classification function} 
\[
f^k_{S^+,S^-} : M^n \to \{0,1\},
\]
such that $f^k_{S^+,S^-}(\bar{x}) = 1$ if and only if the majority of the $k$ closest points to $\bar{x}$ in $S^+ \cup S^-$ are positive. However, the set of $k$ closest points may not always be uniquely defined, as multiple points can have the same distance from $\bar{x}$. To address this, we define $f^k_{S^+,S^-}(\bar{x}) = 1$ if and only if there is a subset $T\subseteq S^+\cup S^-$ of size $k$ such that the majority of points of $T$ belong to $S^+$ and $d_n(\bar x, \bar y) \le d_n(\bar x, \bar z)$ for all $\bar y\in T$ and $\bar z \in (S^+\cup S^-)\setminus T$.
This approach is sometimes referred to as an {\em optimistic view} of $k$-NN classification, as it favors sets that classify $\bar{x}$ as positive when there is ambiguity in the selection of $k$ closest points \cite{DBLP:conf/iscopt/ContardoFRV24}.\footnote{
We focus on the case where \( k \) is odd to ensure that ties occur only when the set of \( k \) closest points is not unique. If \( k \) is even, ties can encompass the entire space (or a significant subset), causing the optimistic tie-breaking rule to classify all points positively, including those in \( S^- \). For example, with \( k=2 \), if \( M^n \) is the real line, \( S^+ \) the odd numbers, and \( S^- \) the even numbers, ties arise everywhere.
}

In some proofs we use the following characterization of the $k$-NN classification function: 

\begin{proposition}
\label{prop_symm}
    (a) We have $f^k_{S^+, S^-}(\bar x) = 1$ if and only if there exist $A\subseteq S^+$ of size $(k+1)/2$  and $B\subseteq S^-$ of size at most $(k-1)/2$ such that $d_n(\bar x, \bar a) \le d_n(\bar x, \bar c)$ for every $\bar a \in A$ and $\bar c\in S^-\setminus B$.

    (b) We have $f^k_{S^+, S^-}(\bar x) = 0$ if and only if there exist $A\subseteq S^-$ of size $(k+1)/2$  and $B\subseteq S^+$ of size at most $(k-1)/2$ such that $d_n(\bar x, \bar a) < d_n(\bar x, \bar c)$ for every $\bar a \in A$ and $\bar c\in S^+\setminus B$.
\end{proposition}

\begin{proof}
Imagine that we start inflating a ball with the center at $\bar x$. As the radius grows, some points of $S^+$ and $S^-$ fall inside the ball. We are waiting until either $(k+1)/2$ points of $S^+$ or of $S^-$ fall inside. If this first happens for $S^+$, we set $f^k_{S^+,S^-}(\bar x) = 1$, and we  set $f^k_{S^+,S^-}(\bar x) = 0$ if it first happens for $S^-$. If this happens simultaneously, we set $f^k_{S^+,S^-}(\bar x) = 1$, favoring positive classification. In other words, we classify $\bar x$ positively if there is a ball that has at least $(k+1)/2$ points of $S^+$ but its interior has at most $(k-1)/2$ points of $S^-$. Likewise, $\bar x$ is classified negatively if there is a ball with at least $(k+1)/2$ points of $S^-$ but with at most $(k-1)/2$ points of $S^+$ (in the interior and the boundary). 
\end{proof}

\section{Problems}\label{sec:problems}

\subsection{Decision problems}

We consider a metric space family $(M,D)$ with $D = \{d_n \mid n>0\}$. We present the different sorts of explanation studied in this paper and their associated decision problems. 

\paragraph{{\bf Abductive explanations}} 
Consider an input vector $\bar x \in M^n$. 
The goal in this case is to find a set $X$ of components over $\{1,\dots,n\}$ that suffice to explain the output of the $k$-NN classification function $f^k_{S^+,S^-}$ on $\bar x$. Intuitively, this means that every input vector $\bar y$ that coincides with $\bar x$ over the components in $X$ is classified in the same way by $f^k_{S^+,S^-}$. We formalize these ideas next using the well-known notion of {\em sufficient reason}. 

Fix an odd integer $k \geq 1$. 
Consider then two sets $S^+,S^- \subseteq M^n$ and an input vector $\bar x \in M^n$. Let $X \subseteq \{1,\dots,n\}$. We call $X$ a {\em sufficient reason for $\bar x$ with respect to $f^k_{S^+,S^-}$}, 
if 
$$f^k_{S^+,S^-}(\bar x) = f^k_{S^+,S^-}(\bar y), \ \ \ \ \ \text{for every $\bar y \in M^n$ that satisfies $\bar x[i] = \bar y[i]$, for each $i \in X$.}$$
The most basic decision problem in this case is verifying if an $X \subseteq \{1,\dots,n\}$ is in fact a sufficient reason for $\bar x$. This leads to the following problem. 

\medskip 

\begin{center}
\fbox{\begin{tabular}{lp{8cm}}
{\small PROBLEM} : & {\sc $k$-Check Sufficient Reason}$(M,D)$ \\
{\small INPUT} : & Two sets $S^+,S^- \subseteq M^n$, a vector $\bar x \in M^n$, an $X \subseteq \{1,\dots,n\}$
%of dimension $d$, 
%\\ & $x \in \{0,1\}^d$ an instance, and $0 \leq \delta \leq 1$
\\ 
{\small OUTPUT} : & {\sc Yes}, if $X$ is a sufficient reason for $\bar x$ with respect to $f^k_{S^+,S^-}$  
\end{tabular}}
\end{center}

\medskip 

%Not all sufficient reasons are equally informative. For instance, $X = \{1,\dots,n\}$ is always a sufficient reason for $\bar x$, but arguably a very uninformative one. It is then natural to look for sufficient reasons that are as small as possible in terms of their cardinality. Formally, if $X$ is a sufficient reason for $\bar x$ with respect to $f^k_{S^+,S^-}$, then we say that $X$ is {\em minimum} if there is no sufficient reason $Y$ for $\bar x$ with respect to $f^k_{S^+,S^-}$ for which $|Y| < |X|$. This leads to our next decision problem. 

Not all sufficient reasons are equally informative. For instance, $X = \{1,\dots,n\}$ is always a sufficient reason for $\bar x$, but arguably a very uninformative one. It is then natural to look for \emph{minimum} sufficient reasons, that is, sufficient reasons that are as small as possible in terms of their cardinality. This is formalized by the next decision problem.
\medskip 

\begin{center}
\fbox{\begin{tabular}{lp{9cm}}
{\small PROBLEM} : & {\sc $k$-Minimum Sufficient Reason}$(M,D)$ \\
{\small INPUT} : & Two sets $S^+,S^- \subseteq M^n$, a vector $\bar x \in M^n$, an integer $\ell>0$
%{\small INPUT} : & Two sets $S^+,S^- \subseteq M^n$, a vector $\bar x \in M^n$, an $X \subseteq \{1,\dots,n\}$
%of dimension $d$, 
%\\ & $x \in \{0,1\}^d$ an instance, and $0 \leq \delta \leq 1$
\\ 
{\small OUTPUT} : & {\sc Yes}, if there is a sufficient reason $X$ for $\bar x$ w.r.t. $f^k_{S^+,S^-}$ with $|X|\leq \ell$
%{\small OUTPUT} : & {\sc Yes}, if $X$ is a minimum sufficient reason for $\bar x$ w.r.t. $f^k_{S^+,S^-}$  
\end{tabular}}
\end{center}

\medskip 

When the problem of checking minimum sufficient reasons is computationally hard, one might be satisfied with finding a {\em minimal} one, i.e., one that does not properly contain another sufficient reason. Formally, if $X$ is a sufficient reason for $\bar x$ with respect to $f^k_{S^+,S^-}$, then $X$ is {\em minimal} if there is no sufficient reason $Y$ for $\bar x$ with respect to $f^k_{S^+,S^-}$ that satisfies $Y \subsetneq X$. Clearly, every minimum sufficient reason is also minimal, but the converse does not hold in general as shown next. 

\begin{example} 
We consider the discrete setting. Assume that $S^+ = \{(0,1,1),(1,0,1),(1,1,1)\}$ and $S^- = \{0,1\}^3 \setminus S^+$. It is easy to see that both sets $\{1,2\}$ and $\{3\}$ of components are sufficient reasons for $\bar x = (0,0,0)$ with respect to $f^k_{S^+,S^-}$, for  $k = 1$. This is because every vector $\bar y$ for which $\bar y[1] = \bar y[2] = 0$, or for which $\bar y[3] = 0$, belongs to $S^-$, and thus $f^k_{S^+,S^-}(\bar x) = f^k_{S^+,S^-}(\bar y) = 0$. Moreover, neither $\{1\}$ nor $\{2\}$ nor $\emptyset$ are sufficient reasons for $\bar x$. Hence, $\{1,2\}$ and $\{3\}$ are minimal sufficient reasons for $\bar x = (0,0,0)$, but only $\{3\}$ is a minimum one. \qed 
\end{example} 
 
This motivates our next decision problem. 

\medskip 

\begin{center}
\fbox{\begin{tabular}{lp{8cm}}
{\small PROBLEM} : & {\sc $k$-Minimal Sufficient Reason}$(M,D)$ \\
{\small INPUT} : & Two sets $S^+,S^- \subseteq M^n$, a vector $\bar x \in M^n$, an $X \subseteq \{1,\dots,n\}$
%of dimension $d$, 
%\\ & $x \in \{0,1\}^d$ an instance, and $0 \leq \delta \leq 1$
\\ 
{\small OUTPUT} : & {\sc Yes}, if $X$ is a minimal sufficient reason for $\bar x$ w.r.t. $f^k_{S^+,S^-}$  
\end{tabular}}
\end{center}

\medskip 

It is easy to observe that a greedy strategy turns a polynomial time algorithm for {\sc $k$-Check Sufficient Reason} into a polynomial time algorithm for {\sc $k$-Minimal Sufficient Reason}.

%As it has been observed several times \cite{DBLP:conf/kr/HuangII021,DBLP:conf/nips/ArenasBOS22}, a polynomial time algorithm for {\sc $k$-Check Sufficient Reason} can be turned into a polynomial time algorithm for {\sc $k$-Minimal Sufficient Reason}. 

\begin{proposition}
\label{prop_check_minimal_reduction}
For any $k, M, D$, the  {\sc $k$-Minimal Sufficient Reason}$(M,D)$ problem reduces in polynomial time to {\sc $k$-Check Sufficient Reason}$(M,D)$.
\end{proposition}

\begin{proof}
If a set is a sufficient reason, then all its supersets are. Hence, to decide if $X\subseteq \{1, \ldots, n\}$ is a minimal sufficient reason, it suffices to check if $X$ is a sufficient reason, and then check, for each subset $X \setminus \{i\}$ obtained by removing one element $i \in X$, that $X \setminus \{i \} $ is not a sufficient reason. 
\end{proof}

\paragraph{{\bf Counterfactual explanations}}  
These are explanations that aim to find what should be changed from an input vector $\bar x$ in order to obtain a different classification outcome. 
Typically, one aims to find counterfactual explanations that are not ``too far'' from $\bar x$, which is formalized by saying that the distance between $\bar x$ and its counterfactual explanation is bounded. 

Fix an odd integer $k \geq 1$.
Given sets $S^+,S^- \subseteq M^n$ and $\bar x \in M^n$, a {\em counterfactual explanation for $\bar x$ with respect to $f^k_{S^+,S^-}$} is a vector $\bar y \in M^n$ with $f^k_{S^+,S^-}(\bar x) \neq f^k_{S^+,S^-}(\bar y)$. We look for counterfactual explanations that are  close  to the vector $\bar x$. This leads to the following decision problem.  

% \textcolor{red}{José: There is a technical nasty problem with the definition as it is (for the continuous version). To see the issue, consider $k=1$. If $\bar{x}$ is classified positively, checking if there exists a counterfactual explanation at distance at most $y$ corresponds to checking if the interior of a voronoi diagram intersects a (closed ball). This implies nasty issues, as the convex set considered is not closed nor open. To avoid the problem I suggest to change the definition of counterfactual explanation, avoiding the optimistic view in the definition as follows:}

% \textcolor{red}{
% Formally, given sets $S^+,S^- \subseteq M^n$, vector $\bar x \in M^n$, and integer $\ell > 0$, a {\em counterfactual explanation for $\bar x$ with respect to $f^k_{S^+,S^-}$ at distance at most $\ell$} is a vector $\bar y \in M^n$ that satisfies $d_n(\bar x,\bar y) \leq \ell$ and there exists a set $T$ of $k$ closest points in $S^+\cup S^-$ such that the classification following $T$ differs from the one in $\bar x$, that is: (i) if $f_{S^+,S^-}(\bar{x})=1$ there exists $T\in S^+\cup S^-$ with $k$ elements such that the majority of points belong to $S^-$; and (ii) if $f_{S^+,S^-}(\bar{x})=0$ there exists $T\in S^+\cup S^-$ with $k$ elements such that the majority of points belong to $S^+$.
% }

% \textcolor{red}{It is unclear to me if this changes something elsewhere...}

%This leads to the following decision problem. 

\medskip 

\begin{center}
\fbox{\begin{tabular}{lp{8cm}}
{\small PROBLEM} : & {\sc $k$-Counterfactual Explanation}$(M,D)$ \\
{\small INPUT} : & Two sets $S^+,S^- \subseteq M^n$, a vector $\bar x \in M^n$, a rational $\ell > 0$
%of dimension $d$,ss
%\\ & $x \in \{0,1\}^d$ an instance, and $0 \leq \delta \leq 1$
\\ 
{\small OUTPUT} : & {\sc Yes}, if there is a 
counterfactual explanation $\bar y$ for $\bar x$ with respect to $f^k_{S^+,S^-}$ such that $d_n(\bar x, \bar y)\leq \ell$
\end{tabular}}
\end{center}

\medskip 

Thus, {\sc Counterfactual Explanation} asks whether it is possible to find a vector $\bar y$ that is relatively close to $\bar x$ and that is classified differently than $\bar x$ under $f^k_{S^+,S^-}$. For instance, in the discrete case this asks if it is possible to ``flip'' the classification of $\bar x$ by ``flipping'' at most $\ell$ of its components. Figure \ref{fig:example-2d} illustrates how counterfactuals look in the continuous setting under the $\ell_2$-distance. 

\definecolor{redFigure}{RGB}{211, 115, 130}
\definecolor{blueFigure}{RGB}{123, 148, 191}
\begin{figure}
    % \begin{subfigure}{0.49\linewidth}
    \centering
    \begin{tikzpicture}[scale=0.5]
    % generate random points
    \pgfmathsetseed{1908} % init random with the year Voronoi published his paper ;)
    \def\pts{}
    \xintFor* #1 in {\xintSeq {1}{\n}} \do{
      \pgfmathsetmacro{\ptx}{.9*\maxxy*rand} % random x in [-.9\maxxy,.9\maxxy]
      \pgfmathsetmacro{\pty}{.9*\maxxy*rand} % random y in [-.9\maxxy,.9\maxxy]
      \edef\pts{\pts, (\ptx,\pty)} % stock the random point
    }
    
    % counter for alternating colors
    \newcount\colorcount
    \colorcount=0
    
    % draw the points and their cells
    \xintForpair #1#2 in \pts \do{
      \edef\pta{#1,#2}
      \begin{scope}
        \xintForpair \#3#4 in \pts \do{
          \edef\ptb{#3,#4}
          \ifx\pta\ptb\relax % check if (#1,#2) == (#3,#4) ?
            \tikzstyle{myclip}=[];
          \else
            \tikzstyle{myclip}=[clip];
          \fi;
          \path[myclip] (#3,#4) to[half plane] (#1,#2);
        }
        \clip (-\maxxy,-\maxxy) rectangle (\maxxy,\maxxy); % last clip
        
        % alternate between light red and light blue
        \ifodd\colorcount
          \fill[redFigure] (#1,#2) circle (4*\biglen); % light red cell
          \fill[draw=red,very thick] (#1,#2) circle (1.4pt); % red point
        \else
          \fill[blueFigure] (#1,#2) circle (4*\biglen); % light blue cell
          \fill[draw=blue!70,very thick] (#1,#2) circle (1.4pt); % blue point
        \fi
        \global\advance\colorcount by 1
      \end{scope}
    }
    \pgfresetboundingbox
    \draw (-\maxxy,-\maxxy) rectangle (\maxxy,\maxxy);

    \node[fill=green!60,  thick, circle, inner sep=0pt, outer sep=0pt, minimum width=5pt, draw] (x) at (2, 1.3)  {}; 
    
    \node at (2.3, 0.8) {\footnotesize Input point $\bar{x}$};

\node[fill=yellow!60,  thick, circle, inner sep=0pt, outer sep=0pt, minimum width=5pt, draw] (y) at (0.9, 2.1)  {}; 

  \node at (0.7, 2.6) {\footnotesize Optimal counterfactual $\bar{y}$};

\draw[->, thick]  (x) -- (y);

\end{tikzpicture}

% \end{subfigure}
% \begin{subfigure}{0.49\linewidth}
% \centering
% \input{sections/images/voronoi_fig2}
% \caption{$\bar{y}_1$ is optimal for $k=1$, and $\bar{y}_3$  for $k=3$. The Voronoi cell of point $\bar{r}$ is refined into 3 subcells, according to which is the second closest point.}
% \end{subfigure}
    \caption{Illustration of  minimum distance counterfactual explanations over $\mathbb{R}^2$ in the $\ell_2$ metric, with $k = 1$. Blue (red) areas are classified negatively (positively).}\label{fig:example-2d}
\end{figure}
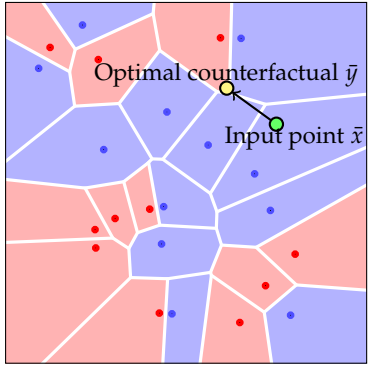

\subsection{Computation problems} 

For simplicity, we focus our complexity analysis on the decision problems introduced above. However, in the context of explainable AI, it is, of course, more important to compute an optimal explanation (if one exists). Our study, however, also sheds light on the computation problems. In fact, as shown in the explainability literature, the hardness of a decision problem often implies hardness for its associated computation problem \cite{DBLP:conf/nips/ArenasBOS22}. Conversely, the tractability of a decision problem often implies that the associated computation problem can be solved in polynomial time. 
In this paper, we show that all our tractability results extend from decision to computation.

\section{Minimum Sufficient Reasons}
\label{sec:msr}
In this section, we show that {\sc $k$-Minimum Sufficient Reason} is \np-hard for both the continuous and the discrete setting, for every odd integer $k \geq 1$. In the continuous case, hardness holds regardless of the norm being used. 

\begin{theorem} \label{theo:minimum-sr}
The following statements hold: 
\begin{enumerate}
\item {\sc $k$-Minimum Sufficient Reason}$(\{0,1\},D_H)$ is \np-hard under  many-to-one reductions, for every fixed odd integer $k \geq 1$. 
    \item {\sc $k$-Minimum Sufficient Reason}$(\mathbb{R},D_p)$ is \np-hard under many-to-one reductions for every fixed odd integer $k \geq 1$ and integer $p > 0$. 
\end{enumerate}
\end{theorem} 

\begin{proof} 
We first consider the discrete case in item (1). We start by providing some intuition regarding our construction. Assume that $k=1$. We reduce from the well-known \np-complete \emph{Vertex Cover} problem:
given an undirected graph $G=(V,E)$, and an integer $\ell\geq 0$, check whether there is a \emph{vertex cover} $C$ in $G$ of size $|C|\leq \ell$. Recall that a vertex cover is a subset of nodes $C\subseteq V$ such that every edge in $E$ has an endpoint in $C$. Let $G$ be a graph with vertex set $V=\{1,\dots,n\}$ and edge set $E=\{e_1,\dots,e_m\}$. Our encoding
considers $n$-dimensional vectors (each component corresponds to a vertex). The vector $\bar x$ is the zero vector, which will be classified positively. The negative example set $S^-$ corresponds to the edges $e_1,\dots,e_m$: each $e_j$ is encoded in the natural way, a component is $1$ iff it is an endpoint of $e_j$. Hence, $S^-$ is simply an encoding of the incidence matrix of $G$. The positive example set $S^+$
contains all the vectors resulting from vectors in $S^-$ by flipping a $1$ to a $0$. Then each vector in $S^-$ has two "guards" in $S^+$ that are closer to $\bar x$. Using these "guards", it follows that if $C$ is a vertex cover, then $C$ is a sufficient reason. On the other hand, if $C$ is not a vertex cover, an edge/vector $e_j$ not being covered directly provides a
counterexample to $C$ being a sufficient reason. We formalize this reduction below.

Given an instance of Vertex Cover, we construct an instance of  
{\sc $1$-Minimum Sufficient Reason}$(\{0,1\},D_H)$ as follows. Assume that $V=\{1,\dots,n\}$ and $E=\{e_1,\dots,e_m\}$, for $n,m\geq 1$. Suppose that $n$ is the vector dimension, and take $\bar x= (0,\dots,0) \in \{0,1\}^n$. For each $j\in \{1,\dots,m\}$, we define the vector $\bar{y}_{j}\in\{0,1\}^n$ such that $\bar{y}_{j}[i] = 1$ if $e_j$ is incident to the vertex $i$, and $\bar{y}_{j}[i] = 0$, otherwise. We define $S^{-} = \{\bar{y}_{j}\mid j\in \{1,\dots,m\}\}$.
For $\bar{y}_{j}$, we denote by $\bar{y}_{j}^1$ and $\bar{y}_{j}^2$, the vectors obtained from $\bar{y}_{j}$ by flipping the first and second component with value $1$, respectively, to $0$, and keeping the remaining vector components unchanged. We finally define $S^+ = \bigcup_{j}\{\bar{y}_{j}^1, \bar{y}_{j}^2\}$. Note that $f^{1}_{S^+,S^-}(\bar{x})=1$, as $d_H(\bar{x},\bar{y})=1$ for every $\bar{y}\in S^{+}$, while $d_H(\bar{x},\bar{y})=2$ for every $\bar{y}\in S^{-}$.

We claim that there is a vertex cover $C$ with $|C|\leq \ell$ if and only if there is a sufficient reason $X$ for $\bar{x}$ with respect to $f^{1}_{S^+,S^-}$ such that $|X|\leq \ell$. 
Suppose first there is such a vertex cover $C\subseteq \{1,\dots,n\}$. We show that $C$ is a sufficient reason. Let $\bar{z}\in\{0,1\}^n$ be an arbitrary vector such that $\bar{z}[i]=\bar{x}[i]=0$, for all $i\in C$.
We show that for every $\bar{y}\in S^-$, there exists $\bar{y}'\in S^+$, such that $d_H(\bar z, \bar y) > d_H(\bar z, \bar{y}')$, and hence $f^{1}_{S^+,S^-}(\bar{z})=1$ as required. Let $\bar{y}_{j}\in S^-$. Since $C$ is a vertex cover, one of the endpoints of $e_j$ is in $C$ and then there is $i\in C$ such that $\bar{y}_{j}[i]=1$. Pick one such $i$, and define $\bar{y}_j'\in S^+$ as the vector resulting from $\bar{y}_{j}$ by flipping $\bar{y}_{j}[i]$ to $0$. 
Since $\bar{y}_{j}$ and $\bar{y}_{j}'$ only differ in one component $i\in C$, we have
$$d_H(\bar z,\bar{y}_{j}) > d_H(\bar z,\bar{y}_{j}') \ \ \iff\ \ |\bar{z}[i]-\bar{y}_{j}[i]| > \ |\bar{z}[i]-\bar{y}_{j}'[i]|\ \ \iff \ \ 1>0$$ 
and then the condition holds.

Assume now that $X\subseteq\{1,\dots,n\}$ is a sufficient reason with $|X|\leq \ell$. We show that $X$ is a vertex cover of $G$. By contradiction, suppose there is an edge $e_j\in E$ whose endpoints are not in $X$. Then the vector $\bar{y}_{j}$ satisfies that $\bar{y}_{j}[i]=\bar{x}[i]=0$ for all $i\in X$. As $\bar{y}_{j}\in S^-$, it follows that $f^{1}_{S^+,S^-}(\bar{y}_{j})=0$, which is a contradiction. 

The hardness of {\sc $k$-Minimum Sufficient Reason}$(\{0,1\}, D_H)$, when $k \geq 3$, is deferred to Section~\ref{sec:further}, where an even stronger result is proved (Theorem~\ref{theo:sigma2p-msr-discrete}).

Now we turn to item (2). The construction for the continuous case with $k\geq 1$ and $\ell_p$-norms is a direct adaptation of the previous idea. Now each of the vectors $e_j$ in $S^-$ must be represented by $(k+1)/2$ sufficiently close vectors: we increment the two components with value $1$ by small quantities. In $S^+$ we add suitable "guards" as before. %See Figure \ref{fig:example-reduction-msr} for an example.

 Given an instance of Vertex Cover, we construct an instance of {\sc $k$-Minimum Sufficient Reason}$(\mathbb{R},D_p)$ as follows. Assume that $V=\{1,\dots,n\}$ and $E=\{e_1,\dots,e_m\}$, for $n,m\geq 1$. Suppose that $n$ is the vector dimension, and take $\bar x= (0,\dots,0) \in \mathbb{R}^n$. Choose $(k+1)/2$ numbers such that $1/2 > \varepsilon_1 >\cdots>\varepsilon_{(k+1)/2} >0$. For each $j\in \{1,\dots,m\}$, and $h\in \{1,\dots,(k+1)/2\}$, we define the vector $\bar{y}_{j,h}\in\mathbb{R}^n$ such that $\bar{y}_{j,h}[i] = 1+\varepsilon_h$ if $e_j$ is incident to the vertex $i$, and $\bar{y}_{j,h}[i] = 0$, otherwise. We also define: $$S^{-} \ = \ \{\bar{y}_{j,h}\mid j\in \{1,\dots,m\}, h\in \{1,\dots,(k+1)/2\}\}.$$ 
For $\bar{y}_{j,h}$, we define $\bar{y}_{j,h}^1$ and $\bar{y}_{j,h}^2$ as the vectors obtained from $\bar{y}_{j,h}$ by changing the first and second components equal to $1+\varepsilon_h$ to $\varepsilon_h$, respectively, and keeping all other components unchanged. We finally define: $$S^+ \ =\ \bigcup_{j,h}\{\bar{y}_{j,h}^1, \bar{y}_{j,h}^2\}.$$ Note that $f^{k}_{S^+,S^-}(\bar{x})=1$, as for every $\bar{y}^a_{j,(k+1)/2}\in S^{+}$ and $\bar{y}_{j',h}\in S^{-}$: 
$$\lVert \bar{y}^a_{j,(k+1)/2} \rVert_p^p \ = \ \varepsilon_{(k+1)/2}^p + (1+\varepsilon_{(k+1)/2})^p \ < \ 2(1+\varepsilon_{h})^p \ = \ \lVert \bar{y}_{j',h} \rVert_p^p.$$

Suppose first  $C\subseteq \{1,\dots,n\}$ is a vertex cover. We claim that $C$ is a sufficient reason for $\bar{x}$ with respect to $f^{k}_{S^+,S^-}$. Let $\bar{z}\in\mathbb{R}^n$ be an arbitrary vector such that $\bar{z}[i]=\bar{x}[i]=0$, for all $i\in C$.
We show that there is an inyective function $g:S^{-}\to S^{+}$ such that for every $\bar{y}\in S^{-}$, we have $\lVert \bar z - \bar y\rVert_p^p > \lVert \bar z - g(\bar y)\rVert_p^p$, and hence $f^{k}_{S^+,S^-}(\bar{z})=1$ as required. Let $\bar{y}_{j,h}\in S^-$. Since $C$ is a vertex cover, one of the endpoints of $e_j$ is in $C$ and then there is $i\in C$ such that $\bar{y}_{j,h}[i]=1+\varepsilon_h$. Pick one such $i$, and define $g(\bar{y}_{j,h})\in S^+$ as the vector resulting from $\bar{y}_{j,h}$ by changing $\bar{y}_{j,h}[i]$ to $\varepsilon_h$. Indeed, the function $g$ is inyective: 
\begin{itemize}
    \item for $\bar{y}_{j,h},\bar{y}_{j',h'}\in {S^-}$ with $j\neq j'$, we have $g(\bar{y}_{j,h})\neq g(\bar{y}_{j',h'})$ as their non-zero components differ; and  
    \item for $\bar{y}_{j,h},\bar{y}_{j,h'}\in {S^-}$ with $h\neq h'$, we have $g(\bar{y}_{j,h})\neq g(\bar{y}_{j,h'})$ as $\varepsilon_{h}\neq \varepsilon_{h'}$. 
\end{itemize}
Since $\bar{y}_{j,h}$ and $g(\bar{y}_{j,h})$ only differ in one component $i\in C$, we have
$$\lVert \bar z - \bar{y}_{j,h}\rVert_p^p \ > \ \lVert \bar z - g(\bar{y}_{j,h})\rVert_p^p \ \ \iff \ \ |\bar{z}[i]-(1+\varepsilon_h)|^p \ > \ |\bar{z}[i]-\varepsilon_h|^p \ \ \iff \ \ (1+\varepsilon_h)^p \ > \ \varepsilon_h^p,$$
and then $g$ satisfies the required conditions.

Assume now that $X\subseteq\{1,\dots,n\}$ is a sufficient reason. We show that $X$ is a vertex cover of $G$. By contradiction, suppose there is an edge $e_j\in E$ whose endpoints are not in $X$. Then the vector $\bar{y}_{j,1}$ satisfies that $\bar{y}_{j,1}[i]=\bar{x}[i]=0$ for all $i\in X$. We claim that $f^{k}_{S^+,S^-}(\bar{y}_{j,1})=0$, which is a contradiction. This follows from the fact that $\lVert \bar{y}_{j,1} - \bar{y}_{j,h} \rVert_p^p < \lVert \bar{y}_{j,1} - \bar{y}_{j',h'}^a \rVert_p^p$, for every $h\in \{1,\dots,(k+1)/2\}$ and $\bar{y}_{j',h'}^a\in S^+$. Indeed, since $\bar{y}_{j',h'}^a$ has exactly  one component with value $1+\varepsilon_{h'}$, and $\bar{y}_{j,1}$ has two components with value $1+\varepsilon_1$, we have:
$$\lVert \bar{y}_{j,1} - \bar{y}_{j',h'}^a \rVert_p^p \ \geq \ (1+\varepsilon_1 - \varepsilon_{h'})^p \ \geq \ 1 \ \geq \ 2\left(\frac{1}{2}\right)^p \ > \ 2(\varepsilon_1 - \varepsilon_h)^p \ = \ \lVert \bar{y}_{j,1} - \bar{y}_{j,h} \rVert_p^p.$$
This finishes the proof of the theorem. 
\end{proof} 

This resolves the complexity of one of our main problems across all the settings examined in the paper. In the next two sections, we address the remaining problems in the continuous setting, specifically exploring how the choice of the metric affects their tractability. To this end, we focus on two of the most commonly studied metrics: $\ell_2$ and $\ell_1$.

\section{The Continuous Setting Based on the 
$\ell_2$-distance}
\label{sec:l2}
It turns out that in case of the $\ell_2$-norm, all mentioned problems, apart from  {\sc $k$-Minimum Sufficient Reason}, are tractable. The main reason is that in the case of the $\ell_2$-norm, an inequality $d_n(\bar{x},\bar{a})\le d_n(\bar{x},\bar{c})$ is the \emph{linear} inequality in $\bar x$ given by $(\bar{a}-\bar{c})^{\top}\bar{x}\ge \frac{1}{2}(\bar{a}-\bar{c})^{\top}(\bar{a}+\bar{c})$. To see the correctness of this inequality, observe that with the $\ell_2$-norm, the set of equidistant points to $\bar{a}$ and $\bar{c}$ is a hyperplane $H$; see Figure~\ref{fig:VoronoiL1L2.L2}. The vector $\bar{a}-\bar{c}$ is orthogonal to $H$, and thus $H$ is defined by the equality $(\bar{a}-\bar{c})^\top x=b$ for some $b$. To find $b$, we simply observe that the average point $\frac{1}{2}(\bar{a}+\bar{c})$ belongs to $H$, and hence $b=\frac{1}{2}(\bar{a}-\bar{c})^\top (\bar{a}+\bar{c})$. It is worth noticing that other norms behave differently, with the set of equidistant points not in a hyperplane (see Figure~\ref{fig:VoronoiL1L2.L1}).

By Proposition~\ref{prop_symm}, this gives a representation of the set $\{\bar x\in\mathbb{R}^n\mid  f_{S^+,S^-}(\bar{x})=1\}$ as a union of at most $|S^+\cup S^-|^{2k} $ many polyhedra, one for each pair of sets $A$ and $B$ in the statement of the proposition. This is a polynomial number of polyhedra in the input. These polyhedra are explicitly given, as we can describe them by a system of linear inequalities in polynomial time. Analogously, the set $\{\bar x\in\mathbb{R}^n \mid f_{S^+,S^-}(\bar{x})=0\}$ is a union of polynomially many ``open polyhedra'', that is, sets of solutions to a 
system of \emph{strict} linear inequalities.

% \begin{figure}
%     \centering
%     \begin{subfigure}[t]{0.45\textwidth}
%         \centering
%          \includegraphics[width=0.5\textwidth, trim={350 450 880 270},clip]{sections/images/VoronoiL1L2.pdf}
%         \caption{Two Voronoi cells for the $\ell_2$-norm. The separating hyperplane is shown as a  dark-red line.}
%         \label{fig:VoronoiL1L2.L2}
%     \end{subfigure}%
%     ~
%     \begin{figure}
%         \centering
%          \includegraphics[width=0.5\textwidth, trim={215  65 1015 655},clip]{sections/images/VoronoiL1L2.pdf}
%         \caption{Two Voronoi cells for the $\ell_1$-norm.}
%         \label{fig:VoronoiL1L2.L1}
%     \end{figure}
%     %\caption{Caption place holder}
%     %\label{fig:VoronoiL1L2}
% \end{figure}

\begin{figure}
    \centering
    \captionbox{Set of equidistant points between $\bar{a}$ and $\bar{c}$ for the $\ell_2$-norm. %The separating hyperplane is shown as a  dark-red line.
    \label{fig:VoronoiL1L2.L2}}[0.4\textwidth]{
    \includegraphics[width=0.25\textwidth, trim={350 450 880 270},clip]{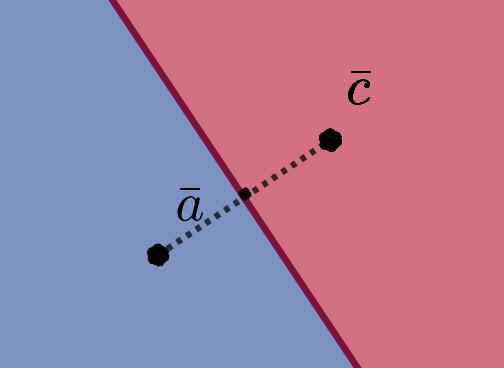}}
    ~$\qquad$
    \captionbox{Set of equidistant points between $\bar{a}$ and $\bar{c}$ for the $\ell_1$-norm.\label{fig:VoronoiL1L2.L1}}[0.4\textwidth]{
    \includegraphics[width=0.25\textwidth, trim={215  65 1015 655},clip]{sections/images/VoronoiL1L2.pdf}}
    %\caption{Caption place holder}
    %\label{fig:VoronoiL1L2}
\end{figure}

\paragraph{{\bf Abductive explanations}}
We start by showing tractability of {\sc $k$-Check Sufficient Reason}. By Proposition \ref{prop_check_minimal_reduction}, this implies tractability of  {\sc $k$-Minimal Sufficient Reason} for the $\ell_2$-norm, which in turn implies that minimal sufficient reasons can be computed in polynomial time in this case. 

\begin{proposition} \label{prop:cont-check-sr}
The problem {\sc $k$-Check Sufficient Reason}$(\mathbb{R},D_2)$, and hence also {\sc $k$-Minimal Sufficient Reason}$(\mathbb{R},D_2)$, can be solved in polynomial time 
for every fixed odd integer $k \geq 1$. 
\end{proposition} 

\begin{proof}
    Assume first that $f_{S^+,S^-}^k(\bar{x})=0$. Let $ X\subseteq \{1,\ldots,n\}$ and  consider the affine subspace
    $$U(X,\bar{x}) := \{\bar{y} \in \mathbb{R}^n \mid \bar{x}[i] = \bar{y}[i], \text{for every $i\in X$}\}.$$
    Then $X$ is \emph{not} a sufficient reason for $\bar{x}$ if and only if $U(X,\bar{x})$  intersects the set $\{\bar y\in\mathbb{R}^n \mid f_{S^+,S^-}(\bar{y})=1\}$. By Proposition \ref{prop_symm}, this set is a union of polynomially many polyhedra. It remains to check if our affine subspace intersects one of these polyhedra. The intersection of an affine subspace and a polyhedron is a polyhedron, and checking emptiness of a polyhedron is equivalent to linear programming which thus can be done in polynomial time~\cite{schrijver_theory_1998}.

   In the case $f_{S^+,S^-}^k(\bar{x})=1$, we have to check, whether our affine subspace intersects the set $\{\bar y\in\mathbb{R}^n  \mid f_{S^+,S^-}(\bar{y})=0\}$. This time,  by Proposition \ref{prop_symm}, this set is a union of sets of solutions to systems of strict linear inequalities. The same argument as in the previous case reduces our problem to the emptiness problem for an intersection of an affine subspace with an open polyhedron. The latter is just the feasibility problem for systems of linear equalities and strict linear inequalities, which can be reduced to linear programming (with non-strict inequalities), solving our problem in polynomial time. Namely, let $S$ be a system of linear equalities and strict linear inequalities.  Consider a system $\widehat{S}$ of non-strict linear inequalities over variables of $S$ and a new variable $\varepsilon$,  obtained by turning every strict inequality $l > 0$ of $S$ into a non-strict inequality $l \ge \varepsilon$. Feasibility of $S$ is equivalent to existence of a feasible solution to $\widehat{S}$ with positive $\varepsilon$. To find out if the latter is true, it is enough to find the optimal solution to the problem of maximizing $\varepsilon$ subject to  $\widehat{S}$.
   %, which can be done in polynomial-time by polynomial-time solvability of linear programming.
\end{proof}

As a corollary, we obtain the following: 

\begin{corollary}
     \label{coro:cont-minimal-sr-comp}
 Consider the setting $(\mathbb{R},D_2)$ and let $k \geq 1$ be any odd integer. There is a polynomial time algorithm that, given sets $S^+,S^- \subseteq \mathbb{R}^n$ and a vector $\bar x \in \mathbb{R}^n$, computes a minimal sufficient reason $X$ for $\bar x$ with respect to $f^k_{S^+,S^-}$. 
\end{corollary} 

\paragraph{{\bf Counterfactual explanations}}
Tractability of the {\sc $k$-Counterfactual Explanation} problem is proved similarly to {\sc $k$-Check Sufficient Reason}, but this time we use polynomial-time solvability of convex quadratic programming~\cite{kozlov_polynomial_1980}. 
A very similar algorithm is described (although without a theoretical analysis) by  Magagnini et al.~\cite{magagniniNearestNeighborsCounterfactuals2025}.

\begin{theorem} \label{theo:l2-cont-counterfactual}
The problem {\sc $k$-Counterfactual Explanation}$(\mathbb{R},D_2)$ can be solved in polynomial time for every fixed odd integer $k \geq 1$.  
%A counter-factual explanation can be found in polynomial time, if one exists.
\end{theorem} 

\begin{proof}
 In the {\sc $k$-Counterfactual Explanation}$(\mathbb{R},D_2)$ problem, given $\bar{x} \in \mathbb{R}^n$ and $\ell > 0$, the goal is to check if the ball $B_\ell(\bar{x}) = \{\bar{y} \in \mathbb{R}^n \mid \|\bar{y} - \bar{x}\|_2 \leq \ell\}$ contains a vector with a different $f^k_{S^+, S^-}$ value than $\bar{x}$. If $f^k_{S^+, S^-}(\bar{x}) = 0$, this reduces to checking whether $B_\ell(\bar{x})$ intersects the set $\{\bar{y} \in \mathbb{R}^n \mid f^k_{S^+, S^-}(\bar{y}) = 1\}$, which, by Proposition \ref{prop_symm}, is a union of polynomially many polyhedra.  
Thus, the problem reduces to determining whether $B_\ell(\bar{x})$ intersects a given polyhedron $P$. This can be solved via convex quadratic programming, which minimizes a positive definite quadratic form under a system of non-strict linear inequalities and is solvable in polynomial time due to Kozlov, Tarasov, and Khachiyan~\cite{kozlov_polynomial_1980}. Specifically, we minimize $q(\bar{y}) = \|\bar{x} - \bar{y}\|_2^2$ subject to constraints defining $P$. The answer to {\sc $k$-Counterfactual Explanation}$(\mathbb{R},D_2)$ is {\sc Yes} if and only if the minimum value is at most $\ell^2$, with the optimal solution providing the counterfactual explanation.

Similarly, when $f^k_{S^+, S^-}(\bar{x}) = 1$, the problem reduces to checking whether the ball $B_\ell(\bar{x})$ intersects a given open polyhedron, defined as the solution set to a system of strict linear inequalities. The argument from the previous paragraph requires modification because the algorithm in \cite{kozlov_polynomial_1980} assumes non-strict inequalities in the constraints.  
To address this, we first check whether $P$ is empty, reducing the problem to linear programming as described in the proof of Proposition~\ref{prop:cont-check-sr}. If $P$ is non-empty (since otherwise, there is nothing left to check), we construct a polyhedron $\widehat{P}$ by converting all strict inequalities of $P$ into non-strict ones. Note that $P$ corresponds to the interior of $\widehat{P}$.
We claim that $P$ intersects $B_\ell(\bar x)$ if and only if $\widehat{P}$ intersects the \emph{interior} of $B_\ell(\bar x)$. The latter can be reduced to a problem of minimizing a convex quadratic objective subject to (closed) polyhedron, which in turn can be solved in polynomial time with the techniques by Kozlov, Tarasov, and  Khachiyan~\cite{kozlov_polynomial_1980}. Indeed, consider the problem of minimizing $q(\bar y) = \| \bar x - \bar y\|_2^2$ subject to $\bar{y}\in \widehat{P}$ (which is a system of non-strict inequalities, as required in the algorithm). The answer is {\sc Yes} if and only if the optimum is strictly less than $\ell^2$.

We now establish the claim. First, assume that $P$ intersects $B_\ell(\bar{x})$. We show that $P$ (and hence $\widehat{P}$) also intersects the interior of $B_\ell(\bar{x})$. If $P$ has a point on the boundary of the ball, it must also have a point inside, as $P$ is open and any boundary point of the ball has interior points arbitrarily close to it.  
Now assume that $\widehat{P}$ intersects the interior of $B_\ell(\bar{x})$. Since $P$ is the interior of $\widehat{P}$, it is non-empty, and $\widehat{P}$ is a full-dimensional polyhedron. By Proposition 2.1.8 in \cite{hiriart-urruty_fundamentals_2001}, every point in a full-dimensional polyhedron has interior points arbitrarily close to it. Thus, if $\widehat{P}$ intersects the interior of the ball, $P$, as the interior of $\widehat{P}$, must also intersect the interior of the ball. 
\end{proof}

By extending the techniques used in the proof of Theorem \ref{theo:l2-cont-counterfactual}, we can 
conclude that in the continuous setting, a counterfactual explanation can be computed in polynomial time, provided one exists, when $\ell_2$-distances are used.
\begin{corollary}
     \label{coro:cont-counterfactual}
 Consider the setting $(\mathbb{R},D_2)$ and let $k \geq 1$ be any odd integer. There is a polynomial time algorithm that, given sets $S^+,S^- \subseteq \mathbb{R}^n$, vector $\bar x \in \mathbb{R}^n$, and rational $\ell > 0$, it computes a counterfactual explanation for $\bar x$ with respect to $f^k_{S^+,S^-}$ at distance at most $\ell$ in case there exists one. 
\end{corollary}

\begin{proof}
This reduces to finding a point in the intersection of the ball $B_\ell(\bar x)$ and an open polyhedron $P$, if this intersection is non-empty. If it is non-empty, then  $\widehat{P}$ intersects the interior of $B_\ell(\bar x)$. We can find a point $\bar{y}$ in this intersection by minimizing the quadratic form  $q(\bar y) = \| \bar x - \bar y\|_{2}^{2}$ subject to $\widehat{P}$. Suppose $\bar y$ lies on the boundary of $\widehat{P}$ (otherwise, we are done). That is, some inequalities, defining $\widehat{P}$, turn into equalities on $\bar y$. We now need to find a direction that points to the interior of $\widehat{P}$ from  this point.  More precisely, our task is to find a vector $\beta$ such that $\langle \alpha, \beta\rangle > 0$ for every inequality $\langle \alpha, \bar y\rangle \ge c$ which turns into equality on $\bar y$. After such $\beta$ is found, we just need to move from $\bar y$ along $\beta$ a little bit so that we still inside $B_\ell(\bar x)$. Finding such $\beta$ is reducible to finding a solution to a system of strict linear inequalities. In turn, the latter can be reduced in polynomial time to linear programming as explained in the proof of Proposition \ref{prop:cont-check-sr}.
\end{proof}

\paragraph{{\bf The parameterized complexity of Counterfactual Explanation for the $\ell_2$-distance.}} Running time of our algorithms  for problems  {\sc $k$-Minimal Sufficient Reason}$(\mathbb{R},D_2)$ and {\sc $k$-Counterfactual Explanation}$(\mathbb{R},D_2)$  depends exponentially on $k$. More specifically, it is $n^{O(k)}$, where $n$ is the size of the dataset. This comes from the fact that in our algorithms we essentially go through all possible $k$-sized sets of points of the dataset that could be $k$ closest neighbors.

This raises a question -- is there a fixed-parameter tractable algorithm for our problems in the case of $\ell_2$-distance, where $k$ is considered as a parameter? That is, are there algorithms running in time $f(k)\cdot n^c$, for some computable function $f$ and constant $c$? We show that this is unlikely for the counterfactual explanations.

\begin{theorem}
    \label{thm_w1hard}
   The problem  {\sc $k$-Counterfactual Explanation}$(\mathbb{R},D_2)$ is W$[1]$-hard \footnote{The class W$[1]$ is the analog of $\np$ in parameterized complexity. For further details see \cite{cygan2015parameterized}.}
   when $k$ is a parameter, and \np-complete when $k$ is part of the input.
\end{theorem}
\begin{proof}
Both results follow from the following lemma.

\begin{lemma}
    \label{lem_reduction_clique}
    There exists a polynomial-time many-to-one reduction from the $k$-RegClique problem (given a regular graph $G$, decide, whether it has a $k$-clique) to the  {\sc $(2k-1)$-Counterfactual Explanation}$(\mathbb{R},D_2)$.
\end{lemma}
Indeed, $k$-RegClique is known to be both W$[1]$-complete~\cite[Theorem 13.4]{cygan2015parameterized} when $k$ is a parameter and \np-complete when $k$ is a part of the input~\cite[Problem GT.20, page 194]{garey1979computers}.

In the rest of the proof, we establish Lemma \ref{lem_reduction_clique} (the fact that {\sc $k$-Counterfactual Explanation}$(\mathbb{R},D_2)$ belongs to $\np$ is immediate).
For readability, throughout the proof we omit the bar notation and write vectors as $x,x_i,y,0,\dots$. For our reduction, we first require the following lemma that provides an embedding of the nodes of our graph into an Euclidean space that respects the connections between the nodes. The resulting vectors are actually $0,1$-vectors so we formulate our lemma for the Hamming distance.

\begin{lemma}
\label{lem_embedding}
    Let $G$ be a regular graph with $n$ nodes and degree $d$. Let $m = n^2 + n + d - 5$. In $O(n^3)$-time, we can assign to each node of $G$ a vector in $\{0, 1\}^m$ such that:
    \begin{itemize}
        \item[a)] all $n$ vectors have Hamming weight $2(n + d - 3)$;
        \item[b)] vectors that correspond to any two distinct nodes of $G$ that are connected by an edge have Hamming distance $2(n + d - 3)$;
        \item[c)] vectors that correspond to any two distinct nodes of $G$ that are not connected by an edge  have Hamming distance $2(n + d - 1)$.
    \end{itemize}
\end{lemma}
\begin{proof} For simplicity, let us first explain how to satisfy just conditions b) and c). We can do this using $n^2$ coordinates that will be split into $n$ blocks of size $n$. Both the blocks and coordinates in each block are indexed by the nodes of $G$.

We define the vector of a node $u$ as follows. In each of its blocks that corresponds to a vertex, different from $u$, we put the ``one-hot encoding'' of $u$ (the vector that has $1$ in the $u$-coordinate and $0$ in the rest of the coordinates). Now, in the $u$-block we put the vector that has 1s exactly in coordinates that correspond to neighbors of $u$ (thus, it will have $d$ ones).

Take any two distinct nodes $u,v$ and their corresponding vectors $f_u, f_v$. In the blocks that correspond to vertices, distinct from both $u$ and $v$, both $f_u$ and $f_v$ have a single 1 in distinct position. These blocks thus contribute $2(n - 2)$ to the Hamming distance. Now, consider a $u$-block or a $v$-block. In these blocks, one of the vectors $f_u, f_v$ have $d$ ones, and the other just a single one. If $u,v$ are connected in $G$, then both blocks have a common one so that they contribute $2(d - 1)$ to the Hamming distance. Now, if $u,v$ are not connected, in the $u$-block and $v$-block with have $d$ ones and 1 separate one, contributing $2(d +1)$ to the Hamming distance. This establishes conditions b) and c).

To obtain condition a), observe that, in the current construction, all vectors have the same number of 1s, namely $n - 1+ d$. It thus suffices to add $2(n + d - 3) - (n - 1 + d) = n +d -5$ coordinates where all vectors have 1s.
\end{proof}

We now explain how we get from a $d$-regular graph $G$ to an instance of the  {\sc $(2k - 1)$-Counterfactual Explanation}$(\mathbb{R},D_2)$ problem. We first give an argument for the case when points in a dataset can have multiplicities. In the end we will explain how to get rid of this assumption. 

Namely, we take $0,1$-vectors from Lemma \ref{lem_embedding}, assigned to nodes of $G$, all classified positively. We also add the all-0 vector to the dataset, classified negatively and with multiplicity $k$. Due to conditions a), b), c), the all-0 point has the $\ell_2$-distance $\alpha = \sqrt{2(n + d - 3)}$ from all the other points of the dataset. Any two positive points of the data set that correspond to two connected nodes of $G$ also have the $\ell_2$-distance equal to $\alpha$ between them. At the same time, points, corresponding to nodes, not connected by an edge, are all at distance $\beta =  \sqrt{2(n + d - 1)}$ from each other. 

In our instance of the  {\sc $(2k - 1)$-Counterfactual Explanation}$(\mathbb{R},D_2)$ problem, the question will be: is there a point at distance at most $R$ (a number that will be defined later) from the all-0 vector $x = (0, \ldots, 0)$ where $f^{2k-1}_{S^+,S^-}$ takes the positive value? Observe that at $x$, the $(2k-1)$-NN takes the negative value -- precisely at $0$, there is a negative point with multiplicity $k$, and positive points are at some positive distance from $0$.

As the point $0$ has multiplicity $k$, in order for $f^{2k-1}_{S^+,S^-}(y)$ to be positive for a point $y$,    there have to exist $k$ positive points $x_1, \ldots, x_k$ in the dataset such that $y$ is closer or at the same distance to all these points as to $0$. The minimal distance from $0$ at which such $y$ exists is formally defined in the following definition.

\begin{definition}
    Let $x_1, \ldots, x_k$ be $k$ points in an Euclidean space $\mathbb{R}^m$ for some $m$.  Define $r(x_1, \ldots, x_k)$ to be the minimum of $\|y\|_2$ over all $y\in\mathbb{R}^m$ satisfying:
    \begin{equation}
    \label{eq_l2conditions}
        \|y - x_1 \|_2 \le \|y\|_2, \ldots, \|y - x_k \|_2 \le \|y\|_2.
    \end{equation}
    (if no such $y$ exists, we set $r(x_1, \ldots, x_k) = +\infty$).
\end{definition}

The answer to the counterfactual explanation problem is thus determined by whether there exist $k$ positive points $x_1, \ldots, x_k$ in the dataset with $r(x_1, \ldots, x_k) \le R$. We now relate this with the existence of the $k$-clique in the initial graph $G$. 

If a $k$-clique exists, the corresponding $k$ points $x_1, \ldots, x_k$ of the dataset will all be at distance $\alpha$ from each other and from $0$. As we will see later, $r(x_1, \ldots, x_k) = \alpha\sqrt{\frac{k}{2(k + 1)}}$ in this case. Now, when no $k$-clique exists, we will argue that $r(x_1, \ldots, x_k)$ is slightly bigger than  $\alpha\sqrt{\frac{k}{2(k + 1)}}$ for any $k$ positive points $x_1, \ldots, x_k$. This will happen because  $0$ is still at distance $\alpha$ from $x_1, \ldots, x_k$, and distances between $x_1, \ldots, x_k$ are all at least $\alpha$ with at least one distance being at least $\beta > \alpha$ (for any $k$ nodes there will be at least two nodes not connected by an edge, so their points will be at a larger distance). The following lemma justifies this reasoning.

\begin{lemma}
\label{lem_geometric}
    Let $x_1, \ldots, x_k$ be $k$ points in an Euclidean space $\mathbb{R}^m$ for some $m$.
    Further, let $0 < \alpha \le \beta$. Then the following hold:
    \begin{itemize}
        \item[a)] if the $\ell_2$-distance between any two points in $\{0, x_1, \ldots, x_k\}$ is precisely $\alpha$, then  
        \[r(x_1, \ldots, x_k) \le \alpha \sqrt{\frac{k}{2(k + 1)}}\]

         \item[b)] if  the $\ell_2$-distance from $0$ to any point in $\{x_1, \ldots, x_k\}$ is precisely $\alpha$, the distance between any two points in $\{x_1, \ldots, x_k\}$ is at least $\alpha$, and for at least one pair of points in $\{x_1, \ldots, x_k\}$ the distance is at least $\beta$,  then 
        \[r(x_1, \ldots, x_k) \ge \alpha \sqrt{\frac{k}{2(k + 1 - \delta)}}, \qquad \text{where } \delta = \frac{\beta^2 - \alpha^2}{k\alpha^2}.\]
        
    \end{itemize}
    \end{lemma}   
    \begin{proof}
    Within the proof, we omit the subscript $\|\|_2$ in the notation for the $\ell_2$-norm.
    
    Let us start with a). For any $i, j\in\{1, \ldots, k\}$, we have:
    \[\langle x_i, x_j\rangle = \frac{\langle x_i, x_i\rangle + \langle x_j, x_j\rangle - \langle x_i - x_j, x_i - x_j\rangle }{2} = \begin{cases}
        \alpha^2 & i = j, \\
        \alpha^2/2 & i \neq j.
    \end{cases}\]
        We take $y = (x_1 + \ldots + x_k)/(k + 1)$, the center of mass of the regular simplex, formed by $0, x_1, \ldots, x_k$. It is of equal distance to all the points $0, x_1, \ldots, x_k$, thus satisfying \eqref{eq_l2conditions}. The square of this distance is computed as:
        \[\langle \frac{x_1 + \ldots + x_k}{k + 1}, \frac{x_1 + \ldots + x_k}{k + 1}\rangle = \frac{k \alpha^2 + (k^2 - k) \alpha^2/2}{(k +1)^2} = \frac{k \alpha^2}{2(k + 1)}, \]
        giving us the required upper bound.

We now establish b). Let $y$ be any point, satisfying \eqref{eq_l2conditions}. Rewriting these conditions in terms of dot-products, we obtain
\begin{equation}
    \label{eq_l2conditions2}
    \langle y - x_i, y - x_i \rangle \le \langle y, y\rangle \iff \langle y, x_i \rangle \ge \|x_i\|^2/2  = \alpha^2/2\qquad i = 1, \ldots, k.
\end{equation}

        Define $z = (x_1+ \ldots + x_k)$. We lower bound $\|y\|$ by the length of the projection of the vector $y$ to the direction of the vector $z$, that is,
        \[\|y\| \ge \frac{|\langle y, z\rangle|}{\|z\|}.\]
        Let us first explain why $z \neq 0$. Indeed, by \eqref{eq_l2conditions2}, the dot-product $\langle y, x_i\rangle$ is strictly positive for every $i = 1, \ldots, k$ as $\|x_i\| \ge \alpha > 0$. Hence, $\langle y, z\rangle$ is also strictly positive as a sum of positive values. As a result, we obtain:
         \begin{align*}
            \|y\| \ge \frac{|\langle y, z\rangle|}{\|z\|} = \frac{\langle y, x_1\rangle + \ldots + \langle y, x_k\rangle}{\sqrt{\langle x_1+ \ldots + x_k, x_1 + \ldots + x_k\rangle}}
        \end{align*}
      The numerator in the last expression is bounded from below by $k\alpha^2/2$ due to \eqref{eq_l2conditions2}. We now upper bound the denominator. The expression $\langle x_1+ \ldots + x_k, x_1 + \ldots + x_k\rangle$ has $k$ terms of the form $\|x_i\|^2$ for $i = 1, \ldots, k$ that contribute precisely $k\alpha^2$ to the sum. Additionally, we have $k^2 - k$ dot-products of the form $\langle x_i, x_j\rangle$ for $i\neq j$. Each of them can be written as
      \[\langle x_i, x_j\rangle = \frac{\langle x_i, x_i\rangle + \langle x_j, x_j\rangle - \langle x_i - x_j, x_i - x_j\rangle }{2} = \alpha^2 - \|x_i - x_j\|^2/2.\]
      By the conditions of the lemma, each of these products is bounded from above by $\alpha^2/2$, but for at least one of them, the upper bound is even better, namely $\alpha^2 - \beta^2/2$, that is, smaller by $(\beta^2 - \alpha^2)/2$. We thus can upper bound the product $\langle x_1+ \ldots + x_k, x_1 + \ldots + x_k\rangle$ as follows:
      \[\langle x_1+ \ldots + x_k, x_1 + \ldots + x_k\rangle \le k \alpha^2 + (k^2 - k)\alpha^2/2 - (\beta^2 - \alpha^2)/2.\]
      Overall, we obtain the following lower bound on $\|y\|$:
\begin{align*}
      \|y\| &\ge \frac{k\alpha^2/2}{\sqrt{k \alpha^2 + (k^2 - k)\alpha^2/2 - (\beta^2 - \alpha^2)/2}} = \alpha \cdot \sqrt{\frac{k \cdot k\alpha^2}{2 \cdot 2 \cdot (k \alpha^2 + (k^2 - k)\alpha^2/2 - (\beta^2 - \alpha^2)/2)}} \\
      &=  \alpha \sqrt{\frac{k\cdot k \alpha^2}{2 \cdot  ((k + 1)k \alpha^2 ´-  (\beta^2 - \alpha^2))}} 
      = \alpha \sqrt{\frac{k}{2(k + 1) - \frac{2(\beta^2 - \alpha^2)}{k\alpha^2}}},
\end{align*}
as required.
    \end{proof}

The above reduction clearly takes polynomial time, but it remains to pick a rational number $R$, satisfying:
\begin{equation}
\label{eq_lambda}
     \lambda_1 =  \alpha \sqrt{\frac{k}{2(k + 1)}} \le R <  \alpha \sqrt{\frac{k}{2(k + 1 - \delta)}} = \lambda_2, \qquad \delta = \frac{\beta^2 - \alpha^2}{k\alpha^2}.
\end{equation}
To avoid technicalities, we slightly modify the construction so that $\lambda_1$ becomes a rational number and we can set $R = \lambda_1$. Namely, just repeat every coordinate $T$ times in all vectors of the dataset, for $T$ to be defined later. All relative distances will be multiplied by $\sqrt{T}$. After that, the value of $\lambda_1$ will become:
\[\lambda_1 =\sqrt{T\cdot 2(n + d - 3)} \sqrt{\frac{k}{2(k + 1)}} =
\sqrt{(n + d - 3) T \frac{k}{k+1}}.\]
Choosing $T = (n + d - 3) k (k +1)$, we get
\begin{equation}
\label{eq_rad}
    \lambda_1 = (n + d - 3) k
\end{equation}

\medskip

We now explain how to get rid of the multiplicities in our dataset. Instead of just point 0 with multiplicity $k$, we will have 0 and $k - 1$ other negative points $y_1, \ldots, y_{k-1}$ with the following properties:
\begin{itemize}
    \item every $y_i$ is at the same distance $\alpha$ from all positive points of the dataset.
    \item every $y_i$ is $\varepsilon$-close to $0$, for some $\varepsilon > 0$ small enough that $R + \varepsilon < \lambda_2$.
\end{itemize}
Assuming we have constructed such $y_1, \ldots, y_{k - 1}$, we explain why the reduction still works after this modification. Indeed, in order to change the $(2k-1)$-NN classification, we have to become closer (or at the same distance) to some $k$ positive points as to some negative points. If a $k$-clique exists, we can take $k$ corresponding positive points $x_1, \ldots, x_k$, we can become as close to them as to 0 by moving at distance at most $R =  \alpha \sqrt{\frac{k}{2(k + 1)}}$ (as before, the points $0, x_1, \ldots, x_k$ form a configuration as in the item a) of Lemma \ref{lem_geometric}).

Now, assume that no $k$-clique exists. Pick arbitrary $k$ positive points $x_1, \ldots, x_k$. By the same argument as before, we require distance at least $\lambda_2 > R$ to become at least as close to $x_1, \ldots, x_k$ as to $0$. However, we should also argue that we cannot become as close to $x_1, \ldots, x_k$ as to $y_i$ for some $i$, by moving at distance at most $R$. Indeed, $x_1, \ldots, x_k$ correspond to a non-clique, and $y_i$ has distance $\alpha$ to all of them, so $y_i, x_1, \ldots, x_k$ form the same configuration as described in item b)  of Lemma \ref{lem_geometric} (after translation by $-y_i$). Hence, we have to move at least $\lambda_2$ from $y_i$, and hence, at least $\lambda_2 - \varepsilon > R$ from $0$ (because $0$ and $y_i$ are $\varepsilon$-close).

It remains to construct $y_1, \ldots, y_{k - 1}$ with the required properties. Let us first calculate the value of $\varepsilon$ that would work. Recall that $\alpha = \sqrt{T\cdot 2(n + d - 3)}, \beta = \sqrt{T \cdot 2(n + d - 1)}$, so we get the following expression for $\delta$ in \eqref{eq_lambda}:
\[\delta = \frac{2}{k \cdot (n + d - 3)}.\]
Define 
\[\lambda_2' =  \alpha \sqrt{\frac{k}{2(k + 1 - \delta')}}, \qquad \text{where } \delta' = \left(1 - \left(\frac{n^4 - 1}{n^4}\right)^2\right) (k + 1)  \]
Observe that 
\[0 < \delta' = \left(1 - \frac{n^4 - 1}{n^4}\right)\left(1 + \frac{n^4 - 1}{n^4}\right)(k + 1)  \le 2(k + 1)/n^4  < \frac{2}{k \cdot (n + d - 3)} =  \delta,\] where the last inequality is due to the fact that $k, d\le n$, meaning that $\lambda_1 < \lambda_2' < \lambda_2$. The value of $\delta'$ is also chosen so that $\lambda_2'$ is rational. Indeed, $\lambda_1$ is rational, and the fraction $\lambda_1/\lambda_2'$ is equal to:
\[\sqrt{\frac{k + 1 - \delta'}{k + 1}} = \frac{n^4 - 1}{n^4}.\]
Recalling \eqref{eq_rad}, we get that it suffices to set
$\varepsilon = \frac{1}{n^4} < \lambda_2' - \lambda_1 = \frac{k(n + d - 3)}{n^4 - 1}$.

We now construct $y_1, \ldots, y_{k - 1}$.
Recall also that all positive points were 0,1-vectors with the same number of 1s $W$. Let $m$ be the dimensionality of vectors in our dataset. It is equal to the dimensionality of vectors from Lemma \ref{lem_embedding} times $T$. In the end of the argument, we will require a simple bound $m \ge n$. %Recall also that all positive vectors in our dataset are currently 0,1 vectors with the same number of 1s $W$. 
Define
\[\tau = \frac{1}{m^{100}}, \qquad S = 2W m^{100} - m.\]
We add $S(k - 1)$ new coordinates where all previous points of the dataset have $0$s. Each point $y_1, \ldots, y_{k - 1}$ will have value $\tau$ in all $m$ old coordinates. Now, we split new coordinates into $k - 1$ blocks of size $S$. Each $y_i$ for $i = 1, \ldots, k -1$ will additionally have $S$ coordinates of the $i$th block equal to $\tau$, with the rest of the new coordinates equal to $0$.

We need to show two things: (a) all the points $y_1, \ldots, y_{k -1}$ are $\varepsilon$-close to $0$; (b) all of them have the same distance to all positive points of the dataset as $0$. We start with (a). Each $y_1, \ldots, y_{k - 1}$ has $m + S$ coordinates, equal to $\tau$, so the square of its distance to $0$ is:
\[(m + S) \tau^2 = \frac{2Wm^{100}}{m^{200}} \le \frac{2}{m^{99}} < \frac{1}{n^8} = \varepsilon^2,\]
because $m \ge n$.

We now show (b). As we have mentioned, every positive point $x_j$ of the dataset has $W$ 1s in the old $m$ coordinates, and just 0s in the new coordinates. Hence, square of the distance from 0 to it is $W$. The following calculation shows that the square of the distance from each $y_i$ to $x_j$ is also $W$, thus finishing the proof of the theorem:
\begin{align*}
    \|y_i - x_j\|^2 &= (m - W) \tau^2 + W (1 -  \tau)^2 + S \tau^2 = m \tau^2 + W - 2W\tau  + S \tau^2 = W + (m + S) \tau^2 - 2W \tau \\
    &= W + \frac{2W m^{100}}{m^{200}} - \frac{2W}{m^{100}} = W.\qedhere
\end{align*}
\end{proof}

\section{The Continuous Setting Based on the 
$\ell_1$-distance}
\label{sec:l1}
We first show that the positive results for counterfactual explanations under the $\ell_2$-norm do not extend to the $\ell_1$-norm. For abductive explanations, we demonstrate that some favorable properties of the $\ell_2$-norm remain, allowing minimal sufficient reasons to be computed efficiently when $k = 1$. 

\paragraph{{\bf Counterfactual explanations}} 
We show that the {\sc $k$-Counterfactual Explanation} problem for the $\ell_1$-distance is \np-hard even when $S^+$ and $S^-$ are of minimal non-degenerate size.

\begin{theorem} \label{theo:l1-cont-counterfactual}
For every odd integer $k\geq 1$, 
the problem {\sc $k$-Counterfactual Explanation}$(\mathbb{R},D_1)$ is \np-complete even when $|S^+| = |S^-| = (k+1)/2$. 
\end{theorem} 

\begin{proof} 
The \np\ upper bound is straightforward: if a counter-factual explanation exists, it can be found as a solution to a linear program, constructable in polynomial time from the input. 
We now show the $\np$ lower bound. Let us start with the case $k = 1$ (we reduce from $k = 1$ to the general case in the end of the proof).
We reduce from a variation of the knapsack problem, where the goal is to determine whether at least half the total value of all items can fit into a given knapsack. Specifically, we are given $n$ items, each with a weight $w_i$ and value $v_i$ (positive integers for $i = 1, \ldots, n$), and a positive integer $W > 0$, representing the knapsack's maximum weight capacity. The question is whether there exists a subset $T \subseteq \{1, \ldots, n\}$ such that $\sum_{i \in T} w_i \leq W$ and $\sum_{i \in T} v_i \geq (v_1 + \ldots + v_n)/2$. The hardness of this problem arises from a reduction from the partition problem \cite{lewis1997elements}. 
%In the partition problem, we are given $n$ positive integers $x_1, \ldots, x_n$ and asked whether there exists a subset $S \subseteq \{1, \ldots, n\}$ such that $\sum_{i \in S} x_i = (x_1 + \ldots + x_n)/2$. This problem is \np-complete (see Chapter 7 in~\cite{lewis1997elements}). The reduction constructs $n$ objects where $v_i = w_i = x_i$ for each $i$ and sets $W = (x_1 + \ldots + x_n)/2$. 

%We reduce from a version of the knapsack problem where the question is whether we can fit into a given knapsack at least half of the total value of all items. More precisely, we are given $n$ items, and for $i = 1, \ldots, n$, the $i$th object has weight $w_i$ and value $v_i$ that are positive integers. We are also given a positive integer $W > 0$, the maximal weight one can fit into the knapsack. The question is whether there exists a subset $T\subseteq \{1, \ldots, n\}$ such that $\sum_{i\in T} w_i \le W$ and $\sum_{i \in T} v_i \ge (v_1 + \ldots + v_n)/2$. Hardness of this version of the knapsack problem  is due to  a reduction from the partition problem. In the partition problem, we are given $n$ positive integer numbers $x_1,\ldots, x_n$, and the question is whether there exists $S\subseteq \{1, \ldots, n\}$ with $\sum_{i\in S} x_i = (x_1 + \ldots + x_n)$. This problem is \np-complete, see Chapter 7 in~\cite{lewis1997elements}. The reduction creates $n$ objects with $v_i = w_i = x_i$ for the $i$th objects, and $W = (x_1 + \ldots + x_n)/2$.  

We reduce to the {\sc $1$-Counterfactual Explanation}$(\mathbb{R}, D_1)$ problem with $|S^+| = |S^-| = 1$ as follows. The dimension $n$ corresponds to the number of items in the knapsack instance. We set $\bar{x} = \bar{0} \in \mathbb{R}^n$ and the radius $\ell = W$. The sets $S^+ = \{\bar{g}\}$ and $S^- = \{\bar{h}\}$ are defined as:  
\[
\bar{g}_i = w_i, \qquad \bar{h}_i = w_i - \gamma \cdot v_i, \qquad i = 1, \ldots, n,
\]
where $\gamma = 1/(2 \max_i v_i)$ ensures $\gamma \cdot v_i \leq 1/2 < 1$ for all $i = 1, \ldots, n$. Since $w_i$ is a positive integer, we have $0 < \bar{h}_i < \bar{g}_i$, so the interval $[\bar{h}_i, \bar{g}_i]$ lies to the right of 0 and has a length of $\gamma \cdot v_i$. These properties imply $\|\bar{h} - \bar{0}\|_1 < \|\bar{g} - \bar{0}\|_1$, resulting in $f^1_{S^+, S^-}(\bar{0}) = 0$. We are asked if there exists $\bar y$ with $\|y\|_1\le \ell = W$ such that $f^1_{S^+,S^-}(\bar y) = 1$, or, equivalently, $\| \bar h - \bar y\|_1 \ge \| \bar g - \bar y \|_1$. We show that the answer is {\sc Yes} if and only if the original knapsack instance has a solution.

%We reduce to  {\sc $1$-Counterfactual Explanation}$(\mathbb{R},D_1)$ with $|S^+| = |S^-| = 1$ as follows. Dimension will be $n$, the number of items in the knapsack instance we reduce from. We set $\bar x = \bar 0 \in\mathbb{R}^n$ and the radious $\ell = W$. We now define $S^+ = \{\bar g\}, S^- = \{\bar h\}$ by:
%\[\bar g_i = w_i, \qquad \bar h_i = w_i - \gamma \cdot v_i, \qquad i = 1, \ldots, n,\]
%where $\gamma = 1/(2\max_i v_i)$ so that $\gamma \cdot v_i \le 1/2 < 1$ for every $i = 1, \ldots, n$. Since $w_i$ is a positive integer, we have $0 < \bar h_i < \bar g_i$. As a result, the interval $[\bar h_i, \bar g_i]$ lies to the right of $0$, having length $\gamma \cdot v_i$. These observations imply that $\| \bar h - 0\|_1 < \| \bar g - 0 \|_1$, giving us $f^1_{S^+,S^-}(\bar 0) = 0$. 

Assume first that the original knapsack instance has a solution. Define a vector $\bar{y} \in \mathbb{R}^n$ by setting $\bar{y}_i = 0$ for items not placed in the knapsack and $\bar{y}_i = \bar{g}_i = w_i$ for items that are placed. Note that $\|\bar{y}\|_1$ equals the total weight of the items in the knapsack, which does not exceed $W = \ell$. We now need to show that $\|\bar{h} - \bar{y}\|_1 \geq \|\bar{g} - \bar{y}\|_1$, or equivalently:  
\[
\|\bar{h} - \bar{y}\|_1 - \|\bar{g} - \bar{y}\|_1 \ = \ \sum_{i=1}^n (|\bar{h}_i - \bar{y}_i| - |\bar{g}_i - \bar{y}_i|) \ \geq \  0.
\]
It is more convenient to express this inequality in the equivalent form:  
\begin{equation}
\label{eq_length}
\sum_{i=1}^n (|\bar{h}_i - \bar{y}_i| - |\bar{g}_i - \bar{y}_i| + \gamma v_i)/2 \ \geq \ \gamma (v_1 + \ldots + v_n)/2.
\end{equation}
The term $|\bar{h}_i - \bar{y}_i| - |\bar{g}_i - \bar{y}_i|$ represents the distance from $\bar{y}_i$ to the left endpoint of $[\bar{h}_i, \bar{g}_i]$ minus the distance from $\bar{y}_i$ to the right endpoint of $[\bar{h}_i, \bar{g}_i]$. It is minus the length of the interval to the left of it, and plus the length of the interval to the right of it, and the length of the interval in our case is $\gamma v_i$. 
Thus, if $\bar{y}_i = 0$, the left-hand side of \eqref{eq_length} contributes 0; if $\bar{y}_i = \bar{g}_i$ (the right endpoint), it contributes $\gamma v_i$. Therefore, the left-hand side of \eqref{eq_length} is the sum of the values of the items in the knapsack, scaled by $\gamma$. This proves \eqref{eq_length}, as it follows directly from the fact that we began with a feasible solution to the original knapsack problem.

 We now show the other direction. Assume that there exists $\bar y\in\mathbb{R}^n$ such that $\|\bar{y}\|_1 \le \ell = W$ and $\| \bar h - \bar y\|_1 \ge \| \bar g - \bar y \|_1$, with the latter being equivalent to \eqref{eq_length}. Consider again the quantity $|\bar h_i - \bar y_i| - |\bar g_i - \bar y_i|$ as a function of $\bar y_i$. To the right of $\bar g_i = w_i$ it is constant and is equal to the length of the interval. Hence, without loss of generality, $\bar y_i \le \bar g_i$ for every $i = 1, \ldots, n$, as otherwise we can decrease $\|\bar{y}\|_1$ without changing the left-hand side of \eqref{eq_length}. Likewise, to the left of the interval, the quantity in question is also constant and is equal to the minus of the length of the interval, and the minimal absolute value there is $0$. Hence, we may assume that if $\bar y_i$ lies to the left of $[\bar h_i, \bar g_i]$, then $\bar y_i = 0$ for every $i = 1, \ldots, n$. Again, if not, $\|\bar{y}\|_1$ can be decreased without changing the left-hand side of \eqref{eq_length}.

We now establish that, without loss of generality, we may assume $\bar{y}_i = 0$ or $\bar{y}_i = \bar{g}_i$ for every $i = 1, \ldots, n$. First, suppose there are two different components, $\bar{y}_i$ and $\bar{y}_j$, that lie within their respective intervals but are strictly smaller than their right endpoints. Start decreasing $\bar{y}_i$ and increasing $\bar{y}_j$ at the same rate. The terms involving $\bar{y}_i$ and $\bar{y}_j$ in \eqref{eq_length} will begin to decrease and increase, respectively, at twice the same rate, leaving the left-hand side of \eqref{eq_length} unchanged. Similarly, $\|\bar{y}\|_1$ remains constant.  This process continues until either $\bar{y}_i$ reaches $\bar{h}_i$ or $\bar{y}_j$ reaches $\bar{g}_j$. In the first case, we can further decrease $\bar{y}_i$ to 0 without reducing the left-hand side of \eqref{eq_length}. In the second case, $\bar{y}_j$, which was strictly within its interval, becomes equal to its right endpoint. In both scenarios, the number of indices $i$ for which $\bar{y}_i \in (\bar{h}_i, \bar{g}_i)$ strictly decreases.  
We repeat this procedure until at most one index $i$ satisfies $\bar{y}_i \in (\bar{h}_i, \bar{g}_i)$. 

%We know establish that, without loss of generality, we may assume that $\bar y_i = 0$ or $\bar y_i = \bar g_i$ for every $i = 1, \ldots, n$. First, imagine  that we have two different $\bar y_i$ and $\bar y_j$ that are in their intervals but smaller then the right endpoints. Start  decreasing $\bar y_i$ and increasing $\bar y_j$ with the same speed. The terms with $\bar y_i$ and with $\bar y_j$ in \eqref{eq_length} will start to decrease, and correspondingly, increase with twice the same speed, leaving the left-hand side of   \eqref{eq_length} without change, as well as $\|\bar y\|_1$. This happens until either $\bar y_i$ hits $\bar h_i$ or $\bar y_i$ hits $\bar g_j$. In the first alternative, we can further make $\bar y_i$ to be equal to $0$ without decreasing the left-hand side of \eqref{eq_length}. In the second alternative, $\bar y_j$ that have been strictly inside its interval is now equal to its right endpoint. In both cases, the number of $i$ with $\bar y_i \in [\bar h_i, \bar g_i)$ strictly increase. We repeat this procedure until at most one $i$ with $\bar y_i \in [\bar h_i, \bar g_i)$ is left.

Now, if at most one problematic $i$ remains where $\bar{y}_i \in (\bar{h}_i, \bar{g}_i)$, we simply increase $\bar{y}_i$ to $\bar{g}_i$. We claim that the solution remains feasible. Increasing $\bar{y}_i$ raises the left-hand side of \eqref{eq_length}, so the corresponding inequality is still satisfied.  
We now explain why $\|\bar{y}\|_1 \le W = \ell$ continues to hold. This follows from the fact that $W$ is an integer and the length of any interval is less than 1. Before the increase, $\|\bar{y}\|_1 = |\bar{y}_1| + \ldots + |\bar{y}_n|$ was at most $W$. After the increase, the sum becomes integral because $\bar{y}_i = 0$ or $\bar{y}_i = \bar{g}_i = w_i$ for every $i$. Consequently, the total sum cannot exceed $W$, as the increase is too small to make the sum reach $W + 1$.

%Now, if we already have at most one bad $i$ with $\bar y_i \in [\bar h_i, \bar g_i)$, we simply increase it to $\bar g_i$. We claim that we still have a feasible solution. We increase the left-hand size of \eqref{eq_length}, so the corresponding inequality is preserved. We have to explain why we still have $\|\bar y\|_1 \le W = \ell$. It is based on the fact that $W$ is an integer and the length of any interval is smaller than $1$. Before the increase,  $\|y\|_1 = |\bar y_1| + \ldots |\bar y_n|$ was at most $W$. After the increase, the sum is integral, because $\bar y_i = 0$ or $\bar y_i = \bar g_i= w_i$ for every $i$. Hence,  the sum is still at most $W$ because it could not become $W+1$, the increase is too small.

Therefore, it holds that $\bar y_i = 0$ or $\bar y_i = \bar g_i = w_i$ for every $i = 1,\ldots, n$. We place objects with $\bar y_i = w_i$ into the knapsack. The condition $\|\bar{y}\|_1 \le \ell = W$ ensures that the sum of the weights of the objects placed in the knapsack does not exceed the capacity. Now, notice that the left-hand side of \eqref{eq_length} becomes the sum of $\gamma v_i$ for the objects placed in the knapsack, which guarantees that the total value of the objects placed is at least half of the total value of all objects.

\medskip

We only have to show that it is enough to establish the theorem for $k = 1$. This is because  {\sc $1$-Counterfactual Explanation}$(\mathbb{R},D_1)$ for  $|S^+| = |S^-| = 1$ reduces to {\sc $k$-Counterfactual Explanation}$(\mathbb{R},D_1)$ for  $|S^+| = |S^-| = (k+1)/2$ for every odd integer $k\ge 1$. Indeed, consider any input to {\sc $1$-Counterfactual Explanation}$(\mathbb{R},D_1)$ with $|S^+| = |S^-| = 1$. It consists of two singleton sets $S^+ = \{\bar g\}$ and $S^- = \{\bar h\}$ for $\bar g, \bar h\in \mathbb{R}^n$, a vector $\bar x \in\mathbb{R}^n$ and an integer $\ell > 0$. The question is whether there exists $\bar y\in\mathbb{R}^n$ such that $\|\bar x - \bar y\|_1 \le \ell$ and $f^{1}_{S^+,S^-}(\bar x) \neq f^{1}_{S^+,S^-}(\bar y)$. In this special case, we have $f^{1}_{S^+,S^-}(\bar z) = 1$ if $\| \bar z - \bar g\|_1\le \| \bar z - \bar h\|_1$ and  $f^{1}_{S^+,S^-}(\bar z) = 0$ otherwise. In other words,  $f^{1}_{S^+,S^-}(\bar z)$ can be defined as the sign of the difference  $\| \bar z - \bar h\|_1\ - \| \bar z - \bar g\|_1$ (where the sign of a non-negative number is 1 and the sign of a negative number is 0).

Without loss of generality, we may assume that $\bar x = \bar{0}$.
We now transform our input into an input for {\sc $k$-Counterfactual Explanation}$(\mathbb{R},D_1)$ with $|S^+| = |S^-| = (k + 1)/2$. We first add arbitrary $(k-1)/2$ points to $S^+$ and $(k-1)/2$ points to $S^-$, for example:
\begin{align*}
    \bar p_1 =(1, 0, \ldots,0),\,\, \ldots,\,\, \bar p_{\frac{k - 1}{2}} = \left(\frac{k - 1}{2}, 0, \ldots,0\right) &\text{ to } S^+,\\
     \bar p_\frac{k + 1}{2} =\left(\frac{k + 1}{2}, 0, \ldots,0\right),\,\, \ldots,\,\, \bar p_{k-1} = \left(k-1, 0, \ldots,0\right) &\text{ to } S^-.
\end{align*}
We then add one more coordinate to all vectors, where we put $0$ for $\bar x, \bar p_1, \ldots, \bar p_{k-1}$ and $M = 10 (\ell + k)$ for $\bar g, \bar h$. Let $B_\ell$ be the ball of radius $\ell$ around $\bar x = \bar{0} $  (in the space with the new coordinate). Observe that any point among  $\bar p_1, \ldots, \bar p_{k-1}$  is much closer to any point in $B_\ell$ than $\bar g$ and $\bar h$. Hence, for any potential counterfactual explanation $\bar y \in B_\ell$, the closest  $k-1$  points in $S^+\cup S^-$ will be  $\bar p_1, \ldots, \bar p_{k-1}$, among which half is from $S^+$ and half is from $S^-$. Hence, the value $f^k_{S^+, S^-}(\bar y)$ for any $\bar y\in B_\ell$ is determined by whether $\bar g$ or $\bar h$  (with an additional coordinate) is closer to $\bar y$, with value 1 in case of the equality. This means that the answer to the problem is determined by whether the sign of $\|\bar y - \bar h\|_1 - \|\bar y - \bar g\|_1$ is different for some $\bar y \in B_\ell$ from the sign of $\|\bar x - \bar h\|_1 - \|\bar x - \bar g\|_1$. Despite we have an additional coordinate, the answer to this is the same as in the initial {\sc $1$-Counterfactual Explanation} problem. This is because $\bar g$ and $\bar h$ have the same value in this additional coordinate so that it cannot contribute to the difference, meaning that we can restrict ourselves to $\bar y \in B_\ell$ that have $0$ in this additional coordinate.
%
%Therefore, it is the case that $\bar y_i = 0$ or $\bar y_i = \bar g_i = w_i$ for every $i = 1,\ldots, n$. We put into knapsack objects with $\bar y_i = w_i$. The condition $\|y\|_1 \le \ell = W$ guarantees that the sum of weights of the objects that are put is not too big. Now, notice that
%the left-hand side of \eqref{eq_length} becomes the sum of $\gamma v_i$ for $i$ that are put into the knapsack, giving us that the total value of objects that are put is at least half of the total value of all objects.
\end{proof} 

\paragraph{{\bf Abductive explanations}}
Next, we observe that the {\sc 1-Check Sufficient Reason} problem is tractable for the $\ell_1$-norm (together with the {\sc 1-Minimal Sufficient Reason}, by Proposition \ref{prop_check_minimal_reduction}). 

\begin{proposition}
\label{prop_l1_1check}
The problem {\sc $1$-Check Sufficient Reason}$(\mathbb{R},D_1)$, and hence also {\sc $1$-Minimal Sufficient Reason}$(\mathbb{R},D_1)$, is polynomial-time solvable. 
%Moreover, a minimal sufficient reason can be computed in polynomial time in this case.  
\end{proposition}

\begin{proof}
For a given $S^+, S^-\subseteq \mathbb{R}^n$, $\bar x \in\mathbb{R}^n$ and $X\subseteq \{1, \ldots, n\}$, our task is to decide, whether there is a vector, coinciding with $\bar x$ on coordinates in $X$ but differing in the value of $f^1_{S^+, S^-}$.  We will use the following notation: for a vector $\bar v\in\mathbb{R}^n$, we write $\bar v = (\bar v_1, \bar v_2)$, denoting by $\bar v_1$ the projection of $\bar v$ to coordinates from $X$ and by $\bar v_2$ the projection to the remaining coordinates. 

In this notation, our task is to see, if there is $\bar y_2 \in \mathbb{R}^{[n]\setminus X}$ with the property that $f^1_{S^+, S^-}((\bar x_1, \bar y_2)) \neq f^1_{S^+, S^-}((\bar x_1, \bar x_2))$. First, assume that $f^1_{S^+, S^-}((\bar x_1, \bar x_2)) = 0$. By Proposition \ref{prop_symm} for $k  = 1$, we are asked if there exists  $\bar y_2 \in \mathbb{R}^{[n]\setminus X}$ and $\bar a\in S^+$ such that:
\begin{equation}
\label{eq_ac}
\|(\bar x_1, \bar y_2) - (\bar a_1, \bar a_2)\|_1 \le \|(\bar x_1, \bar y_2) - (\bar c_1, \bar c_2)\|_1 \text{ for every } \bar c\in S^-.
\end{equation}
Using linearity of the $\ell_1$-norm under concatenation of vectors, we rewrite \eqref{eq_ac} as follows:
\begin{equation}
\label{eq_ac2}\|\bar x_1  - \bar a_1\|_1 - \|\bar x_1  - \bar c_1\|_1 \le  \|\bar y_2 - \bar c_2\|_1 - \|\bar y_2 - \bar a_2\|_1 \text{ for every } \bar c\in S^-.
\end{equation}
The left-hand side of \eqref{eq_ac2} does not depend on $\bar y_2$. The right-hand side  of \eqref{eq_ac2}, for every $\bar c_2$, attains its maximum at $\bar y_2 = \bar a_2$, by the triangle inequality. This gives the following algorithm  for our problem: for every $\bar a\in  S^+$, check if $\bar y_2 = \bar a_2$ satisfies all inequalities in \eqref{eq_ac}. If for some $\bar a\in S^+$ it does, the set $X$ is not a sufficient reason, otherwise $X$ is a sufficient reason.

The argument for the case $f^1_{S^+, S^-}((\bar x_1, \bar x_2)) = 0$ is exactly the same, with the roles of $S^+$ and $S^-$ swapped and with non-strict inequalities replaced by strict ones.
\end{proof}

As a corollary, we obtain the following: 

\begin{corollary}
     \label{coro:cont-minimal-sr-comp-l1}
Consider the setting $(\mathbb{R},D_1)$. There is a polynomial time algorithm that, given sets $S^+,S^- \subseteq \mathbb{R}^n$ and a vector $\bar x \in \mathbb{R}^n$, it computes a minimal sufficient reason $X$ for $\bar x$ with respect to $f^1_{S^+,S^-}$. 
\end{corollary} 

We finish this section by showing that, starting from $k = 3$, the problems 
  {\sc $k$-Check Sufficient Reason}$(\mathbb{R},D_1)$ and  {\sc $k$-Minimal Sufficient Reason}$(\mathbb{R},D_1)$  become intractable.

\begin{theorem} \label{theo:continuous-check-sr-k3} For every odd integer $k \geq 3$, 
the problem
 {\sc $k$-Check Sufficient Reason}$(\mathbb{R},D_1)$ is \conp-complete, and the problem  {\sc $k$-Minimal Sufficient Reason}$(\mathbb{R},D_1)$ is $\np$-hard  under Turing reductions.
\end{theorem} 
\begin{proof}

Clearly, the {\sc $k$-Check Sufficient  Reason}$(\mathbb{R},D_1)$ is in $\conp$. Indeed, if a given set of coordinates is not a sufficient reason, then a counter-example can be found as a solution to a linear program, constructible from the input in polynomial time.

    We now show the \conp-hardness of the {\sc $k$-Check Sufficient Reason}$(\mathbb{R},D_1)$ problem. We reduce from the complement to the  partition problem, where we are given $n$ positive integers $v_1, \ldots, v_n$, and the question is whether there exists $T\subseteq \{1, \ldots, n\}$ such that $\sum_{i \in T} v_i = \sum_{i \notin T} v_i$. \np-completeness of the partition problem (as well as $\conp$-completeness of its complement) is standard~\cite{lewis1997elements}.

    Let us first assume for simplicity that the same point can appear in the data set multiple times (we will explain later how to get rid of this simplification). Given $n$ positive integers $v_1, \ldots, v_n$ (an input to the partition problem), consider an $n$-dimensional data set $S$, consisting of the  points
    \begin{align*}
        \bar{\alpha} &= (0, 0, \ldots, 0) \qquad \text{classified 1, multiplicity 1,}\\
        \bar{\beta} &= (2v_1, 2v_2, \ldots, 2 v_n) \qquad  \text{classified 1, multiplicity $(k - 1)/2$,}\\
        \bar{\gamma} &= (v_1, v_2, \ldots, v_n) \qquad \text{classified 0, multiplicity $(k + 1)/2$.}
    \end{align*}
    The size of the dataset is $k+1$, and we have a point $\bar{\gamma}$ with multiplicity $(k+1)/2$, classified as $0$. Hence, 
we have $f_{S^+, S^-}^k(\bar{x}) = 0$ if and only if $\bar{\gamma}$ is not the farthest point to $\bar{x}$ in the dataset, i.e., if
$\|\bar{\gamma}- \bar{x}\|_1 < \|\bar{\alpha}- \bar{x}\|_1$ or  $\|\bar{\gamma} - \bar{x}\|_1 < \|\bar{\beta} - \bar{x}\|_1$ (here we require strict inequalities because of our optimistic view on $k$-NN classification). Observe that $f_{S^+, S^-}^k(\bar{\alpha}) = 0$ because $\|\bar{\gamma}\|_1 < \|\bar{\beta}\|_1$. We claim that there exists $\bar{y}\in\mathbb{R}^n$ such that $f_{S^+, S^-}^k(\bar y) = 1$ if and only if the initial partition problem has a solution. In other words, we claim that the empty set of coordinates is \emph{not} a sufficient reason for $\bar x = \bar 0$ if and only if the initial partition problem has a solution.

The condition $f_{S^+, S^-}^k(\bar y) = 1$ can be rewritten as
\[\begin{cases}
     \|\bar \alpha- \bar y\|_1 \le \|\bar \gamma - \bar y\|_1 \\
      \|\bar \beta- \bar y\|_1 \le \|\bar \gamma - \bar y\|_1
\end{cases},\]
or, equivalently,
\[
  \begin{cases}
             \sum\limits_{i = 1}^n (|\bar y_i - 0| - |\bar y_i - v_i|) \le 0 \\
            \sum\limits_{i = 1}^n (|\bar y_i - 2v_i| - |\bar y_i - v_i|) \le 0.
        \end{cases}
\]
We claim that we can restrict our search to $\bar y\in\mathbb{R}^n$ where $\bar y_i \in\{0, 2v_i\}$ for every $i = 1, \ldots, n$. Indeed, consider the two  expressions that appear in our conditions as functions of $\bar y_i$:
\[A_1(\bar y_i) = |\bar y_i - 0| - |\bar y_i - v_i|, \qquad A_2(\bar y_i) = |\bar y_i - 2v_i| - |\bar y_i - v_i|.\]
If $\bar y_i$ is to the left of $v_i$, then we can decrease $A_1$ without changing $A_2$ by moving $\bar y_i$ to $0$. In turn, if $\bar y_i$ is to the right of $v_i$, we can decrease $A_2$ without changing $A_1$ by moving $\bar y_i$ to $2v_i$.

Now, take any ``restricted'' $\bar y$. Let $T$ denote the set of $i \in\{1, \ldots, n\}$ such that $\bar y_i = 2v_i$ (for $i\notin T$, we have $\bar y_i = 0$). Then our conditions can be rewritten as: 
\[
        \begin{cases}
            \sum\limits_{i \in  T}^n v_i  +  \sum\limits_{i \notin  T}^n (-v_i)\le 0 \\
            \sum\limits_{i \in  T}^n(-v_i)  +  \sum\limits_{i \notin  T}^n v_i
           \le 0.
        \end{cases}
\]
These conditions are equivalent to $ \sum_{i \in  T}^n v_i =  \sum_{i \notin  T}^n v_i$, as required in the partition problem.

We now explain how to get rid of points with multiplicities. There are $k +1$ points in our dataset (taking into account multiplicities). We add $k + 1$  auxiliary coordinates. For $i = 1, \ldots, k + 1$, for the $i$th point of the dataset, we put 1 to the $i$-th auxiliary coordinate and 0 to all the other auxiliary coordinates. For example, for $k = 3$, we obtain the following dataset:
\begin{align*}
    &(1, 0, 0, 0, \bar{\alpha})  \qquad \text{classified 1},\\
        &(0, 1, 0, 0, \bar{\beta})  \qquad \text{classified 1},\\
        &(0, 0, 1, 0, \bar \gamma)  \qquad \text{classified 0},\\
                &(0,0 , 0, 1, \bar \gamma)  \qquad \text{classified 0}.    
\end{align*}
We consider $\bar x = (0, \ldots, 0)$ as the input and ask, whether the set of auxiliary coordinates is a sufficient reason for $\bar x$. The same argument as before shows that this is the case if and only if the initial instance of the partition problem has a solution. Indeed, we are not allowed to change 0s in the $k+ 1$ auxiliary coordinates of $\bar x$. Hence, these coordinates contribute 1 to distances from $\bar y$ to all the points. This does not change the differences of distances on which our argument is based on.

Finally, we explain why this reduction implies that  {\sc $k$-Minimal Sufficient Reason}$(\mathbb{R},D_1)$ is $\np$-hard (under Turing reductions) as well. The crucial thing is that the set of $X$ for which we ask, whether $X$ is a sufficient reason, has constant-size, namely, $k+1$. It remains to note that $X$ is a sufficient reason if and only if it has a subset which is a minimal sufficient reason. Hence, if we could solve   {\sc $k$-Minimal Sufficient Reason}$(\mathbb{R},D_1)$ in polynomial-time, we could solve  the partition problem in polynomial-time by running the algorithm for the minimal sufficient reason problem on all subsets of $X$.
    \end{proof}

\section{Results on the Discrete Setting}\label{sec:discrete}
We consider the metric space family $(\{0,1\},D_H)$, where $D_H$ is the Hamming distance on $\{0,1\}^n$, for every $n > 0$. Our results for the discrete setting resemble those for the continuous setting under the $\ell_1$-distance; however, our proof techniques differ. %Additionally, we show that, in the discrete setting, the {\sc $k$-Check Sufficient Reason} problem is \np-hard for every odd integer $k \geq 3$. Our hardness result holds even when the input subset of components $X$ is the empty set. This directly implies hardness for {\sc $k$-Minimal Sufficient Reason} and for the problem of computing a minimal sufficient reason, when $k\geq 3$.

\paragraph{{\bf Counterfactual explanations}}
We prove that {\sc $k$-Counterfactual Explanation} is intractable for every odd integer $k \geq 1$, employing different techniques than those used for the \np-hardness of the problem in the continuous setting under the $\ell_1$-distance. Unlike the continuous case, where Theorem \ref{theo:l1-cont-counterfactual} shows \np-hardness even when $|S^+| = |S^-| = (k+1)/2$, this does not hold in the discrete setting. We address this by constructing a reduction with an unbounded size for $S^+$.  

\begin{theorem} \label{theo:discrete-counterfactual}
The problem {\sc $k$-Counterfactual Explanation}$(\{0,1\},D_H)$ is \np-complete for every odd integer $k \geq 1$.  
\end{theorem} 

The $\np$ upper bound is straightforward, so we focus on the hardness part of 
Theorem \ref{theo:discrete-counterfactual}. We first establish the intractability of the following intermediate problem.
Fix an integer $p \geq 0$. 
The input to the problem is given by a Boolean matrix $B$ of dimension $m \times n$ and an integer $\ell \leq n$. We want to know if it is possible to find a set $T \subseteq \{1,\dots,n\}$ of size $|T| \leq \ell$ such that, after {\em flipping} every column from $B$ whose index is in $T$, the {\em weight} of at least $m-p$ rows in the resulting matrix is at most $|T|-1$. 
Here, ``flipping a column'' means replacing the 0s with 1s and vice versa in that column's elements. 
Additionally, the weight of a row is defined as the number of 1s it contains.
We call this problem {\em $p$-Boolean Matrix Column Flipping}, or simply {\sc $p$-BMCF}.   

\begin{proposition} \label{prop:bmcf}
    {\sc $p$-BMCF} is \np-complete, for every integer $p \geq 0$. 
\end{proposition}

\begin{proof}
    Fix $p \geq 0$. 
    We reduce from the following modified version of the {\em Vertex Cover} problem: Given an undirected graph $G = (V,E)$ and an integer $\ell$, is there  a set $V' \subseteq V$ with $|V'| \leq \ell$ such that at least $|E|-p$ edges have an endpoint in $V'$? Notice that for $p = 0$ this is exactly the Vertex Cover problem. The fact that this problem is \np-complete for $p > 0$ is obtained by an easy reduction from Vertex Cover (simply extend the input graph with $p$ additional isolated edges). 

    We define a Boolean matrix $A$ as the transpose of the incidence matrix of $G$, i.e., $A[i,j] = 1$ if and only if edge $e_i$ is incident to vertex $v_j$. We then define our matrix $B$ by extending $A$ with a column of all 1s on the right. We consider the instance $(B,\ell+1)$ for the problem 
    {\sc $p$-BMCF}. 

    Suppose first that $T \subseteq \{1,\dots,|V|+1\}$ is a solution for the instance $(B,\ell+1)$ of {\sc $p$-BMCF}. Notice, by construction, that $|T| > 0$ (as we can assume without loss of generality that $G$ contains at least $p+1$ edges, and thus $B$ contains at least $p+1$ rows). Then there exists a solution $T' \subseteq \{1,\dots,|V|+1\}$ for $(B,\ell+1)$ such that $|T| = |T'|$ and $T'$ contains index $|V|+1$. Indeed, if $T$ does not contain index $|V|+1$ then we can obtain such a solution $T'$ by simply removing any index from $T$ and adding index $|V|+1$. Then $|T'| \leq \ell+1$ and $|T''| \leq \ell$, for $T'' = T' \setminus \{|V|+1\}$. We claim that $U = \{v_i \mid i \in T''\}$ satisfies that at least 
    $|E|-p$ edges of $G$ have an endpoint in $U$. 
In fact, take an arbitrary edge $e \in E$ such that row $e$ has weight at most $|T'|-1$ after flipping the columns in $T'$. We know that there exist at least $|E| - p$ such edges. We claim that, for each such an edge $e$, it is the case that $U \cap e \neq \emptyset$. Assume, on the contrary. Then the weight of row $e$ after flipping the columns in $T'$ is 
    $$2 \, - \, |U \cap e| + (|T'|-1) - |U \cap e| \ = \ |T'|+1.$$
    This is a contradiction.  

    Suppose, in turn, that $G$ has a vertex set $U \subseteq V$ with $|U| \leq \ell$ such that at least $|E|-p$ edges have an endpoint in $U$. Let us define $T = \{i \mid v_i \in U\}$ and $T' = T \cup \{|V|+1\}$. Then $|T'| = |U| + 1 \leq \ell + 1$. We claim that $T'$ is a solution for the instance $(B,\ell+1)$ of {\sc $p$-BMCF}. Take an arbitrary edge $e \in E$ that is covered by $U$. In a similar fashion as above, the weight of row $e$ after flipping the columns in $T'$ is 
    $$2 \, - \, |U \cap e| + (|T'|-1) - |U \cap e| \ = \ 
    |T'|+1 - 2|U \cap e|.$$
    But $|U \cap e| > 0$ since $U$ covers $e$, which implies that the weight of row $e$ after flipping the columns in $T'$ is at most $|T'|-1$. 
\end{proof}

\begin{proof}[Proof of Theorem \ref{theo:discrete-counterfactual}] Fix an odd integer $k = 2p + 1$, for $p \geq 0$. 
    The proof is by reduction from {\sc $p$-BMCF}. Suppose that the input to {\sc $p$-BMCF} is given by a Boolean matrix $B$ of dimension $m \times n$ and an integer $\ell \leq n$. 
    From the proof of Proposition \ref{prop:bmcf}, we can assume that $B$ contains no repeated rows, each row of $B$ contains at least two 0s, and $m \geq p+1$. From $B$, we define the input $(S^+,S^-,\bar x,\ell)$ for {\sc $k$-Counterfactual Explanation} as follows: 
 %   \begin{itemize} 
 %       \item $|B| \geq k$, 
 %       \item every row of $B$ contains a 1, and 
 %       \item $B$ contains no repeated rows. 
 %   \end{itemize} 
 \begin{itemize}
     \item Each row $b$ of $B$ defines a tuple in $S^+$ of dimension $n+p+1$. 
    This tuple is obtained by extending $b$ with $p+1$ 
    0s on the right. 
    \item We define $S^-$ to contain all tuples of the form $\{0\}^{n+j} \times \{1\} \times \{0\}^{p-j}$, for $1 \leq j \leq p+1$. Observe that there are $p+1$ such tuples.  
    \item Finally, $\bar x = \{1\}^{n+p+1}$.
    %\item Finally, $\ell' = \ell + p$
 \end{itemize}
    Notice that the $p+1$ closest points to $\bar x$ in the set $S^+ \cup S^-$ belong to $S^+$, as every row of $B$ contains at least two 0s. 
    Therefore, $f^k_{S^+,S^-}(\bar x) = 1$.  

    %For $\bar y \in \{0,1\}^n$, we define $B_{\bar y}$ to be the Boolean matrix of dimension $m \times n$ such that $B_{\bar y}[i,j] = 1$ if and only if the $j$th coordinate of the $i$th row of $B$ coincides with the $j$th coordinate of $\bar y$. Hence, the distance between $\bar y$ and the $i$th element of $S^+$ is equal to $n$ minus the weight of the $i$th row of $B_{\bar y}$. Observe also that $B_{\bar x} = B$, and that if $\bar y$ is obtained from $\bar x$ by flipping from 1 to 0 all elements indexed in $T$, for $T \subseteq \{1,\dots,n\}$, then $B_{\bar y}$ is obtained from $B_{\bar x} = B$ by precisely flipping the columns indexed in $T$.  

    Assume first that the input $(B,\ell)$ to {\sc $p$-BMCF} has a solution $T \subseteq \{1,\dots,n\}$. Then $|T| \leq \ell$. Let $\bar y \in \{0,1\}^{n+p}$ be the point that is obtained from $\bar x$ by flipping from 1 to 0 precisely 
    the elements indexed in $T$. Then the Hamming distance between $\bar x$ and $\bar y$ is $|T| \leq \ell$, and the Hamming distance between $\bar y$ and any point in $S^-$ is $n-|T| + p$. Let $B_T$ be the matrix that is obtained from $B$ after flipping the columns in $T$. We know that at least $m-p$ of the rows in $B_T$ have weight at most $|T|-1$. 
    Take any such row $b$, and let $\phi(b)$ be its corresponding element in $S^+$. It follows that the  distance between $\bar y$ and $\phi(b)$ is at least $n - (|T|-1) + p + 1 = n - |T| -p + 2 > n-|T| + p$. Therefore, $S^+$ contains at most $p$ elements which are at distance at most $n-|T| + p$ from $\bar y$. Hence, $f^k_{S^+,S^-}(\bar y) = 0$.

    %In turn, the weight of any row in $B_{\bar y}$ is at most $d-1$, which implies that the distance between $\bar y$ and any element in $S^+$ is at least $n - (d - 1) > n - d$. Thus, $f^1_{S^+,S^-}(\bar x) = 0$.     

    Assume, in turn, that the input $(S^+,S^-,\bar x,\ell)$ to the problem {\sc Counterfactual Explanation} has a solution given by $\bar y$, where $\bar y$ is at distance $d \leq \ell$ from $\bar x$. Let $T \subseteq \{1,n+p+1\}$ be the set of indices for which $\bar y$ takes value 0 and $T' = T \cap \{1,\dots,n\}$. We claim that $T'$ is a solution for $(B,\ell)$ with respect to {\sc $p$-BCMF}. Notice first that the Hamming distance between $\bar y$ and an arbitrary element in $S^-$ is at most 
    $$(n-|Y'|) + (p+1) - (d - |Y'|) + 1 \ = \ n + p - d + 2.$$ This means that there cannot be $p+1$ elements in $S^+$ whose Hamming distance from $\bar y$ is at most $n+p-d + 2$ (as, otherwise, $f^k_{S^+,S^-}(\bar y) = 1$, which is a contradiction). Suppose, for the sake of contradiction, that after flipping the columns in $T'$ and obtaining matrix $B_{T'}$ there are at least $p+1$ rows whose weight is at least $|Y'|$. Take any such row $b$. Then the distance between $\bar y$ and the element in $S^+$ that uniquely represents $b$ is at most
    $$(n-|Y'|) + (p+1) - (d - |Y'|) \ = \ n + p - d + 1 \ < \ n + p - d + 2.$$
    This is our desired contradiction. 
\end{proof}

\paragraph{{\bf Abductive explanations}}
We start by establishing that the problems {\sc $1$-Check Sufficient Reason} and {\sc $1$-Minimal Sufficient Reason}
are tractable in the discrete setting. 

\begin{proposition} \label{prop:discrete-check-sr-k1}
The problem {\sc $1$-Check Sufficient Reason}$(\{0,1\},D_H)$, and hence also {\sc $1$-Minimal Sufficient Reason}$(\{0,1\},D_H)$, is polynomial-time solvable. 
%Moreover, a minimal sufficient reason can be computed in polynomial time in this case.  
\end{proposition}

\begin{proof}
We show that {\sc $1$-Check Sufficient Reason}$(\{0,1\},D_H)$ can be solved in polynomial time. Let $\bar{x}\in\{0,1\}^n$, $S^+,S^-\subseteq \{0,1\}^n$, and $X\subseteq\{1,\dots,n\}$. Suppose that $X$ is not a sufficient reason. Assume without loss of generality that $f^{1}_{S^+,S^-}(\bar x)=1$. Then there is a vector $\bar{z}\in\{0,1\}^n$ with $\bar{z}[i]=\bar{x}[i]$ for every $i\in X$ such that $f^{1}_{S^+,S^-}(\bar z)=0$. For each $\bar{y}\in S^{-}$, let $\bar{y}_X$ be the vector such that $\bar{y}_X[i]= \bar{x}[i]$, for every $i\in X$, and $\bar{y}_X[i]= \bar{y}[i]$, for every $i\not\in X$. We claim that without loss of generality, the vector $\bar{z}$ can always be chosen to be a vector in $\{\bar{y}_X\mid \bar{y}\in S^{-}\}$. Indeed, since $f^{1}_{S^+,S^-}(\bar z)=0$, there is $\bar{y}\in S^-$ such that $d_H(\bar{z},\bar{y})<d_H(\bar{z},\bar{w})$, for every $\bar{w}\in S^+$. By flipping the components of $\bar{z}$ belonging to $\{i\not\in X\mid \bar{z}[i]\neq \bar{y}[i]\}$ to obtain $\bar{y}_X$, the distance $d_H(\bar{z},\bar{y})$ decreases by an amount of $r=|\{i\not\in X\mid \bar{z}[i]\neq \bar{y}[i]\}|$, while the distances $d_H(\bar{z},\bar{w})$ can decrease by at most a quantity of $r$. Hence the strict inequality still holds, in particular, we have that $d_H(\bar{y}_X,\bar{y})<d_H(\bar{y}_X,\bar{w})$, for every $\bar{w}\in S^+$. 

The above discussion implies that in order to check whether $X$ is a sufficient reason when $f^{1}_{S^+,S^-}(\bar x)=1$ (the other case is analogous), it suffices to check that none of the vectors in $\{\bar{y}_X\mid \bar{y}\in S^{-}\}$ satisfies $f^{1}_{S^+,S^-}(\bar{y}_X)=0$. This can be easily checked in polynomial time.
\end{proof}

As a corollary, we obtain the following: 

\begin{corollary}
     \label{coro:discrete-minimal-sr-comp} 
Consider the discrete setting $(\{0,1\},D_H)$. 
There is a polynomial time algorithm that, given sets $S^+,S^- \subseteq \mathbb{R}^n$ and a vector $\bar x \in \mathbb{R}^n$, it computes a minimal sufficient reason $X$ for $\bar x$ with respect to $f^1_{S^+,S^-}$. 
\end{corollary} 

In contrast, {\sc $k$-Check Sufficient Reason}  becomes intractable in this setting when $k \geq 3$. 

\begin{theorem} \label{theo:discrete-check-sr-k3}
The problem {\sc $k$-Check Sufficient Reason}$(\{0,1\},D_H)$ is \conp-complete for every odd integer $k \geq 3$.  
\end{theorem} 

\begin{proof}
The $\conp$ upper bound is direct: given $S^+, S^-$, $\bar{x}\in\{0,1\}^n$ and $X\subseteq \{1,\dots, n\}$, for all $\bar{y}\in\{0,1\}^n$ such that $\bar{y}[i] = \bar{x}[i]$, for every $i\in X$, we can verify in polynomial time whether $f^{k}_{S^+,S^-}(\bar y) = f^{k}_{S^+,S^-}(\bar x)$.

We show that the complement of {\sc $k$-Check Sufficient Reason}$(\{0,1\},D_H)$ is \np-hard for every fixed odd integer $k \geq 3$. We reduce from the Vertex Cover problem: given an undirected graph $G=(V,E)$ with $n$ vertices, and an integer $q\geq 0$, decide whether there is a subset of vertices $C\subseteq V$ with size $|C|\leq q$ such that $C$ is a vertex cover, that is, every edge in $E$ has an endpoint in $C$. We can assume without loss of generality that $q\geq n/2$. Indeed, if $q<n/2$, we consider the instance $(G', q'=n-q)$, where $G'$ results from $G$ by  adding $n-2q$ fresh nodes, each of them connected by an edge to every original node in $G$. Since any subset of vertices of size $n$ or $n-1$ is a vertex cover, we can also assume that $q\leq n-2$.

 Take a graph $G=(V,E)$ with $V=\{1,\dots,n\}$ and $E=\{e_1,\dots,e_m\}$, for $m\geq 1$. The vector dimension is $n+(k+1)/2 + (2q-n)$ and then we write vectors in $\{0,1\}^{n+(k+1)/2+(2q-n)}$ as  concatenations $(\bar{x}, \bar{y}, \bar{z})$ of three vectors $\bar{x}\in \{0,1\}^{n}$, $\bar{y}\in \{0,1\}^{(k+1)/2}$ and $\bar{z}\in \{0,1\}^{2q-n}$.  We denote by $\bar{0}_r$ and $\bar{1}_r$ the vectors $(0,\dots,0)\in \{0,1\}^{r}$ and $(1,\dots,1)\in \{0,1\}^{r}$, respectively. For $h\in\{1,\dots, (k+1)/2\}$ we define the canonical vector $\bar{\alpha}_h\in\{0,1\}^{(k+1)/2}$ such that $\bar{\alpha}_h[h']=1$ for $h'=h$ and $\bar{\alpha}_h[h']=0$, for $h'\neq h$. For each $j\in\{1,\dots,m\}$, we define the vector $\bar{y}_j\in\{0,1\}^{n}$ such that $\bar{y}_j[i]=1$ if the edge $e_j$ is incident to the vertex $i$, and $\bar{y}_j[i]=0$ otherwise. The sets $S^-,S^{+}\subseteq\{0,1\}^{n+(k+1)/2 + (2q-n)}$ are defined as follows:
\begin{align*}
S^{-} & =\{(\bar{y}_j, \bar{\beta}, \bar{1}_{2q-n}) \mid j\in\{1,\dots,m\}, \bar{\beta}\in \{0,1\}^{(k+1)/2}\setminus \{\bar{0}_{(k+1)/2}\}\}.\\
S^{+} & = \{(\bar{0}_n, \bar{\alpha}_1, \bar{1}_{2q-n})\}\cup \{(\bar{1}_n, \bar{\alpha}_h, \bar{0}_{2q-n})\mid h\in\{2,\dots,(k+1)/2\}\}.
\end{align*} 
Note that $|S^+| = (k+1)/2$. Finally, we set $\bar{x}=(\bar{0}_n, \bar{0}_{(k+1)/2}, \bar{0}_{2q-n})$. Note that $f^{k}_{S^+,S^-}(\bar x)=0$ as $d_H(\bar x, (\bar{1}_n, \bar{\alpha}_h, \bar{0}_{2q-n}))=n+1$ for every $h\in \{2,\dots,(k+1)/2\}$, while 
$$d_{H}(\bar x, (\bar{y}_1, \bar{\alpha}_h, \bar{1}_{2q-n}))= 2 + 1 + 2q-n\leq 3 + 2(n-2) - n = n-1$$
for every $h\in\{1,\dots,(k+1)/2\}$.

We claim that there is a vertex cover $C\subseteq\{1,\dots,n\}$ with $|C|\leq q$ if and only if the empty set is not a sufficient reason for $\bar{x}$ w.r.t $f^{k}_{S^+,S^-}$. The latter condition is equivalent to the existence of a vector $\bar{z}\in \{0,1\}^{n+(k+1)/2+(2q-n)}$ with $f^{k}_{S^+,S^-}(\bar z)=1$. It suffices to show the following two properties:

\begin{enumerate}
\item If there is a vertex cover $C\subseteq\{1,\dots,n\}$ of $G$ with $|C|=q$,  then $f^{k}_{S^+,S^-}(\bar z)=1$, for $\bar{z}=(\bar{w}_C,\bar{0}_{(k+1)/2}, \bar{0}_{2q-n})$, where $\bar{w}_C[i]=0$ if $i\in C$, and $\bar{w}_C[i]=1$ if $i\in \{1,\dots,n\}\setminus C$. 
\item Suppose that $f^{k}_{S^+,S^-}(\bar z)=1$ for $\bar{z}=(\bar{w}, \bar{\gamma}, \bar{t})$ and let $C=\{i\in\{1,\dots,n\}\mid \bar{w}[i]=0\}$. Then we have only two possibilities: either $|C|\leq q$ and $C$ is a vertex cover of $G$; or $|C|=q+1$ and every subset $C'\subseteq C$ with $|C'|=q$ is a  vertex cover of $G$.
\end{enumerate}

These two properties imply the correctness of the reduction. If there is a vertex cover of size at most $q$, without loss of generality, we can assume the size is $q$ and then we can apply (1). For the backward direction, if $f^{k}_{S^+,S^-}(\bar z)=1$ for some $\bar{z}$, property (2) implies that, in any of the possible two cases, there is a vertex cover of size at most $q$. 

We start with property (1). Suppose $C$ is a vertex cover with $|C|=q$ and define $\bar{w}_C\in\{0,1\}^n$ such that  $\bar{w}_C[i]=0$ if $i\in C$ and $\bar{w}_C[i]=1$, if $i\in\{1,\dots,n\}\setminus C$. We claim that $f^{k}_{S^+,S^-}(\bar z)=1$ for the vector $\bar{z}=(\bar{w}_C, \bar{0}_{(k+1)/2}, \bar{0}_{2q-n})$. For every $(\bar{y}_j, \bar{\beta}, \bar{1}_{2q-n})\in S^{-}$, we have
$$d_H(\bar{z}, (\bar{y}_j, \bar{\beta}, \bar{1}_{2q-n})) \geq |V\setminus C| + 1 + 2q-n = n-q + 1 + 2q - n = q+1.$$
Here we are using the fact that $C$ covers the edge $e_j$, that is, $|C\cap e_j|\geq 1$, and then:
$$ d_H(\bar{w}_C, \bar{y}_j) = |C\cap e_j| + |V\setminus C| - (2-|C\cap e_j|) \geq 1 + |V\setminus C| - (2-1) = |V\setminus C|.$$
On the other hand, we have: 
$$d_H(\bar{z}, (\bar{0}_n, \bar{\alpha}_1, \bar{1}_{2q-n})) = |V\setminus C| + 1 + 2q-n = q+1$$
and
$$d_H(\bar{z}, (\bar{1}_n, \bar{\alpha}_h, \bar{0}_{2q-n})) = |C| + 1 = q+1$$
for every $h\in\{2,\dots,(k+1)/2\}$.

We now turn to property (2). Suppose that $f^{k}_{S^+,S^-}(\bar z)=1$, for some $\bar{z}= (\bar{w}, \bar{\gamma}, \bar{t})\in \{0,1\}^{n+(k+1)/2 + (2q-n)}$ and define $C = \{i\in \{1,\dots,n\}\mid \bar{w}[i] = 0\}$. We can always pick  vectors $\bar{\beta}_1,\dots, \bar{\beta}_{(k+1)/2}$ from $\{0,1\}^{(k+1)/2}\setminus \{\bar{0}_{(k+1)/2}\}$ such that $d_H(\bar{\gamma}, \bar{\beta}_h)\leq 1$, for every $h\in\{1,\dots, (k+1)/2\}$.

By contradiction, suppose first that $|C|\geq q+2$. Then:
\begin{align*}
d_H((\bar{w}, \bar{\gamma}, \bar{t}), (\bar{y}_1, \bar{\beta}_h, \bar{1}_{2q-n})) & \leq d_H(\bar{w}, \bar{y}_1) + 1 + d_H(\bar{t}, \bar{1}_{2q-n})\\
& = |C\cap e_1| + |V\setminus C| - (2-|C\cap e_1|) + 1 + d_H(\bar{t}, \bar{1}_{2q-n}) \\
& \leq 2+ |V\setminus C| + 1 + d_H(\bar{t}, \bar{1}_{2q-n})\\
& \leq 2 + (n-q-2) + 1 + 2q - n\\
& = q+1
\end{align*}
for every $h\in\{1,\dots, (k+1)/2\}$. On the other hand, we have
$$d_H((\bar{w}, \bar{\gamma}, \bar{t}), (\bar{1}_n, \bar{\alpha}_h, \bar{0}_{2q-n})) \geq d_H(\bar{w}, \bar{1}_n) = |C|\geq q+2$$
for every $h\in\{2,\dots, (k+1)/2\}$. By Proposition~\ref{prop_symm}, choosing $A=\{(\bar{y}_1, \bar{\beta}_h, \bar{1}_{2q-n})\mid h\in \{1,\dots,(k+1)/2\}\}\subseteq S^-$ and $B = \{(\bar{0}_n, \bar{\alpha}_1, \bar{1}_{2q-n})\}\subseteq S^+$,  we have that $f^{k}_{S^+,S^-}(\bar z)=0$; a contradiction. 

Assume that $|C|\leq q$ and, towards a contradiction, suppose that $C$ is not a vertex cover. There must be an edge $e_j$ such that $|C\cap e_j| = 0$.  We have that
\begin{align*}
d_H((\bar{w}, \bar{\gamma}, \bar{t}), (\bar{y}_j, \bar{\beta}_h, \bar{1}_{2q-n})) & \leq d_H(\bar{w}, \bar{y}_j) + 1 + d_H(\bar{t}, \bar{1}_{2q-n})\\
& = |C\cap e_j| + |V\setminus C| - (2-|C\cap e_j|) + 1 + d_H(\bar{t}, \bar{1}_{2q-n}) \\
& = |V\setminus C| - 1 + d_H(\bar{t}, \bar{1}_{2q-n})
\end{align*}
for every $h\in\{1,\dots, (k+1)/2\}$. On the other hand, we have
$$d_H((\bar{w}, \bar{\gamma}, \bar{t}), (\bar{0}_n, \bar{\alpha}_1, \bar{1}_{2q-n})) \geq d_H(\bar{w}, \bar{0}_n) + d_H(\bar{t}, \bar{1}_{2q-n}) = |V\setminus C| + d_H(\bar{t}, \bar{1}_{2q-n}).$$

We conclude that $f^{k}_{S^+,S^-}(\bar z)=0$, using Proposition~\ref{prop_symm} with $A=\{(\bar{y}_j, \bar{\beta}_h, \bar{1}_{2q-n})\mid h\in \{1,\dots,(k+1)/2\}\}\subseteq S^-$ and $B = S^+\setminus \{(\bar{0}_n, \bar{\alpha}_1, \bar{1}_{2q-n})\}$. This is a contradiction. 

Finally, suppose that $|C|=q+1$. We claim that there is no edge with an endpoint in $V\setminus C$. By contradiction, suppose there is such an edge, say $e_j$. In particular, $|C\cap e_j|\leq 1$.  We obtain:
\begin{align*}
d_H((\bar{w}, \bar{\gamma}, \bar{t}), (\bar{y}_j, \bar{\beta}_h, \bar{1}_{2q-n})) & \leq d_H(\bar{w}, \bar{y}_j) + 1 + d_H(\bar{t}, \bar{1}_{2q-n})\\
& = |C\cap e_j| + |V\setminus C| - (2-|C\cap e_j|) + 1 + d_H(\bar{t}, \bar{1}_{2q-n}) \\
& \leq 1 + |V\setminus C| - (2-1) + 1 + d_H(\bar{t}, \bar{1}_{2q-n})\\
& \leq (n-q-1) + 1 + 2q - n\\
& = q
\end{align*}
for every $h\in\{1,\dots, (k+1)/2\}$. On the other hand, we have
$$d_H((\bar{w}, \bar{\gamma}, \bar{t}), (\bar{1}_n, \bar{\alpha}_h, \bar{0}_{2q-n})) \geq d_H(\bar{w}, \bar{1}_n) = |C| = q+1$$
for every $h\in\{2,\dots, (k+1)/2\}$. Again, by Proposition~\ref{prop_symm}, we conclude that $f^{k}_{S^+,S^-}(\bar z)=0$; a contradiction. 
It follows that every edge of $G$ is contained in $C$, and, since $|C|=q+1$, any subset $C'\subseteq C$ with $|C'|=q$ is a vertex cover of $G$.
\end{proof}

The hardness of {\sc $k$-Check Sufficient Reason}$(\{0,1\},D_H)$ holds even when the input subset of components $X\subseteq\{1,\dots,n\}$ is the empty set. As a consequence, {\sc $k$-Minimal Sufficient Reason}$(\{0,1\},D_H)$ is also hard when $k\geq 3$. This also implies hardness for the problem of computing minimal sufficient reasons.

\begin{corollary} \label{coro:discrete-minimal-sr-k3}
The problem {\sc $k$-Minimal Sufficient Reason}$(\{0,1\},D_H)$ is \conp-hard for every odd integer $k \geq 3$. 
\end{corollary}

\section{Further Complexity Results}\label{sec:further}

We start by noting that some of the $\np$-hardness results presented in Section \ref{sec:msr} for minimum sufficient reasons are actually completeness results. Indeed, the results in Sections \ref{sec:l2}, \ref{sec:l1} and \ref{sec:discrete} show that checking sufficient reasons can be done in polynomial time in several settings. This directly gives us $\np$ algorithms for these cases.

\begin{corollary}\label{coro:np-c-msr}
The following problems are $\np$-complete: {\sc $1$-Minimum Sufficient Reason}$(\{0,1\},D_H)$, {\sc $1$-Minimum Sufficient Reason}$(\mathbb{R},D_1)$, and {\sc $k$-Minimum Sufficient Reason}$(\mathbb{R},D_2)$, for  every fixed odd integer $k \geq 1$. 
\end{corollary}

In the remaining of the section we establish two additional complexity results. We first pinpoint the complexity of {\sc $k$-Minimum Sufficient Reason}$(\{0,1\},D_H)$, when $k\geq 3$, and show that this problem is complete for $\Sigma_2^p$, the second level of the polynomial hierarchy. Then we study the parameterized complexity of {\sc $k$-Counterfactual Explanation} in the setting based on $\ell_2$-distance.

\paragraph{{\bf The complexity of Minimum Sufficient Reason in the discrete setting, for $k\geq 3$}} We assume familiarity with the complexity class $\Sigma_2^p$, which corresponds to $\np^{\np}$, that is, $\np$ with an oracle in $\np$. We show the following:

\begin{theorem} \label{theo:sigma2p-msr-discrete}
{\sc $k$-Minimum Sufficient Reason}$(\{0,1\},D_H)$ is $\Sigma_2^p$-complete, for every fixed odd integer $k\geq 3$.
\end{theorem}

We start by establishing the $\Sigma_2^p$-completeness of the following auxiliary problem:

\newcommand{\nforall}{%
\ooalign{$\forall$\cr\hidewidth$\!/$\hidewidth\cr}%
    }

\medskip

\begin{center}
\fbox{\begin{tabular}{lp{8cm}}
{\small PROBLEM} : & {$\exists\nforall$-\sc Vertex-Cover} \\
{\small INPUT} : & 
Graph $G = (V, E)$, and two positive integers $p < q$.

%of dimension $d$,ss
%\\ & $x \in \{0,1\}^d$ an instance, and $0 \leq \delta \leq 1$
\\ 
{\small OUTPUT} : & {\sc Yes}, if there is $S \subseteq V$, with $|S| \leq p$, such that no superset $S'$ ($S \subseteq S' \subseteq V$) with $|S'| \leq q$ is a vertex cover of $G$. {\sc No} otherwise.
\end{tabular}}
\end{center}

\medskip

%We will prove that $\exists\nforall$-{\sc Vertex-Cover} is $\Sigma_2^p$-complete.
This problem is tightly related to \emph{``interdiction''} optimization problems over graphs, in which one has some base problem (i.e., min vertex-cover, max clique, min dominating set, etc.) and the task is to find a small subset of the vertices that intersects every optimal solution. A nice general framework for proving $\Sigma_2^p$-completeness of interdiction problems has been recently established by Gr\"une and Wulf~\cite{grüne2025complexityblockingsolutions}. For our purposes, nonetheless, it will be enough to consider a simpler and older result from \cite{rutenburgPropositionalTruthMaintenance1994}: the independent set interdiction problem, defined next, is $\Sigma_2^p$-complete.

\medskip

\begin{center}
\fbox{\begin{tabular}{lp{8cm}}
{\small PROBLEM} : & {\sc Independent Set Interdiction} \\
{\small INPUT} : & 
Graph $G = (V, E)$, and two positive integers $p, q$.

%of dimension $d$,ss
%\\ & $x \in \{0,1\}^d$ an instance, and $0 \leq \delta \leq 1$
\\ 
{\small OUTPUT} : & {\sc Yes}, if there is $S \subseteq V$, with $|S| \leq p$, such that every independent set $S' \subseteq V$ with $|S'| \geq q$ holds $S \cap S' \neq \varnothing$. {\sc No} otherwise.
\end{tabular}}
\end{center}

\medskip

% The $\Sigma_2^p$-completeness of the {\sc Independent Set Interdiction} problem is 

%\begin{theorem}[\cite{rutenburgPropositionalTruthMaintenance1994,grüne2025complexityblockingsolutions}]
   % The {\sc Independent Set Interdiction} problem is $\Sigma_2^p$-complete.
%\end{theorem}

%We now establish the complexity of $\exists\nforall$-{\sc Vertex-Cover} by reducing from the preceding problem. We remark that the vertex-cover interdiction problem is only coNP-complete~\cite{grüne2025complexityblockingsolutions} since if $p \geq 2$ it suffices to take any $S$ that is not independent to force that every vertex cover $S'$ intersects $S$. This explains why we reduce from {\sc Independent Set Interdiction} instead.

\begin{theorem}
    The $\exists\nforall$-{\sc Vertex-Cover} problem is $\Sigma_2^p$-complete. Hardness holds even when $n/2\leq q \leq n-2$, where $n$ is the number of vertices of the input graph. 
\end{theorem}

\begin{proof}
    Membership is direct: the NP-part of the computation guesses $S$, checks its size, and calls a coNP-oracle to check that every superset $S' \supseteq S$ of size at most $q$ fails to cover $G$. 

    For hardness, we reduce from {\sc Independent Set Interdiction}. Let $(G, p, q)$ be an input instance, where $G=(V,E)$. We then simply output $(G, p, |V|-q)$, which obviously can be done in polynomial time. 

    To prove correctness we will need some notation. Given a graph $G$, we denote by $\alpha(G)$ the maximum size of an independent set of $G$, and by $\tau(G)$ the minimum size of a vertex cover of $G$. Moreover, given a graph $G$ and a subset $S \subseteq V$, we denote by $\tau(G, S)$ the minimum size of a vertex cover that contains $S$. Our reduction relies on two simple observations:
    \begin{enumerate}
        \item $\alpha(G) + \tau(G) = |V|$. \label{obs:1}
        \item $\tau(G, S) = |S| + \tau(G[V \setminus S])$. \label{obs:2}
    \end{enumerate}
While observation (\ref{obs:1}) corresponds to the well-known fact that the complement of a vertex cover is an independent set (and vice versa), observation (\ref{obs:2}) requires a brief justification. Indeed, note that if $S' \supseteq S$ is any vertex cover of size $\tau(G, S)$, then $S' \setminus S$ must be a vertex cover of $G[V \setminus S]$,  and thus $\tau(G, S) = |S'| = |S| + |S' \setminus S| \geq |S| + \tau(G[V \setminus S])$. Conversely, let $S''$ be a vertex cover of minimum size for $G[V \setminus S]$. Then, $S'' \cup S$ is a vertex cover for $G$ which contains $S$, and thus $|S'' \cup S| = \tau(G[V \setminus S]) + |S| \geq \tau(G, S)$.

Now, observe that $\exists\nforall$-{\sc Vertex-Cover} can be rephrased as whether there exists $S \subseteq V$ with $|S| \leq p$ such that $\tau(G, S) > q$. Similarly, the {\sc Independent Set Interdiction} problem can be rephrased as whether there exists $S \subseteq V$ with $|S| \leq p$ such that $\alpha(G[V \setminus S]) < q$, since the independent sets of $G$ that do not intersect $S$ are exactly the independent sets of $G[V \setminus S]$. Using these rephrasings to prove the correctness of our reduction, it suffices to show that for any $S \subseteq V$ with $|S| \leq p$, we have
\[
\tau(G, S) > q \iff \alpha(G[V\setminus S]) < |V| - q.
\]
This now follows easily from our observations:
\begin{align*}
    \tau(G, S) > q & \iff |S| + \tau(G[V \setminus S]) > q \tag{Observation (\ref{obs:2})}\\
    & \iff  |S| + (|V\setminus S| - \alpha(G[V \setminus S])) > q \tag{Observation (\ref{obs:1})}\\
    & \iff |S| + (|V| - |S| - \alpha(G[V \setminus S]) > q\\
    & \iff \alpha(G[V \setminus S]) < |V| - q.
\end{align*}

It remains to show that hardness holds when $n/2\leq q\leq n-2$. Let $(G,p,q)$ be an instance of $\exists\nforall$-{\sc Vertex-Cover}, where $G=(V,E)$ and $n:=|V|$. If $q<n/2$, we transform the instance to $(G',p,q')$, where $q'=n-q$ and $G'=(V',E')$ results from $G$ by adding $n-2q$ fresh nodes, each of them connected by an edge to every node in $G$. Note that $q'=n'/2$ where $n'=|V'|$. We claim that $(G,p,q)\in \exists\nforall\text{-{\sc Vertex-Cover}}$  if and only if $(G',p,q') \in \exists\nforall\text{-{\sc Vertex-Cover}}$. Suppose first that there is $S\subseteq V$ with $|S|\leq p$ such that every superset $S'$ of $S$ with $|S'|\leq q$ is not a vertex cover of $G$. It holds that there is no superset $S''$ of $S$ with $|S''|\leq q'$ that is a vertex cover of $G'$. By contradiction, if there is such a $S''$, as $q'< n$ and there is an edge from every fresh node to every node in $V$, then $S''$ must contain all the $n-2q$ fresh nodes. It follows that $|S''\cap V|\leq q$. This is a contradiction as $S''\cap V$ is a vertex cover of $G$ and a superset of $S$. On the other hand, assume there is $S\subseteq V'$ with $|S|\leq p$ such that every superset $S'$ of $S$ with $|S'|\leq q'$ is not a vertex cover of $G'$. We claim that there is no superset $S''\subseteq V$ of $S\cap V$ with $|S''|\leq q$ that is a vertex cover of $G$. If there is such a $S''$ then by adding the fresh nodes, we obtain a superset of $S$ that is a vertex cover of $G'$ of size at most $q+n-2q=q'$; a contradiction. Finally, note that any instance $(G,p,q)$ with $q\geq n-1$ is a negative instance of $\exists\nforall\text{-{\sc Vertex-Cover}}$, as any set of at least $n-1$  nodes is a vertex cover of $G$. 
\end{proof}

\begin{proof}[Proof of Theorem \ref{theo:sigma2p-msr-discrete}]
The upper bound is direct: we guess a subset $X$ of components with $|X|\leq \ell$ and then verify with an $\conp$-oracle that is a sufficient reason (recall Theorem \ref{theo:discrete-check-sr-k3}). 

We reduce from $\exists\nforall\text{-{\sc Vertex-Cover}}$ under the restriction  $n/2\leq q\leq n-2$. Given an instance $(G,p,q)$ of $\exists\nforall\text{-{\sc Vertex-Cover}}$, where $G=(V,E)$ and $V=\{1,\dots, n\}$, we define the instance $(S^+, S^-, \bar{x}, p)$ of {\sc $k$-Minimum Sufficient Reason}$(\{0,1\},D_H)$, where $(S^+,S^-,\bar{x})$ is define from $(G,q)$ \emph{exactly} as in the proof of Theorem \ref{theo:discrete-check-sr-k3}. Recall that the dimension of vectors in $(S^+,S^-,\bar{x})$ is $n+(k+1)/2 + (2q-n)$, and hence we write vectors as the concatenation of three vectors of dimensions $n$, $(k+1)/2$ and $(2q-n)$, respectively. The vector $\bar{x}$ is the all zero vector $\bar{x}=(\bar{0}_n,\bar{0}_{(k+1)/2}, \bar{0}_{2q-n})$ (here $\bar{0}_r$ is the all zero vector of dimension $r$). It holds that $f^{k}_{S^+,S^-}(\bar x)=0$. The condition $n/2\leq q\leq n-2$ is needed to apply this reduction. The following properties of $(S^+,S^-,\bar{x})$ are shown in the proof of Theorem \ref{theo:discrete-check-sr-k3}:
\begin{enumerate}
\item If there is a vertex cover $C\subseteq\{1,\dots,n\}$ of $G$ with $|C|=q$,  then $f^{k}_{S^+,S^-}(\bar z)=1$, for $\bar{z}=(\bar{w}_C,\bar{0}_{(k+1)/2}, \bar{0}_{2q-n})$, where $\bar{w}_C[i]=0$ if $i\in C$, and $\bar{w}_C[i]=1$ if $i\in \{1,\dots,n\}\setminus C$. 
\item Suppose that $f^{k}_{S^+,S^-}(\bar z)=1$ for $\bar{z}=(\bar{w}, \bar{\gamma}, \bar{t})$ and let $C=\{i\in\{1,\dots,n\}\mid \bar{w}[i]=0\}$. Then, either $|C|\leq q$ and $C$ is a vertex cover of $G$; or $|C|=q+1$ and every subset $C'\subseteq C$ with $|C'|=q$ is a  vertex cover of $G$.
\end{enumerate}
To see that the reduction is correct, assume first that there exists $S\subseteq\{1,\dots,n\}$ with $|S|\leq p$ such that no superset $S'$ of $S$ with $|S'|\leq q$ is a vertex cover of $G$. We claim that $S$ is a sufficient reason for $\bar x$ w.r.t. $f^k_{S^+,S^-}$. By contradiction, suppose there is a vector $\bar{z}=(\bar{w}, \bar{\gamma}, \bar{t})$ with $\bar{z}[i]=\bar{x}[i]=0$, for every $i\in S$, such that $f^k_{S^+,S^-}(\bar{z})=1$. Define $C=\{i\in\{1,\dots,n\}\mid \bar{w}[i]=0\}$. It holds that $S\subseteq C$. By property (2), we have two possible cases. If $|C|\leq q$, then $C$ is a vertex cover, and we obtain a contradiction. If $|C|= q+1$, we pick any $u\in C\setminus S$ (recall $p < q$), and then $C\setminus \{u\}$ is a superset of $S$ with $|C\setminus \{u\}|=q$ that is a vertex cover; a contradiction. 

For the other direction, suppose that $X\subseteq \{1,\dots, n+(k+1)/2+(2q-n)\}$ is a sufficient reason for $\bar{x}$ with $|X|\leq p$. Take $S=X\cap \{1,\dots,n\}$. We claim that no superset $C$ of $S$ with $|C|\leq q$ is a vertex cover of $G$. By contradiction, suppose that such a $C$ exists. We can assume that $|C|=q$. Take the vector $\bar{z}=(\bar{w}_C, \bar{0}_{(k+1)/2}, \bar{0}_{2q-n})$, where $\bar{w}_C[i]=0$, if $i\in C$, and $\bar{w}_C[i]=1$, otherwise. It holds that $\bar{z}[i]=\bar{x}[i]=0$, for every $i\in X$. Moreover, by property (1), we obtain that $f^k_{S^+,S^-}(\bar{z})=1$. This is a contradiction. 
\end{proof}
% \paragraph{{\bf The parameterized complexity of Counterfactual Explanation for the $\ell_2$-distance.}}

\section{Implementation and Experiments}\label{sec:implementation}
The results presented in the previous sections offer a comprehensive understanding of the requirements for computing explanations for $k$-NN classifiers in practice. In this section, we explore this topic further through a preliminary practical analysis, focusing on the widely used case of $k=1$. This case is particularly appealing as it reduces implementation complexity even for problems that remain tractable for larger values of $k$. We examine the problems of {\sc $1$-Counterfactual Explanation} and {\sc $1$-Minimal Sufficient Reason}. 
%Guided by our theoretical findings and proof techniques, we leverage convex optimization tools for the continuous setting with the $\ell_2$-distance, and employ Integer Quadratic Programming (IQP) and SAT solvers for the $\ell_1$-distance and the discrete setting.

\subsection{Experimental setup and datasets}  
Our implementation is written in Python (3.10), but naturally calls external solvers. Namely, we use the standard \texttt{cvxpy} library for convex programming, the \texttt{Gurobi}~\cite{gurobi} solver for IQP, and the recent~\texttt{cardinality-cadical} solver~\cite{reeves2024clauses} for SAT-solving. All experiments were run on a Macbook Pro M1 2020 with 16GB of RAM.
%\paragraph{\bf{Datasets}} 
We experiment both on the MNIST dataset of handwritten digit recognition~\cite{deng2012mnist} and on synthetic random data. For the MNIST dataset, we consider both the original grayscale $28 \times 28$ images, as well as a binarized version to represent the discrete setting, and different rescalings of the images to experiment with a different number of dimensions. Similarly, we consider subsets of the training data of different sizes (MNIST was originally split into 60\,000 training examples and 10\,000 test images). When computing an explanation for an image classified as digit $d \in \{0, 1, \ldots, 10\}$, we consider all images of digit $d$ as ``positive'' examples, and images of digits $d' \neq d$ as ``negative''.
For the synthetic random data, we consider uniformly random vectors in $\{0, 1\}^n$, labeled according to independent Bernoulli variables of parameter $p = \frac{1}{2}$, since additional experiments with other values of $p$ displayed a similar behavior.

\subsection{Implementation} For computing minimal sufficient reasons over $(\mathbb{R},D_1)$, we directly implement the simple algorithm from~\Cref{prop_l1_1check}, using the efficient \texttt{FAISS} library for fast $1$-NN search~\cite{douze2024faiss}. For computing counterfactual explanations in $(\mathbb{R}, D_2)$, we implement the convex program from~\Cref{theo:l2-cont-counterfactual}, ignoring tie-breaking concerns for simplicity. Computing counterfactual explanations in $(\mathbb{R}, D_1)$, we defer to the optimized implementation of a mixed integer program by Contardo et al.~\cite{DBLP:conf/iscopt/ContardoFRV24}. For counterfactual explanations in the discrete setting, we consider first the following IQP formulation for a vector $\bar{x}$ classified positively:
\begin{align*}
    \text{minimize } & \; \sum_{i=1}^n (\bar{x}[i] - \bar{y}[i])^2 \\
    \text{subject to } 
    & d^+ = \min_{\bar{z} \in S^+} \sum_{i=1}^n (\bar{z}[i] - \bar{y}[i])^2,\\
     & d^- = \min_{\bar{z} \in S^-} \sum_{i=1}^n (\bar{z}[i] - \bar{y}[i])^2,\\
     & d^- < d^+,\\
     & \bar{y} \in \{0, 1\}^d, \; d^+ \in \mathbb{R}, \; d^- \in \mathbb{R},
\end{align*}
where constraints 
%of the form 
\(
m = \min(r_1, \ldots, r_t)
\)
can be expressed by introducing indicator variables $v_1, \ldots, v_t \in \{0, 1\}$ and adding constraints
$m \leq r_i$ and $v_i \cdot r_i \leq m$, for every $i \in \{1,\dots,t\}$, and $\sum_{i = 1}^t v_i = 1$. 

%\begin{align*}
 %    m & \leq r_i, \quad \forall i \in \{1, \ldots, t\} \\
  %  v_i \cdot r_i &\leq m, \quad \forall i \in \{1, \ldots, t\} \\
   % \sum_{i = 1}^t v_i &= 1.
%\end{align*}

\paragraph{\bf SAT encoding}
We also propose a CNF encoding to find a closest counterfactual $\bar{y}$, leveraging the native support for (guarded) cardinality constraints of the recent \texttt{cardinality-cadical} solver~\cite{reeves2024clauses}. Given boolean literals $\ell_1, \ldots, \ell_n$, cardinality constraints are of the form
\(
\sum_{i=1}^n \ell_i \geq b,
\)
for some constant integer $b$ (the ``bound''). On the other hand, ``guarded'' cardinality constraints, provided a different boolean literal $g$, are of the form
\(
    g \implies \left(\sum_{i=1}^n \ell_i \geq b\right).
\)
Our encoding uses boolean variables $y_1, \ldots, y_n$ where $y_i$ corresponds to whether $\bar{y}[i] = 1$, and variables $c_1, \ldots, c_{|S^-|}$, where  $c_i$ intuitively represents that the $i$-th point in $S^-$ (under some fixed ordering) will be the closest point to $\bar{y}$ among $S^+ \cup S^-$. 
We thus add first a clause of the form
\(
\bigvee_{i=1}^{|S^-|} c_i.
\)
Then, if we call $\bar{o}$ to the $i$-th point in $S^-$, we need to enforce that
\begin{equation}\label{eq:boolean-imp}
c_i \implies d_H(\bar{y}, \bar{o}) < d_H(\bar{y}, \bar{s}), \quad \forall \bar{s} \in S^+,
\end{equation}
which we show next how to encode as a guarded cardinality constraint.
Let us focus on a fixed pair $\bar{o}, \bar{s}$, and define the sets
\[
\Delta_0 := \{ i \mid \bar{o}[i] = 0, \bar{s}[i] =  1\} \quad ; \quad \Delta_1 := \{ i \mid \bar{o}[i] = 1, \bar{s}[i] =  0\}.
\]
We then have the following equivalence:
\begin{align*}
d_H(\bar{y}, \bar{o}) < d_H(\bar{y}, \bar{s}) &\iff \sum_{i \in \Delta_0} y_i + \sum_{i \in \Delta_1} \neg y_i < \sum_{i \in \Delta_0} \neg y_i + \sum_{i \in \Delta_1}  y_i\\
&\iff \sum_{i \in \Delta_0} y_i + \sum_{i \in \Delta_1} (1 - y_i) < \sum_{i \in \Delta_0} (1 - y_i) + \sum_{i \in \Delta_1}  y_i\\
& \iff |\Delta_1| - |\Delta_0| < 2\sum_{i \in \Delta_1} y_i -  2\sum_{i \in \Delta_0} y_i \\
& \iff \frac{|\Delta_0 + \Delta_1|}{2} < \sum_{i \in \Delta_0} \neg y_i + \sum_{i \in \Delta_1} y_i,
\end{align*}
which implies that we can encode~\Cref{eq:boolean-imp} as the following guarded cardinality constraint:
\[
c_i \implies \left(\sum_{i \in \Delta_0} \neg y_i + \sum_{i \in \Delta_1} y_i\right) \geq \left\lfloor \frac{|\Delta_0 + \Delta_1|}{2} \right\rfloor + 1.
\]
The encoding therefore has $|S^+| \cdot |S^-|$ guarded cardinality constraints.
Finally, in order to minimize the distance $d_H(\bar{x}, \bar{y})$, note that a cardinality constraint
\[
    \sum_{i \text{ s.t. } \bar{x}[i] = 0} y_i +  \sum_{i \text{ s.t. } \bar{x}[i] = 1} \neg y_i \geq n - k
\]
is equivalent to $d_H(\bar{x}, \bar{y}) \leq k$. Therefore, by doing a binary search over the parameter $k$ (or a linear search if the answer is expected to be small) we obtain a closest counterfactual explanation.

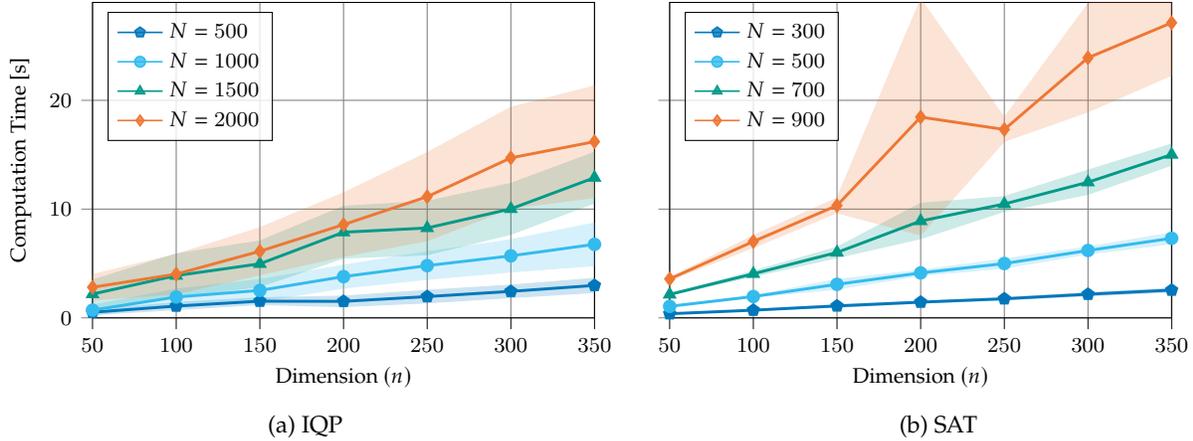
\begin{figure}
    \centering
% Figure size definitions
\newlength\figurewidth
\newlength\figureheight
\setlength\figurewidth{0.5\columnwidth}
\setlength\figureheight{0.7\figurewidth}

\definecolor{darkcyan0119187}{RGB}{0,119,187}
\definecolor{darkcyan0153136}{RGB}{0,153,136}
\definecolor{gray}{RGB}{128,128,128}
\definecolor{mediumturquoise51187238}{RGB}{51,187,238}
\definecolor{tomato23811951}{RGB}{238,119,51}

% This file was created with tikzplotlib v0.10.1.
\begin{subfigure}{0.49\textwidth}
\begin{tikzpicture}[font=\footnotesize, every plot/.style={line width=1.1pt}]

\begin{axis}[
height=\figureheight,
legend cell align={left},
legend style={fill opacity=0.8, draw opacity=1, text opacity=1, at={(0.03,0.97)}, anchor=north west},
minor xtick={},
minor ytick={},
tick align=outside,
tick pos=left,
width=\figurewidth,
x grid style={gray},
xlabel={Dimension $(n)$},
xmajorgrids,
xmin=50, xmax=350,
ylabel near ticks, 
xtick style={color=black},
xtick={50, 100, 150, 200, 250, 300, 350},
y grid style={gray},
ylabel={Computation Time [s]},
ymajorgrids,
ymin=0, ymax=29,
% ytick style={color=black},
% ytick={-5,0,5,10,15,20,25}
]
\path [draw=darkcyan0119187, fill=darkcyan0119187, opacity=0.2]
(axis cs:50,0.7437)
--(axis cs:50,0.2898)
--(axis cs:100,0.756)
--(axis cs:150,1.224)
--(axis cs:200,1.037)
--(axis cs:250,1.38)
--(axis cs:300,1.861)
--(axis cs:350,2.331)
--(axis cs:350,3.634)
--(axis cs:350,3.634)
--(axis cs:300,3.01)
--(axis cs:250,2.532)
--(axis cs:200,1.999)
--(axis cs:150,1.842)
--(axis cs:100,1.407)
--(axis cs:50,0.7437)
--cycle;

\path [draw=mediumturquoise51187238, fill=mediumturquoise51187238, opacity=0.2]
(axis cs:50,1.164)
--(axis cs:50,0.2663)
--(axis cs:100,1.197)
--(axis cs:150,1.531)
--(axis cs:200,2.74)
--(axis cs:250,3.54)
--(axis cs:300,4.221)
--(axis cs:350,4.797)
--(axis cs:350,8.722)
--(axis cs:350,8.722)
--(axis cs:300,7.18)
--(axis cs:250,6.066)
--(axis cs:200,4.846)
--(axis cs:150,3.52)
--(axis cs:100,2.636)
--(axis cs:50,1.164)
--cycle;

\path [draw=darkcyan0153136, fill=darkcyan0153136, opacity=0.2]
(axis cs:50,3.477)
--(axis cs:50,0.8982)
--(axis cs:100,1.886)
--(axis cs:150,2.876)
--(axis cs:200,5.495)
--(axis cs:250,5.803)
--(axis cs:300,7.683)
--(axis cs:350,10.56)
--(axis cs:350,15.22)
--(axis cs:350,15.22)
--(axis cs:300,12.36)
--(axis cs:250,10.72)
--(axis cs:200,10.25)
--(axis cs:150,7.065)
--(axis cs:100,5.883)
--(axis cs:50,3.477)
--cycle;

\path [draw=tomato23811951, fill=tomato23811951, opacity=0.2]
(axis cs:50,3.976)
--(axis cs:50,1.65)
--(axis cs:100,2.207)
--(axis cs:150,3.96)
--(axis cs:200,5.654)
--(axis cs:250,7.101)
--(axis cs:300,10.08)
--(axis cs:350,11.07)
--(axis cs:350,21.32)
--(axis cs:350,21.32)
--(axis cs:300,19.35)
--(axis cs:250,15.19)
--(axis cs:200,11.49)
--(axis cs:150,8.287)
--(axis cs:100,5.839)
--(axis cs:50,3.976)
--cycle;

\addplot [semithick, darkcyan0119187, mark=pentagon*, mark size=2, mark options={solid}]
table {%
50 0.5167
100 1.082
150 1.533
200 1.518
250 1.956
300 2.435
350 2.982
};
\addlegendentry{$N=500$}
\addplot [semithick, mediumturquoise51187238, mark=*, mark size=2, mark options={solid}]
table {%
50 0.715
100 1.917
150 2.526
200 3.793
250 4.803
300 5.7
350 6.76
};
\addlegendentry{$N=1000$}
\addplot [semithick, darkcyan0153136, mark=triangle*, mark size=2, mark options={solid}]
table {%
50 2.188
100 3.884
150 4.97
200 7.873
250 8.264
300 10.02
350 12.89
};
\addlegendentry{$N=1500$}
\addplot [semithick, tomato23811951, mark=diamond*, mark size=2, mark options={solid}]
table {%
50 2.813
100 4.023
150 6.124
200 8.573
250 11.14
300 14.71
350 16.2
};
\addlegendentry{$N=2000$}
\end{axis}
\end{tikzpicture}
\caption{IQP}
\end{subfigure}
\hfill
\begin{subfigure}{0.48\textwidth}
\begin{tikzpicture}[font=\footnotesize, every plot/.style={line width=1.1pt}]
\begin{axis}[
height=\figureheight,
legend cell align={left},
legend style={fill opacity=0.8, draw opacity=1, text opacity=1, at={(0.03,0.97)}, anchor=north west},
minor xtick={},
minor ytick={},
tick align=outside,
tick pos=left,
width=\figurewidth,
x grid style={gray},
xlabel={Dimension $(n)$},
xmajorgrids,
xmin=50, xmax=350,
xtick style={color=black},
xtick={50, 100, 150, 200, 250, 300, 350},
y grid style={gray},
ylabel={},
ymajorgrids,
ymin=0, ymax=29,
yticklabels=\empty
% ytick=\empty
% ytick style={color=black},
% ytick={-5,0,5,10,15,20,25,30,35}
]
\path [draw=darkcyan0119187, fill=darkcyan0119187, opacity=0.2]
(axis cs:50,0.3921)
--(axis cs:50,0.3632)
--(axis cs:100,0.6624)
--(axis cs:150,0.9815)
--(axis cs:200,1.359)
--(axis cs:250,1.658)
--(axis cs:300,2.021)
--(axis cs:350,2.373)
--(axis cs:350,2.728)
--(axis cs:350,2.728)
--(axis cs:300,2.323)
--(axis cs:250,1.852)
--(axis cs:200,1.532)
--(axis cs:150,1.21)
--(axis cs:100,0.7502)
--(axis cs:50,0.3921)
--cycle;

\path [draw=mediumturquoise51187238, fill=mediumturquoise51187238, opacity=0.2]
(axis cs:50,1.095)
--(axis cs:50,1.024)
--(axis cs:100,1.891)
--(axis cs:150,2.639)
--(axis cs:200,3.893)
--(axis cs:250,4.572)
--(axis cs:300,5.875)
--(axis cs:350,6.809)
--(axis cs:350,7.822)
--(axis cs:350,7.822)
--(axis cs:300,6.514)
--(axis cs:250,5.415)
--(axis cs:200,4.371)
--(axis cs:150,3.529)
--(axis cs:100,2.03)
--(axis cs:50,1.095)
--cycle;

\path [draw=darkcyan0153136, fill=darkcyan0153136, opacity=0.2]
(axis cs:50,2.211)
--(axis cs:50,2.097)
--(axis cs:100,3.808)
--(axis cs:150,5.53)
--(axis cs:200,7.279)
--(axis cs:250,9.795)
--(axis cs:300,11.36)
--(axis cs:350,14.04)
--(axis cs:350,15.98)
--(axis cs:350,15.98)
--(axis cs:300,13.58)
--(axis cs:250,11.15)
--(axis cs:200,10.53)
--(axis cs:150,6.494)
--(axis cs:100,4.293)
--(axis cs:50,2.211)
--cycle;

\path [draw=tomato23811951, fill=tomato23811951, opacity=0.2]
(axis cs:50,3.705)
--(axis cs:50,3.45)
--(axis cs:100,6.467)
--(axis cs:150,9.661)
--(axis cs:200,7.624)
--(axis cs:250,16.23)
--(axis cs:300,18.94)
--(axis cs:350,22.31)
--(axis cs:350,31.94)
--(axis cs:350,31.94)
--(axis cs:300,28.87)
--(axis cs:250,18.42)
--(axis cs:200,29.3)
--(axis cs:150,11.01)
--(axis cs:100,7.572)
--(axis cs:50,3.705)
--cycle;

\addplot [semithick, darkcyan0119187, mark=pentagon*, mark size=2, mark options={solid}]
table {%
50 0.3776
100 0.7063
150 1.096
200 1.445
250 1.755
300 2.172
350 2.55
};
\addlegendentry{$N=300$}
\addplot [semithick, mediumturquoise51187238, mark=*, mark size=2, mark options={solid}]
table {%
50 1.06
100 1.96
150 3.084
200 4.132
250 4.994
300 6.194
350 7.316
};
\addlegendentry{$N=500$}
\addplot [semithick, darkcyan0153136, mark=triangle*, mark size=2, mark options={solid}]
table {%
50 2.154
100 4.05
150 6.012
200 8.905
250 10.47
300 12.47
350 15.01
};
\addlegendentry{$N=700$}
\addplot [semithick, tomato23811951, mark=diamond*, mark size=2, mark options={solid}]
table {%
50 3.578
100 7.02
150 10.33
200 18.46
250 17.32
300 23.91
350 27.13
};
\addlegendentry{$N=900$}
\end{axis}
\end{tikzpicture}
\caption{SAT}
\end{subfigure}
    \caption{Runtimes for counterfactual explanations over $\{0, 1\}^n$. The total training set has size $N := |S^+| + |S^-|$, consisting of independent uniformly random samples from $\{0, 1\}^n$. Confidence intervals of $95\%$ over $30$ independent runs are displayed. }
    \label{fig:disc-count-gurobi}
\end{figure}

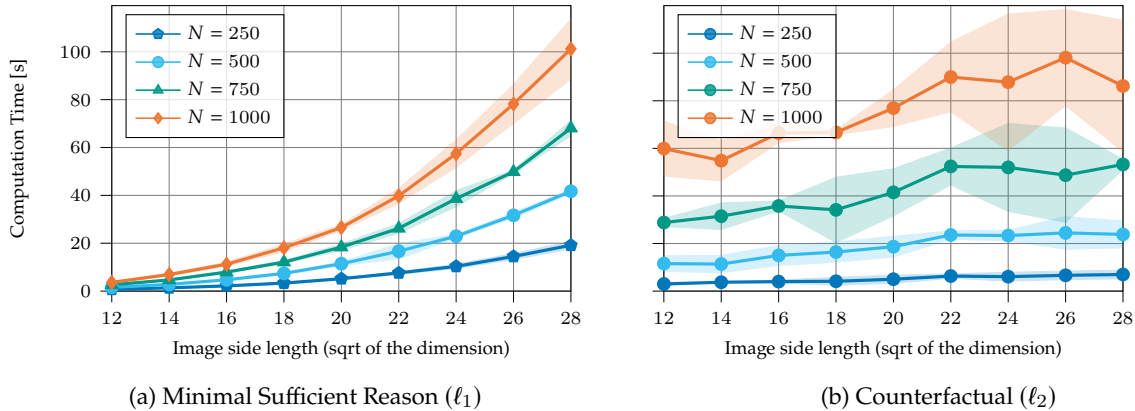
\begin{figure}
    \centering
    \begin{subfigure}{0.49\textwidth}
% Figure size definitions
\setlength\figurewidth{0.95\columnwidth}
\setlength\figureheight{0.7\figurewidth}

% This file was created with tikzplotlib v0.10.1.
\begin{tikzpicture}[font=\scriptsize, every plot/.style={line width=1.2pt}]

\definecolor{darkcyan0119187}{RGB}{0,119,187}
\definecolor{darkcyan0153136}{RGB}{0,153,136}
\definecolor{gray}{RGB}{128,128,128}
\definecolor{mediumturquoise51187238}{RGB}{51,187,238}
\definecolor{tomato23811951}{RGB}{238,119,51}

\begin{axis}[
height=\figureheight,
legend cell align={left},
legend style={fill opacity=0.8, draw opacity=1, text opacity=1, at={(0.03,0.97)}, anchor=north west},
minor xtick={},
minor ytick={},
tick align=outside,
tick pos=left,
width=\figurewidth,
x grid style={gray},
xlabel={Image side length (sqrt of the dimension)},
xmajorgrids,
xmin=12, xmax=28,
xtick style={color=black},
xtick={10,12,14,16,18,20,22,24,26,28,30},
y grid style={gray},
ylabel={Computation Time [s]},
ymajorgrids,
ymin=0, ymax=119.3,
ytick style={color=black},
ytick={-20,0,20,40,60,80,100}
]
\path [draw=darkcyan0119187, fill=darkcyan0119187, opacity=0.2]
(axis cs:12,0.7887)
--(axis cs:12,0.677)
--(axis cs:14,1.217)
--(axis cs:16,1.948)
--(axis cs:18,3.019)
--(axis cs:20,4.898)
--(axis cs:22,7.044)
--(axis cs:24,9.409)
--(axis cs:26,13.22)
--(axis cs:28,17.6)
--(axis cs:28,20.72)
--(axis cs:28,20.72)
--(axis cs:26,15.74)
--(axis cs:24,11.23)
--(axis cs:22,8.123)
--(axis cs:20,5.45)
--(axis cs:18,3.659)
--(axis cs:16,2.414)
--(axis cs:14,1.481)
--(axis cs:12,0.7887)
--cycle;

\path [draw=mediumturquoise51187238, fill=mediumturquoise51187238, opacity=0.2]
(axis cs:12,1.704)
--(axis cs:12,1.449)
--(axis cs:14,2.644)
--(axis cs:16,4.561)
--(axis cs:18,7.161)
--(axis cs:20,10.63)
--(axis cs:22,13.48)
--(axis cs:24,22.12)
--(axis cs:26,30.5)
--(axis cs:28,40.73)
--(axis cs:28,42.68)
--(axis cs:28,42.68)
--(axis cs:26,32.84)
--(axis cs:24,23.73)
--(axis cs:22,19.83)
--(axis cs:20,12.23)
--(axis cs:18,7.631)
--(axis cs:16,4.89)
--(axis cs:14,2.953)
--(axis cs:12,1.704)
--cycle;

\path [draw=darkcyan0153136, fill=darkcyan0153136, opacity=0.2]
(axis cs:12,3.05)
--(axis cs:12,2.214)
--(axis cs:14,4.529)
--(axis cs:16,7.471)
--(axis cs:18,11.44)
--(axis cs:20,17.03)
--(axis cs:22,24.13)
--(axis cs:24,35.32)
--(axis cs:26,49.04)
--(axis cs:28,64.89)
--(axis cs:28,71.29)
--(axis cs:28,71.29)
--(axis cs:26,50.68)
--(axis cs:24,41.94)
--(axis cs:22,28.28)
--(axis cs:20,19.72)
--(axis cs:18,12.78)
--(axis cs:16,8.395)
--(axis cs:14,4.794)
--(axis cs:12,3.05)
--cycle;

\path [draw=tomato23811951, fill=tomato23811951, opacity=0.2]
(axis cs:12,3.957)
--(axis cs:12,3.345)
--(axis cs:14,6.341)
--(axis cs:16,10.39)
--(axis cs:18,16.46)
--(axis cs:20,25.29)
--(axis cs:22,36.66)
--(axis cs:24,51.72)
--(axis cs:26,70.09)
--(axis cs:28,88.61)
--(axis cs:28,113.7)
--(axis cs:28,113.7)
--(axis cs:26,86.3)
--(axis cs:24,63.16)
--(axis cs:22,42.92)
--(axis cs:20,27.84)
--(axis cs:18,19.73)
--(axis cs:16,12.07)
--(axis cs:14,7.495)
--(axis cs:12,3.957)
--cycle;

\addplot [semithick, darkcyan0119187, mark=pentagon*, mark size=2, mark options={solid}]
table {%
12 0.7329
14 1.349
16 2.181
18 3.339
20 5.174
22 7.583
24 10.32
26 14.48
28 19.16
};
\addlegendentry{$N=250$}
\addplot [semithick, mediumturquoise51187238, mark=*, mark size=2, mark options={solid}]
table {%
12 1.576
14 2.798
16 4.725
18 7.396
20 11.43
22 16.65
24 22.93
26 31.67
28 41.71
};
\addlegendentry{$N=500$}
\addplot [semithick, darkcyan0153136, mark=triangle*, mark size=2, mark options={solid}]
table {%
12 2.632
14 4.662
16 7.933
18 12.11
20 18.37
22 26.21
24 38.63
26 49.86
28 68.09
};
\addlegendentry{$N=750$}
\addplot [semithick, tomato23811951, mark=diamond*, mark size=2, mark options={solid}]
table {%
12 3.651
14 6.918
16 11.23
18 18.09
20 26.56
22 39.79
24 57.44
26 78.2
28 101.2
};
\addlegendentry{$N=1000$}
\end{axis}

\end{tikzpicture}
\caption{Minimal Sufficient Reason $(\ell_1)$}
\end{subfigure}
\hfill
\begin{subfigure}{0.49\textwidth} 
\setlength\figurewidth{0.95\columnwidth}
\setlength\figureheight{0.7\figurewidth}

% This file was created with tikzplotlib v0.10.1.
\begin{tikzpicture}[font=\scriptsize, , every plot/.style={line width=1.2pt}]

\definecolor{darkcyan0119187}{RGB}{0,119,187}
\definecolor{darkcyan0153136}{RGB}{0,153,136}
\definecolor{gray}{RGB}{128,128,128}
\definecolor{mediumturquoise51187238}{RGB}{51,187,238}
\definecolor{tomato23811951}{RGB}{238,119,51}

\begin{axis}[
height=\figureheight,
legend cell align={left},
legend style={fill opacity=0.8, draw opacity=1, text opacity=1, at={(0.03,0.97)}, anchor=north west},
minor xtick={},
minor ytick={},
tick align=outside,
tick pos=left,
width=\figurewidth,
x grid style={gray},
xlabel={Image side length (sqrt of the dimension)},
xmajorgrids,
xmin=12, xmax=28,
xtick style={color=black},
xtick={12,14,16,18,20,22,24,26,28,30},
y grid style={gray},
ymajorgrids,
ymin=0, ymax=120,
ytick style={color=black},
ytick={0,20,40,60,80,100},
yticklabels=\empty
]
\path [draw=darkcyan0119187, fill=darkcyan0119187, opacity=0.2]
(axis cs:12,3.646)
--(axis cs:12,2.338)
--(axis cs:14,3.683)
--(axis cs:16,3.373)
--(axis cs:18,2.457)
--(axis cs:20,3.301)
--(axis cs:22,5.434)
--(axis cs:24,4.196)
--(axis cs:26,4.863)
--(axis cs:28,5.146)
--(axis cs:28,8.813)
--(axis cs:28,8.813)
--(axis cs:26,8.393)
--(axis cs:24,7.941)
--(axis cs:22,7.209)
--(axis cs:20,6.682)
--(axis cs:18,5.803)
--(axis cs:16,4.611)
--(axis cs:14,3.757)
--(axis cs:12,3.646)
--cycle;

\path [draw=mediumturquoise51187238, fill=mediumturquoise51187238, opacity=0.2]
(axis cs:12,14.86)
--(axis cs:12,8.249)
--(axis cs:14,7.723)
--(axis cs:16,10.8)
--(axis cs:18,12.17)
--(axis cs:20,14.32)
--(axis cs:22,21.8)
--(axis cs:24,21.15)
--(axis cs:26,17.6)
--(axis cs:28,18.2)
--(axis cs:28,29.43)
--(axis cs:28,29.43)
--(axis cs:26,31.34)
--(axis cs:24,25.53)
--(axis cs:22,25.38)
--(axis cs:20,23.02)
--(axis cs:18,20.68)
--(axis cs:16,19.2)
--(axis cs:14,14.99)
--(axis cs:12,14.86)
--cycle;

\path [draw=darkcyan0153136, fill=darkcyan0153136, opacity=0.2]
(axis cs:12,30.54)
--(axis cs:12,27.13)
--(axis cs:14,25.9)
--(axis cs:16,33.7)
--(axis cs:18,20.42)
--(axis cs:20,31.63)
--(axis cs:22,44.63)
--(axis cs:24,33.52)
--(axis cs:26,28.9)
--(axis cs:28,51.03)
--(axis cs:28,55.44)
--(axis cs:28,55.44)
--(axis cs:26,68.55)
--(axis cs:24,70.47)
--(axis cs:22,60.16)
--(axis cs:20,51.41)
--(axis cs:18,47.88)
--(axis cs:16,37.85)
--(axis cs:14,37.05)
--(axis cs:12,30.54)
--cycle;

\path [draw=tomato23811951, fill=tomato23811951, opacity=0.2]
(axis cs:12,71.45)
--(axis cs:12,48.39)
--(axis cs:14,46.34)
--(axis cs:16,62.42)
--(axis cs:18,65.71)
--(axis cs:20,69.2)
--(axis cs:22,75.33)
--(axis cs:24,59.15)
--(axis cs:26,77.98)
--(axis cs:28,58.63)
--(axis cs:28,113.8)
--(axis cs:28,113.8)
--(axis cs:26,118.3)
--(axis cs:24,116.5)
--(axis cs:22,104.6)
--(axis cs:20,84.65)
--(axis cs:18,67.66)
--(axis cs:16,70.44)
--(axis cs:14,63.31)
--(axis cs:12,71.45)
--cycle;

\addplot [semithick, darkcyan0119187, mark=*, mark size=2, mark options={solid}]
table {%
12 2.992
14 3.72
16 3.992
18 4.13
20 4.992
22 6.321
24 6.069
26 6.628
28 6.98
};
\addlegendentry{$N=250$}
\addplot [semithick, mediumturquoise51187238, mark=*, mark size=2, mark options={solid}]
table {%
12 11.55
14 11.36
16 15
18 16.42
20 18.67
22 23.59
24 23.34
26 24.47
28 23.82
};
\addlegendentry{$N=500$}
\addplot [semithick, darkcyan0153136, mark=*, mark size=2, mark options={solid}]
table {%
12 28.83
14 31.48
16 35.77
18 34.15
20 41.52
22 52.4
24 51.99
26 48.72
28 53.24
};
\addlegendentry{$N=750$}
\addplot [semithick, tomato23811951, mark=*, mark size=2, mark options={solid}]
table {%
12 59.92
14 54.83
16 66.43
18 66.68
20 76.92
22 89.97
24 87.84
26 98.12
28 86.2
};
\addlegendentry{$N=1000$}
\end{axis}

\end{tikzpicture}
\caption{Counterfactual $(\ell_2)$}
\end{subfigure}
    \caption{Runtimes for explanations over the MNIST dataset. The training set used has size $N := |S^+| + |S^-|$. Confidence intervals of $95\%$ over 5 independent runs are displayed.}
    \label{fig:computation-abductive}
\end{figure}

\paragraph{\bf Results.}  On the discrete setting,~\Cref{fig:disc-count-gurobi} displays experimental results over synthetic random data, comparing the performance of IQP and SAT solving. While our implementation of the former scales significantly better than the latter, it is worth mentioning that~\texttt{Gurobi} was run using $8$ threads, whereas~\texttt{cardinality-cadical} is a single-threaded program. On the continuous setting,~\Cref{fig:computation-abductive} displays experimental results for the algorithms described in~\Cref{theo:l2-cont-counterfactual} and~\Cref{prop_l1_1check}. We remark that the use of a library for fast NN-classification such as \text{FAISS}~\cite{douze2024faiss} was key for performance in the computation of minimal sufficient reasons. In general, our results suggest that for hundreds of features, and up to a thousand points, explanations can be computed in practice. A faster implementation in a lower level language, using pruning heuristics as those of~\cite{DBLP:conf/iscopt/ContardoFRV24,eppsteinFindingRelevantPoints2022}, is part of our future work.

\section{Final Remarks}\label{sec:final}
Our work represents an initial step in studying the computational cost of generating explanations for $k$-NN classifiers. As demonstrated, the landscape is nuanced, with the complexity of finding explanations varying depending on the metric used. 
We believe that our results and proof techniques provide valuable insights into the types of methods needed to practically address the explanation problems studied in this paper.
A summary of our results is shown in Table~\ref{tab:example}. The problems of checking and computing minimal sufficient reasons are polynomial-time solvable when checking sufficient reasons is tractable (recall Proposition~\ref{prop_check_minimal_reduction}), and hard otherwise. All results in Table~\ref{tab:example} are completeness results, except for minimum sufficient reasons under the  $\ell_1$-distance and $k>1$, where we only obtain $\np$-hardness. We leave the precise complexity of this case open. 

The kind of explanations we have studied are often said to be ``local''~\cite{localSurvey,hakkoumGlobalLocalInterpretability2024}, since they aim to elucidate the behavior of a classifier in a local region of the space around an input point, as opposed to ``global'' explainability which aims to provide insight into a classier as a whole. Recently, Bassan et al. studied the difference between local and global interpretability from a computational-complexity perspective~\cite{bassan2024localvsglobalinterpretability}. Interestingly, a recent line of work has studied the computational problem of \emph{thinning} $k$-NN classifiers by removing redundant points in the training set~\cite{eppsteinFindingRelevantPoints2022,flores2022improved,rohrer2023reducing}. Arguably, this line of work contributes to the global interpretability of $k$-NN classifiers, and in practice might serve to speed up the computation of local explanations.

There are several intriguing directions for future research. 
\begin{itemize}
    \item  First, we seek to explore {\sc $k$-Counterfactual Explanation} for metrics based on $\ell_p$, where $ p > 2 $. Specifically, we ask whether $\ell_2$ is the only metric for which this problem is tractable.

    \item Second, it is interesting to explore, whether our positive results are extendable to the case of more than two labels. It does not make any difference for $k = 1$: if a point $\bar x$ to be explained was classified with some label, then all the other labels can be merged into a single one, effectively reducing everything to the binary case.  The same reasoning does not work for $k \ge 3$, and it is open, whether the {\sc $k$-Counterfactual Explanation} can still be computed  in polynomial time for $\ell_2$ in the multi-label case. One can generalize our algorithm from Theorem \ref{theo:l2-cont-counterfactual} to have exponential dependence on the number of labels; this answers our question positively when the number of labels is constant, but not in the general case.
    
        \item Lastly, we are interested in determining the extent to which the \np-hard problems discussed in this paper can be approximated. For instance, can {\sc $k$-Minimum Sufficient Reason}, which is \np-hard in all the settings considered, be tackled using polynomial-time approximation algorithms that produce a sufficient reason whose size is reasonably close to the minimum?
\end{itemize}

\begin{table}
%[htbp]
    \centering
    \begin{tabular}{l*{3}{cc}}
        \toprule
       \textbf{Explanation} & {Counterfactual} & \multicolumn{2}{c}{Check Sufficient Reason} & \multicolumn{2}{c}{Minimum Sufficient Reason} \\
         \cmidrule(lr){3-4} 
          \cmidrule(lr){5-6}
    \textbf{Metric space} &   $k\geq 1$  & $k=1$ & $k>1$ & $k=1$ & $k>1$\\
        \midrule
        $(\mathbb{R}, D_2)$  & P (Thm.~\ref{theo:l2-cont-counterfactual}) & P (Prop.~\ref{prop:cont-check-sr}) & P (Prop.~\ref{prop:cont-check-sr}) & NP-c (Coro.~\ref{coro:np-c-msr}) & NP-c (Coro.~\ref{coro:np-c-msr}) \\
       $(\mathbb{R}, D_1)$  & NP-c (Thm.~\ref{theo:l1-cont-counterfactual}) & P (Prop.~\ref{prop_l1_1check}) & coNP-c (Thm.~\ref{theo:continuous-check-sr-k3}) & NP-c (Coro.~\ref{coro:np-c-msr}) & NP-h (Thm.~\ref{theo:minimum-sr})\\
       $(\{0,1\}, D_H)$  & NP-c (Thm.~\ref{theo:discrete-counterfactual}) & P (Prop.~\ref{prop:discrete-check-sr-k1}) & coNP-c (Thm.~\ref{theo:discrete-check-sr-k3}) &  NP-c (Coro.~\ref{coro:np-c-msr}) & $\Sigma_2^p$-c (Thm.~\ref{theo:sigma2p-msr-discrete}) \\
        \bottomrule
    \end{tabular}
    \caption{Summary of complexity results.}
    \label{tab:example}
\end{table}

\paragraph{Acknowledgments.}
Barcel\'o is funded by the National Center for Artificial Intelligence CENIA FB210017, Basal ANID, and by ANID Millennium Science Initiative Program Code
ICN17002. Kozachinskiy is funded by Grants ANID/FONDECYT Iniciación 1250060 and ANID/Basal National Center for Artificial Intelligence CENIA FB210017. Verschae is funded by Grants ANID/FONDECYT Regular 1221460 and Centro de Modelamiento Matemático (CMM) BASAL fund FB210005 for centers of excellence, ANID-Chile. Romero is funded by the National Center for Artificial Intelligence CENIA FB210017, Basal ANID. Subercaseaux is  supported by the U.S. National
Science Foundation under grant DMS-2434625.


\end{document}